\documentclass{article}

\PassOptionsToPackage{numbers,compress}{natbib}
\usepackage[main,final]{neurips_2026}
\usepackage{graphicx}
\setkeys{Gin}{draft=false}
\usepackage[utf8]{inputenc}
\usepackage[T1]{fontenc}
\usepackage{hyperref}
\hypersetup{hypertexnames=false}
\usepackage{url}
\usepackage{booktabs}
\usepackage{amsfonts}
\usepackage{nicefrac}
\usepackage{xcolor}
\usepackage{microtype}
\usepackage{array}
\usepackage{amsmath,amssymb,amsthm}
\usepackage{algorithm}
\usepackage{algpseudocode}
\usepackage{float}
\usepackage{caption}
\usepackage{enumitem}

\newtheorem{proposition}{Proposition}

\title{M$^3$: Reframing Training Measures for Discretized Physical Simulations}

\author{%
Yuan Mei$^{1,2}$
\quad
Xingyu Song$^{3}$
\quad
Xiaowen Song$^{1,2}$
\quad
Naoya Takeishi$^{3}$\\[0.5em]
$^{1}$State Key Laboratory of Fluid Power and Mechatronic Systems (Zhejiang University)\\
$^{2}$Zhejiang Key Laboratory of Advanced Manufacturing Technology (Zhejiang University)\\
$^{3}$The University of Tokyo\\
\texttt{\{meiyuan98, songxw\}@zju.edu.cn}\\
\texttt{\{xsong, ntake\}@g.ecc.u-tokyo.ac.jp}
}

\begin{document}

\maketitle

\begin{abstract}
    Neural surrogate models for physical simulations are trained on discretized samples of
    continuous domains, where the induced empirical measure leads to uneven supervision, biasing
    optimization and causing spatial inconsistencies in physical fidelity. To mitigate this
    measure-induced bias, we propose M$^3$ (Multi-scale Measure Modeling), a scalable framework that
    balances training measures by partitioning space according to physical variation and allocating
    supervision across multiple scales. Applied to three industrial-scale datasets with diverse
    discretizations, M$^3$ consistently improves predictions over most of the physical domain,
    achieving up to 4.9$\times$ lower error in large-scale volumetric cases. These gains persist
    under aggressive subsampling (160M $\rightarrow$ 16M $\rightarrow$ 1.6M points), where
    M$^3$-trained models outperform those trained on higher-resolution data, reducing
    physics-weighted relative $L_2$ error by 3--4$\times$ and the corresponding MSE by up to
    13$\times$. These results highlight the distribution of supervision across the physical domain
    as a key factor in operator learning and establish M$^3$ as a scalable, data-efficient framework
    for physically consistent modeling.
    Code is available at \url{https://github.com/PhysDataRefine/M3}.
\end{abstract}

\section{Introduction}
\label{sec:intro}

Modern neural surrogate models have become an effective tool for approximating complex physical
systems \cite{li2021fno,kovachki2023neural}. By learning mappings over continuous spatial domains,
they enable fast predictions at arbitrary resolutions and query
locations~\cite{li2021fno,alkin2025abupt}. In practice, however, training relies on
finite samples obtained from discretized simulation outputs, and thus the effective training
distribution is implicitly defined by the discretization scheme employed in the simulation. This
induces an empirical measure over the domain that inherits the discretization's non-uniform spatial
structure.

As a result, even full-mesh pointwise training is inherently biased. Finely discretized
regions receive disproportionately more optimization updates, while sparse regions contribute weaker
training signals, leading to spatially non-uniform approximation errors across the domain. This
issue is particularly pronounced in high-fidelity industrial datasets that contain $10^6$--$10^8$
mesh elements per case \cite{ashton2024ahmedml,ashton2024drivaerml,luminary2025shiftwing}, where
full-data training becomes impractical.
Merely subsampling such training data inevitably inherits and amplifies the imbalance due to the
non-uniform coverage in physical space. It may become a dominant source of error, and increasing
model capacity may be insufficient to compensate for the systematically deficient or
underrepresented physical signals.

Evaluation can also create a misleading illusion when metrics are computed with uniform pointwise
weighting. Due to the discretization bias, large errors in sparsely sampled regions can be
overwhelmed by the accumulation of many small errors in dense regions. As a result, performance
tends to be dominated by overrepresented areas, yielding overly optimistic estimates that may not
reflect the model's true behavior over the domain. Such metrics therefore provide an unreliable
basis for comparison, making it difficult to draw clear conclusions about relative model
performance.

These observations motivate viewing subsampling not only as data compression, but also as a problem
of physically meaningful empirical measure design. The key question is how to construct a training
measure that reasonably captures cross-scale physical variation rather than discretization
artifacts. Under a fixed budget, both the distribution support (which locations are supervised) and
the allocation of supervision mass (how often they enter the loss) are critical. To address this, we
propose \textbf{M$^3$} (\emph{Multi-scale Measure Modeling}), a lightweight pipeline that constructs
structured training distributions for large-scale point clouds. It builds a variation-aware spatial
partition via Morton ordering, assigns cells to scale-aware strata, and permits flexible allocation
of budget over the resulting support. This yields balanced spatial coverage across scales in the
physical domain. Based on this formulation, our contributions are summarized as follows:
\begin{itemize}[itemsep=0pt,topsep=0pt,leftmargin=*]
    \item We identify that discretization-induced data structures systematically alter the effective training objective in operator learning, leading to biased optimization and unreliable performance assessment under standard pointwise metrics.

    \item We propose \textbf{M$^3$}, a scalable framework for large-scale simulation subsampling that constructs multi-scale empirical measures aligned with underlying physical structure, enabling data-efficient training under limited computational budgets and improving generalization across benchmarks.

    \item We show that, in operator learning, increasing raw simulation data alone does not necessarily improve performance under spatially non-uniform discretizations. In contrast, models trained with physically structured measures achieve stronger performance with less data.
\end{itemize}

\section{Related Work}
\label{sec:related}
\paragraph{Neural PDE Surrogates and Architectures.}
Neural PDE surrogates model physical systems via operator learning, with representative approaches
including FNOs~\cite{li2021fno} and DeepONets~\cite{lu2021deeponet}, alongside mesh- and graph-based
simulators~\cite{pfaff2021meshgraphnets,sanchezgonzalez2020learning,brandstetter2022message} and
attention-based architectures over mesh
tokens~\cite{li2023factformer,hao2023gnot,wu2024transolver,alkin2025abupt}.
Recent scalable variants further extend this paradigm through latent or compressed interfaces,
including low-rank token exchange (e.g., Transolver~\cite{wu2024transolver}), anchor-based
representations (e.g., AB-UPT~\cite{alkin2025abupt,alkin2025abupt_shift}), and geometry-aware
multiscale encoding (e.g., GAOT~\cite{wen2025gaot}), improving computational efficiency and
information exchange at scale.
These approaches operate on sampled points over discretized domains and optimize for accuracy and
scalability, but they do not control which locations contribute to the training loss or how
frequently.

\paragraph{Sampling and Subset Training.}
\label{par:related-sampling-subset}
Industrial datasets often reach terabytes and contain $10^6$--$10^8$ mesh points per
case~\cite{ashton2024ahmedml,ashton2024drivaerml,luminary2025shiftwing}, making full-mesh
training impractical. As a result, most pipelines rely on subsampled vertices or token
batches~\cite{wu2024transolver,hao2023gnot,alkin2025abupt_shift}. They commonly
adopt a two-stage pipeline: raw data are downsampled to a manageable scale (e.g., $10\%$) offline,
and training tokens are generated via a second-stage subsampling process. Each update corresponds to
a Monte Carlo estimate under the induced sampling distribution, yielding a sampling-dependent
supervision pattern over the discretized domain.

Prior work improves performance by modifying the sampling strategy over discretization-induced
candidate sets. Representative approaches include adaptive residual- or error-driven point selection
for PINNs~\cite{wu2023pinnSampling,song2025rlpinns}, non-uniform minibatch construction methods such
as DPP-based sampling~\cite{bardenet2021dppMinibatch,zhang2017dmsgd}, and stratified or
side-information-guided schemes~\cite{liu2020sgdrs,gopal2016adaptiveSamplingSGD}. Classical adaptive
mesh refinement (AMR)~\cite{berger1989amr} adjusts spatial resolution based on solution features,
but operates on mesh discretization rather than stochastic minibatch optimization.

\paragraph{Data-efficient Learning under Fixed Data Budgets.}
\label{par:related-fixed-support}
Beyond batch-level sampling, a range of statistical techniques, including importance
weighting~\cite{kimura2024importanceWeightingSurvey,shimodaira2000covariateShift},
coresets~\cite{mirzasoleiman2020coresets,sener2018coreset}, and active
learning~\cite{deepActiveLearningSurvey2025}, improve data efficiency by reweighting samples or
selecting informative subsets from an existing dataset. These methods approximate the full-data
objective using a reduced or rebalanced set of samples. Consequently, their improvements are
confined to observed data and do not extend to sparsely sampled or unseen regions under a fixed
budget.

Existing approaches shape the effective training distribution under limited computational budgets
through architectural design, representation choices, or sample selection and weighting. However,
this distribution is typically an implicit byproduct of discretization and data pipelines rather
than a controlled design variable. In contrast, we explicitly construct the training measure via a
multi-scale scheme that governs spatial coverage and allocation over the discretized physical
domain, providing a complementary perspective on data-efficient learning.

\section{M\texorpdfstring{$^3$}{3}: Multi-scale Measure Modeling}
\label{sec:method}
\label{sec:method-setup}

\begingroup
\setlength{\abovedisplayskip}{4pt plus 1pt minus 1pt}
\setlength{\belowdisplayskip}{4pt plus 1pt minus 1pt}
\setlength{\abovedisplayshortskip}{2pt plus 1pt minus 1pt}
\setlength{\belowdisplayshortskip}{2pt plus 1pt minus 1pt}
\setlength{\textfloatsep}{6pt plus 2pt minus 2pt}
\setlength{\intextsep}{6pt plus 2pt minus 2pt}
\setlength{\floatsep}{6pt plus 2pt minus 2pt}
\setlength{\abovecaptionskip}{3pt}
\setlength{\belowcaptionskip}{0pt}
\vspace*{-3pt}
\subsection{Problem Setup and Objective}\label{sec:problem-setup}
We consider supervised learning from discretized physical simulations, where each simulation is
represented as a labeled point set
$\mathcal{P} = \{(\mathbf{x}_i, \mathbf{y}_i)\}_{i=1}^{N}$,
with $\mathbf{x}_i \in \Omega \subset \mathbb{R}^3$ denoting spatial locations and $\mathbf{y}_i \in
\mathbb{R}^d$ the corresponding physical quantities. These samples are generated via a
discretization process (e.g., particle- or grid-based solvers), producing a finite representation of
an underlying continuous physical field. Let $\mu$ denote the underlying data distribution over
$(\mathbf{x}, \mathbf{y})$. The population risk under a pointwise loss $\mathcal{L}$ is defined as
\begin{equation}
\label{eq:population-risk}
\mathcal{R}_{\mu}(\theta)
:= \mathbb{E}_{(\mathbf{x},\mathbf{y})\sim\mu}
\bigl[\mathcal{L}(f_\theta(\mathbf{x}),\mathbf{y})\bigr].
\end{equation}
Given the full dataset $\mathcal{P}$, learning is performed via the empirical measure
$\hat{\mu}_{\mathcal{P}} = \frac{1}{N}\sum_{i=1}^{N}\delta_{(\mathbf{x}_i,\mathbf{y}_i)}$,
and the corresponding empirical risk $\mathcal{R}_{\hat{\mu}_{\mathcal{P}}}(\theta)$. In practice,
due to computational constraints in large-scale simulations ($N \gtrapprox 10^6$--$10^8$), training
is performed on a subset $\mathcal{I}_{\mathrm{tr}} \subseteq \mathcal{I} := \{1,\dots,N\}$ with
$|\mathcal{I}_{\mathrm{tr}}|\ll N$ under a nominal budget $m$, inducing a subsampled training
measure $\mu_S = \frac{1}{|\mathcal{I}_{\mathrm{tr}}|}\sum_{i \in \mathcal{I}_{\mathrm{tr}}}
\delta_{(\mathbf{x}_i,\mathbf{y}_i)}$.
The learning objective becomes minimizing $\mathcal{R}_{\mu_S}(\theta)$. Learning is therefore
governed by $\mu_S$, which is typically an approximation of the empirical distribution
$\hat{\mu}_{\mathcal{P}}$.

From this perspective, the key problem is to construct a target distribution, namely $\mu^\star$,
that better approximates the underlying physical system, instead of the straightforward empirical
measure $\hat{\mu}_{\mathcal{P}}$ and its subsampled version. On top of that, we want to approximate
such a target under a fixed budget by a cell-level allocation whose realized samples define $\mu_S$,
while satisfying three key requirements: (i) preserve rich inter-scale physics information and
prevent intra-scale over-concentration, (ii) ensure spatially balanced coverage of the domain, and
(iii) be efficiently constructible at an industrial scale.

\subsection{Multi-Scale Measure Construction}
\label{sec:method-formulation}
\label{sec:measure-formulation}
To design $\mu^\star$, we first aggregate spatially nearby points into cells, which induces a
discrete partition of the domain and provides a finite support over which $\mu^\star$ is defined. We
then group the cells by their variation levels to form a multi-scale representation of spatial
heterogeneity. Under this construction, designing $\mu^\star$ amounts to allocating mass across
scales and distributing it across cells within each scale in a balanced manner.

Formally, let $\mathcal{C}$ denote a finite set of spatial cells over which supervision budget is
allocated. Each cell $c \in \mathcal{C}$ indexes a localized region of the domain, and cells are
pairwise disjoint up to measure-zero boundaries.
Let $\Lambda \subseteq \{1,\dots,L_{\max}\}$ be the set of scale labels, where
$L_{\max}\in\mathbb{N}$ is the maximum scale index. The index $\ell \in \Lambda$ is ordered from
coarser to finer variation bands, where finer scales capture higher variation and coarser scales
correspond to smoother regions. We group the cells into subsets
$\{\mathcal{C}_\ell\}_{\ell\in\Lambda}$, where each $\mathcal{C}_\ell \subseteq \mathcal{C}$
collects cells with similar variation magnitude, and $\{\mathcal{C}_\ell\}$ forms a partition of
$\mathcal{C}$. Here, cells are purely spatial units, while scale labels are assigned on top of them.
For each $\ell\in\Lambda$, let $\alpha_\ell$ denote the target mass assigned to scale~$\ell$. Let
$\mu_\ell$ denote the uniform measure over $\mathcal{C}_\ell$, ensuring no additional preference
within each scale.

We then define the following target family:
\begin{equation}
\label{eq:mu-star-ms}
\mu^\star \;=\; \sum_{\ell\in\Lambda} \alpha_\ell \,\mu_\ell,
\qquad
\mu_\ell(\{c\})\;=\;\begin{cases}
|\mathcal{C}_\ell|^{-1}, & c\in\mathcal{C}_\ell,\\
0, & \text{otherwise},
\end{cases}
\qquad
\alpha_\ell\ge 0,\ \sum_{\ell\in\Lambda}\alpha_\ell=1.
\end{equation}
The weights $\{\alpha_\ell\}$ control how mass is allocated across scales. We use $\alpha_\ell
\approx |\Lambda|^{-1}$ as a simple coverage prior, since alternatives such as $\alpha_\ell \propto
|\mathcal{C}_\ell|$ can overemphasize dominant scales under hard budgets. In other words, each cell
receives mass $\mu^\star(\{c\})=\frac{\alpha_\ell}{|\mathcal{C}_\ell|}$ for $c\in\mathcal{C}_\ell$,
which enforces inter-scale coverage through $\{\alpha_\ell\}$ and intra-scale fairness via
uniformity within each scale. Given such a target measure $\mu^\star$, constructing a cell-level
allocation under a hard budget (and realizing it as a point-level measure $\mu_S$) induces an
intractable combinatorial problem at scale. To address this, we design a scalable procedure that
maintains inter-scale coverage while preventing intra-scale concentration under discrete
constraints. We implement this as a three-stage pipeline to build $\mu_S$, as shown in the next
section. Appendix~\ref{sec:app-measure-shift} bounds the cell-level discrepancy from $\mu^\star$ by
inter- and intra-scale terms, which the two stages of Step~3 address in turn.

\subsection{Three-Stage Modeling}
\label{sec:construction}
\label{sec:method-overview}
We instantiate the discrete support $\mathcal{C}$ via a variation-adaptive partition of the point cloud $\mathcal{P}$.
Each cell $c \in \mathcal{C}$ aggregates a subset of points and acts as the atomic unit for
subsequent measure construction. Based on this partition, we assign scale labels to cells to obtain
$\{\mathcal{C}_\ell\}$, and then perform budget allocation to construct the empirical measure
$\mu_S$. The pipeline consists of three stages: (i) partition construction, (ii) scale grouping, and
(iii) budgeted allocation, as illustrated in Figure~\ref{fig:m3-pipeline}. We detail each step in
the following subsections.

\begin{figure}[t]
  \centering
  \includegraphics[width=\linewidth]{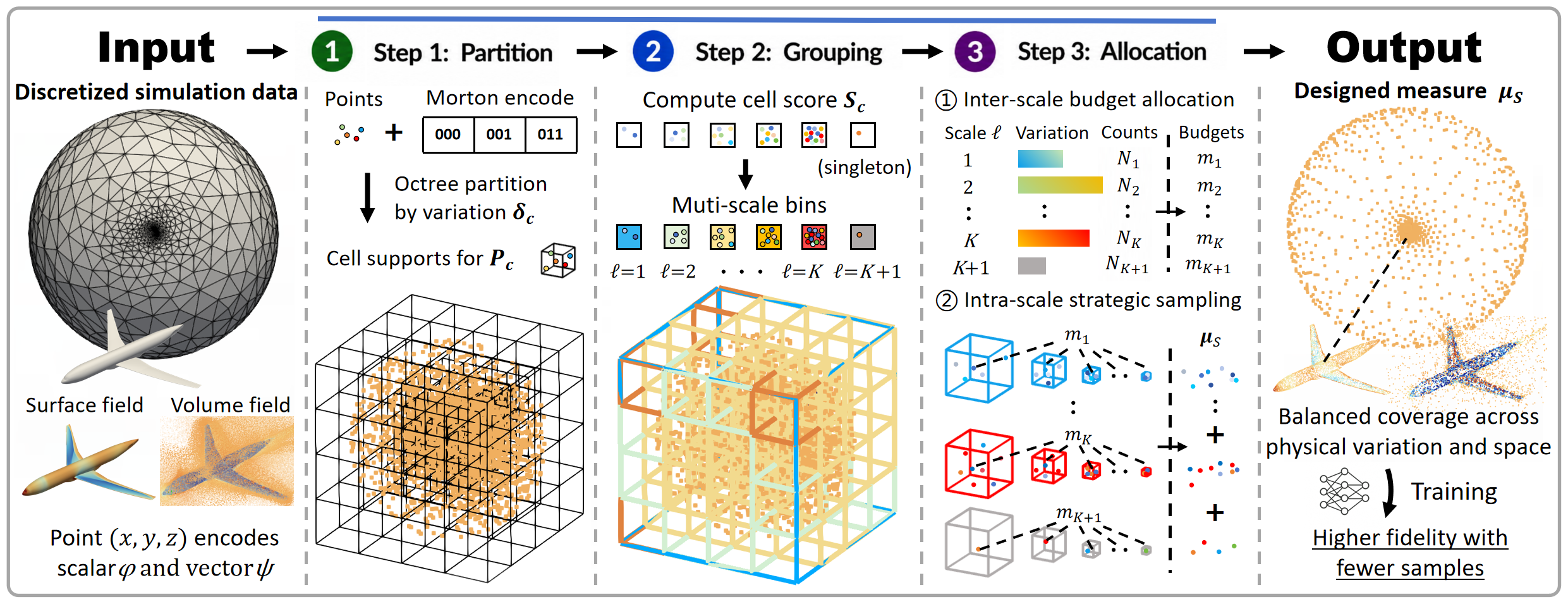}
  \caption{\textbf{Overview of the M$^3$ pipeline.} Starting from discretized simulation data, M$^3$ constructs a variation-adaptive Morton partition, groups cells into multi-scale bins, and allocates the training budget across and within scales,  yielding a structured training measure for better operator learning.}
  \label{fig:m3-pipeline}
\end{figure}

\paragraph{\textbf{Step 1: Variation-Aware Partition.}}
\label{sec:method-partition}
\label{sec:method-variation}
Our goal is to partition irregular high-resolution point clouds into coarse support units guided by
local physical variation, forming a structured basis for sampling and measure construction. We
achieve this via Morton-order~\cite{morton1966zorder} octree~\cite{meagher1982octree} partitioning, which
ensures scalability, preserves spatial locality, and enables efficient indexing during sampling.

Points are first sorted on a lattice-quantized bounding cube. We grow an octree over this cube.
Each cell $c$ has an octree depth $G(c)\in\{0,1,\dots,G_{\max}\}$, defined as the number of
successive refinements from the root bounding cell, and an associated index set
$\mathcal{P}_c\subseteq\{1,\dots,N\}$ of the points it contains. Cells are recursively subdivided
until either (i) a per-cell point cap~$\kappa$ is met, (ii) the depth cap $G(c)\ge G_{\max}$ is
met, or (iii) a low-variation condition (defined below) is satisfied.

To decide where refinement is needed without relying on global gradient information over all $N$
points, we define a lightweight local variation score $\delta_c$ based on within-cell extrema.
All supervised channels use z-score normalization before computing extrema-based variations,
ensuring $\delta_c$ is dimensionless and comparable across fields.
More formally, let $\{\phi^{(k)}\}_{k\in\mathcal{K}}$ denote supervised scalar channels and
let $\{\boldsymbol{\psi}^{(j)}\}_{j\in\mathcal{J}}$ denote supervised vector channels (e.g., for CFD,
surface uses $(\phi,\boldsymbol{\psi})=(p_s,\boldsymbol{\tau})$ and volume uses
$(p_v,\mathbf{u})$).
For a cell $c$, define the scalar range \(\Delta \phi^{(k)}_c :=
\max_{i\in\mathcal{P}_c}\phi^{(k)}_i-\min_{i\in\mathcal{P}_c}\phi^{(k)}_i\),
and the vector \emph{projection diameter}
\begin{equation}
\label{eq:Delta-v-proj}
\Delta \psi^{(j)}_c \;=\;
\max_{\mathbf{d}\in\mathcal{D}}
\left(
\max_{i\in\mathcal{P}_c} \mathbf{d}^{\top}\boldsymbol{\psi}^{(j)}_i
\;-\;
\min_{i\in\mathcal{P}_c} \mathbf{d}^{\top}\boldsymbol{\psi}^{(j)}_i
\right),
\end{equation}
where $\mathcal{D}$ is a finite set of unit vectors $\mathbf{d}\in\mathbb{R}^3$ used to approximate
vector-channel spread by projection, and concrete cardinalities and defaults are listed in
Table~\ref{tab:partition-hparams-datasets}, with procedural details in
Algorithm~\ref{alg:physical_gradient_partition}.
\noindent We aggregate the channelwise variation via
\begin{equation}
\label{eq:delta-score}
\delta_c
=
\max\Bigl(
\max_{k\in\mathcal{K}} w^{(k)}_{\phi}\,\Delta \phi^{(k)}_c,\;
\max_{j\in\mathcal{J}} w^{(j)}_{\psi}\,\Delta \psi^{(j)}_c
\Bigr)
\end{equation}
\noindent The weights $w^{(k)}_{\phi}$ and $w^{(j)}_{\psi}$ tune the relative contribution of each
physical quantity to refinement triggering, balancing scalar versus vector channels under a common
threshold.
\noindent Condition~(iii) holds when $\delta_c \le \varepsilon_{\mathrm{refine}}$, with $\delta_c$
recomputed after each split. As cell diameters shrink, the same extrema-based test is applied over
progressively smaller neighborhoods, providing a multiscale, gradient-free probe of local variation.

The process yields finer cells in high-variation regions and coarser cells elsewhere (see cell
partitioning results in Appendix~\ref{sec:geom-partition}). Under Lipschitz regularity, these
extrema ranges provide a conservative bound on within-cell oscillation at the cell scale, so
$\delta_c$ serves as a proxy for local variation intensity. The partition $\mathcal{C}$ defines the
discrete support for subsequent measure design. Parameter settings are summarized in
Appendix~\ref{sec:app-method-details}, where we also provide the motivation and physical intuition
behind their selection.

\paragraph{\textbf{Step 2: Multi-Scale Grouping.}}
\label{sec:method-bucketing}
Purely geometry-driven depth is a poor proxy for local physical variation when cell spatial scales
differ. This is because refinement is driven by within-cell extrema: a localized high-gradient
region in a large coarse cell can dominate the score and trigger subdivision. As a result, after
splitting, only a subset of child cells retains high variation while others become smooth, yet may
still appear at large octree depths $G(c)$ since refinement is inherited from parent-level extrema
rather than intrinsic variation. Consequently, octree depth $G(c)$ alone may over-refine quiescent
regions and fail to separate physically distinct variations. We therefore recalibrate cells based on
their actual variation scales.

For each cell $c$ with $|\mathcal{P}_c| > 1$, let $h_c > 0$ denote its geometric width. Using the
Step~1 score $\delta_c$ in Equation~\eqref{eq:delta-score}, we define a gradient-scale intensity
$\tilde{T}_c$ and log-intensity score $s_c$ as follows:
\begin{equation}
\label{eq:strata-intensity}
\tilde{T}_c := \frac{\delta_c}{h_c},
\qquad
s_c := \log(\epsilon_{\log}+\tilde{T}_c),
\end{equation}
where $\epsilon_{\log}>0$ is a small constant introduced only for numerical stability when
$\tilde{T}_c \approx 0$. We apply log-scaling to mitigate the heavy-tailed distribution of variation
scores and improve comparability across cells.
The log-intensity score $s_c$ provides a comparable scalar axis across cells of different sizes. We
discretize $\{s_c\}$ into $K \ge 2$ ordered bins using evenly spaced thresholds over its empirical
range, and assign each cell a scale label $\ell(c)$ accordingly.

Cells with $|\mathcal{P}_c| = 1$ lack reliable within-cell variation signals and are assigned to a
dedicated stratum $\ell = K{+}1$. This induces a partition of the cell set $\mathcal{C}$ into scale
strata, $\mathcal{C}_\ell = \{ c \in \mathcal{C} : \ell(c) = \ell \}$ for $\ell \in \Lambda
\subseteq \{1,\dots,L_{\max}\}$ with $L_{\max} := K{+}1$. The resulting stratification provides a
stable and comparable notion of variation across heterogeneous cells, serving as the basis for
subsequent multi-scale allocation.

\paragraph{\textbf{Step 3: Budgeted Allocation.}}
\label{sec:method-sampling}
\label{sec:method-algo-overview}
Step~3 constructs a capacity-constrained cell-level approximation $\bar{\mu}_S$ to the target
$\mu^\star$ under a hard integer budget $m$ on the fixed cell support $\mathcal{C}$ and its induced
scale strata $\{\mathcal{C}_\ell\}$. It then realizes the point-level training measure $\mu_S$ by
sampling within cells. In particular, it aims to match the inter-scale allocation specified by
$\{\alpha_\ell\}$ while distributing the budget uniformly within each scale under discrete
constraints.
Let $N_c := |\mathcal{P}_c|$ denote the number of points in cell $c$. Given a fill ratio
$\rho\in(0,1]$, we define the effective capacity
\[
\tilde{N}_c := \min\Bigl(N_c,\;\max\bigl(1,\lceil \rho N_c\rceil\bigr)\Bigr).
\]
The fill ratio $\rho$ controls the per-cell sampling fraction at which its physical content is
sufficiently represented, preventing cells with large point counts from exhausting the global
budget, thereby reserving capacity for other cells and scales. We seek integer quotas $q_c \in
\{0,1,\dots,\tilde{N}_c\}$ with total realized budget $m' :=
\min\bigl(m,\sum_{c\in\mathcal{C}}\tilde{N}_c\bigr)$ such that $\sum_{c\in\mathcal{C}} q_c = m'$.
The quotas $\{q_c\}$ thus define a cell-level measure $\bar{\mu}_S$ with masses
$\bar{\mu}_S(\{c\})=q_c/m'$ on $\mathcal{C}$. The point-level training measure $\mu_S$ is then
realized by sampling $q_c$ points within each cell.

To model the multi-scale structure, we perform allocation in two stages. First, we assign integer
budgets $\{m_\ell\}$ across scales:
\begin{equation}
\label{eq:level-capacity}
\sum_{\ell\in\Lambda} m_\ell = m',
\qquad
0 \le m_\ell \le \tilde{N}_\ell,
\qquad
\tilde{N}_\ell := \sum_{c\in\mathcal{C}_\ell} \tilde{N}_c,
\end{equation}
where $\tilde{N}_\ell$ denotes the total capacity of scale $\ell$. The budgets $\{m_\ell\}$ are
chosen to approximate the target masses (e.g., $m_\ell \approx \alpha_\ell m'$), subject to the
capacity limits above.
Given $\{m_\ell\}$, we then allocate per-cell quotas within each scale under the constraints
\begin{equation}
\label{eq:cell-quota-constraints}
\sum_{c\in\mathcal{C}_\ell} q_c = m_\ell,
\qquad
0 \le q_c \le \tilde{N}_c.
\end{equation}
We use a water-filling procedure to balance quotas across cells under capacity limits. Starting from
zero allocation, samples are assigned iteratively to cells with remaining capacity, prioritizing
those with lower current allocation. The process stops when the scale budget is exhausted or all
cells reach capacity. For each cell $c$, we sample $q_c$ distinct indices uniformly without
replacement from its associated point set and aggregate them into the training set
$\mathcal{I}_{\mathrm{tr}}$.
Overall, this procedure enforces inter-scale balance and intra-scale uniformity under finite-sample
constraints, yielding an empirical measure $\mu_S$ aligned with the target design.
Pseudocode for Step~3 is provided in Algorithm~\ref{alg:cell_equal_level_aware_morton_sampling}.
\endgroup

\section{Experiments}
\label{sec:experiments}
\noindent\textbf{Scope.}
The aim of the experiments is to examine whether discretization-induced data distributions affect
operator learning, and whether the measure constructed by M$^3$ better reflects the underlying
physical structure of the simulation domain, thereby leading to more physically faithful model
predictions. To this end, we primarily compare M$^3$ with uniform random subsampling from the
original mesh, which provides a strong and natural reference and represents a common practice for
constructing training subsets from large-scale simulation data. Additional evidence supporting its
strength, together with comparisons against other baselines, is provided in
Appendix~\ref{app:baselines}.

\noindent\textbf{Backbone.}
We use AB-UPT \cite{alkin2025abupt} as the main backbone due to its scalability and decoupled
anchor--query design, where latent features are constructed from anchor points independently of the
query locations. This decoupling keeps the latent context anchor-dependent, allowing inference with an input measure different from that used during training, including full-mesh evaluation under a common query distribution. Appendix~\ref{app:backbone-mechanism} motivates this choice through a
mechanism-level comparison with Transolver.

\noindent\textbf{Datasets.}
We conduct experiments on three representative CFD datasets, namely SHIFT-Wing
\cite{luminary2025shiftwing}, AhmedML \cite{ashton2024ahmedml}, and DrivAerML
\cite{ashton2024drivaerml}, as summarized in Table~\ref{tab:exp-datasets-overview}. These datasets
cover aircraft wings, simplified bluff bodies, and realistic car geometries, with physical fields
defined on both 2D surface manifolds and 3D surrounding flows. They span anisotropic, near-uniform,
and highly non-uniform meshes (see Figure~\ref{fig:Mesh_3datasets}), with diverse flow regimes
including transonic flows with shock-induced variations, separated wakes, and recirculation.

\begin{table}[ht]
  \centering
  \small
  \caption{\textbf{Datasets overview.} Surface and volume columns report mesh points per case, where $\sim$M denotes millions. \emph{Empirical} refers to mesh refinement based on conventional engineering practice.}
  \label{tab:exp-datasets-overview}
  \setlength{\tabcolsep}{3pt}
  \resizebox{\linewidth}{!}{
  \begin{tabular}{@{}lccccccc@{}}
    \toprule
    Dataset & Discretization Scheme & Refinement & Streamwise Extent (m) & Surface & Volume & Cases & Size\\
    \midrule
    SHIFT-Wing & anisotropic & AMR & $[-1000,1000]$ & $\sim$3M & $\sim$6M & 1714 & $\sim$2TB \\
    AhmedML & relatively uniform & empirical & $[-4,6]$ & $\sim$1M & $\sim$20M & 500 & $\sim$8TB \\
    DrivAerML & non-uniform & empirical & $[-40,80]$ & $\sim$9M & $\sim$160M & 500 & $\sim$31TB \\
    \bottomrule
  \end{tabular}
  }
\end{table}

\begin{figure}[ht]
  \centering
  \includegraphics[width=\linewidth]{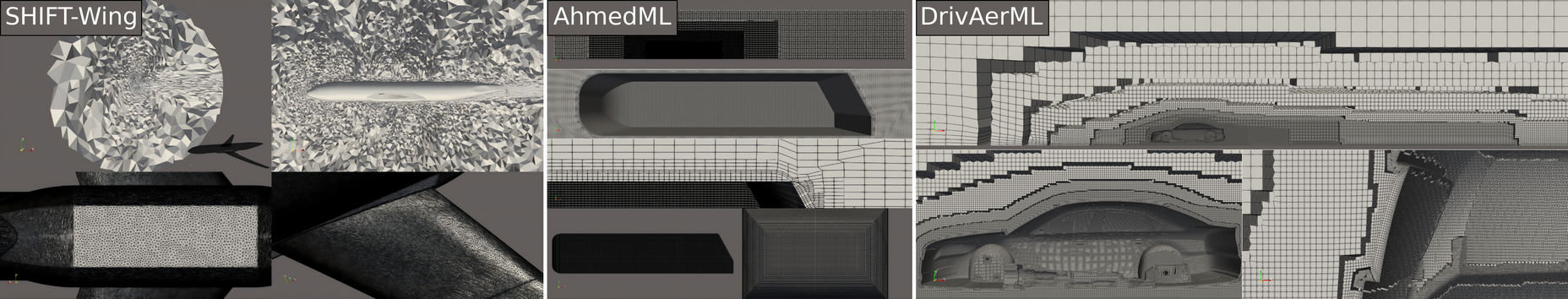}
  \caption{\textbf{Mesh Properties.}
  Meshes from three datasets with distinct spatial discretization characteristics.
  SHIFT-Wing uses a highly anisotropic unstructured polyhedral mesh with a spherical domain,
  AhmedML uses a relatively uniform structured mesh with a rectangular domain,
  and DrivAerML uses a highly non-uniform mesh with large variations in cell size across the rectangular domain.
  }
  \label{fig:Mesh_3datasets}
\end{figure}

\begin{figure}[ht]
  \centering
  \includegraphics[width=\linewidth]{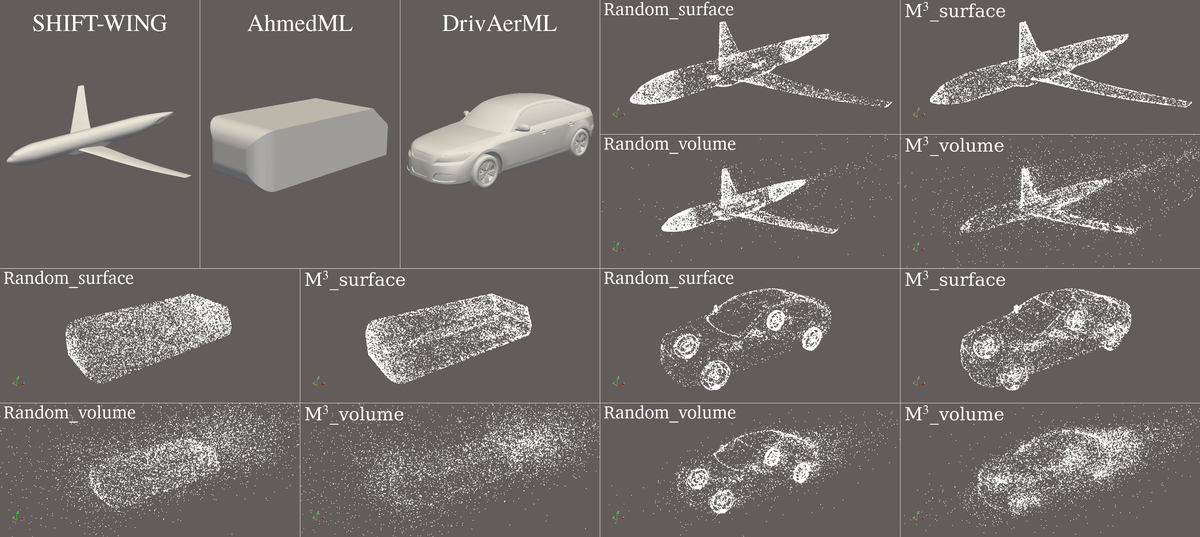}
  \caption{\textbf{Sampling comparison (Random vs.\ M$^3$, 8192 points).}
  Random sampling inherits the density bias of the original simulation data, whereas M$^3$
  constructs a physically informed distribution that mitigates excessive concentration in
  over-resolved regions and captures richer physical phenomena. Additional visualizations are
  provided in Appendix~\ref{sec:drivaerml-step3}.}
  \label{fig:partition-sampling-8192}
\end{figure}

\noindent\textbf{Structure of the Reshaped Training Measure.}
Figure~\ref{fig:partition-sampling-8192} visualizes sampling patterns under the same budget for
random subsampling and the M$^3$ pipeline. Random subsampling largely preserves the original CFD
mesh distribution and thus inherits its density-induced bias, systematically under-sampling
low-density regions where critical flow structures may be missed or underrepresented. Once lost
through sampling, such information is no longer explicitly represented in the training data,
limiting the achievable prediction fidelity across the physical domain. In contrast, M$^3$ actively
reshapes the spatial distribution, yielding more structured sampling on relatively uniform meshes
(AhmedML), multi-scale coverage on highly non-uniform meshes (DrivAerML), and balanced budget
allocation under anisotropic refinement patterns (SHIFT-Wing). This induces a different training
measure, yielding a more balanced distribution of supervision candidates across the domain.

\newcommand{\hlul}[1]{\begingroup\setlength{\fboxsep}{0.6pt}\colorbox{gray!30}{\strut #1}\endgroup}

\begin{table}[ht]
\footnotesize
\setlength{\tabcolsep}{4pt}

\noindent
\begin{minipage}[H]{0.4\textwidth}
\begingroup
\caption{\textbf{Full-mesh evaluation across the datasets.} Models with identical capacity are trained under the same 10\% data budget (M$^3$ vs.\ R, where R denotes uniform random sampling) and evaluated on the full mesh. AhmedML MSE values are scaled by $10^{-6}$ for readability.}
\label{tab:Evaluation-metrics-on-full-resolution-targets}
\endgroup
\end{minipage}
\hfill
\raisebox{-3.5pt}{%
\begin{minipage}[t]{0.58\textwidth}
\centering
\setlength{\tabcolsep}{3pt}
\resizebox{\linewidth}{!}{
\begin{tabular}{p{1.2cm}|*{8}{c}}
\toprule
 & \multicolumn{8}{c}{\textbf{DrivAerML}} \\
\cmidrule(lr){2-9}
\textbf{Metrics} & \multicolumn{2}{c}{$p_s$} & \multicolumn{2}{c}{$\boldsymbol{\tau}$}
  & \multicolumn{2}{c}{$p_v$} & \multicolumn{2}{c}{$\mathbf{u}$} \\
\cmidrule(lr){2-3}\cmidrule(lr){4-5}\cmidrule(lr){6-7}\cmidrule(lr){8-9}
& \textbf{R} & \textbf{M$^3$} & \textbf{R} & \textbf{M$^3$} & \textbf{R} & \textbf{M$^3$}
  & \textbf{R} & \textbf{M$^3$} \\
\midrule
\textbf{MSE} & 204.2 & \hlul{191.1} & \hlul{0.01810} & 0.01846 & \hlul{0.0011611} & 0.0044956 & \hlul{0.777} & 2.419 \\
\textbf{$L_2^{\mathrm{rel}}$ (\%)} & 3.94 & \hlul{3.82} & \hlul{7.65} & 7.76 & \hlul{6.37} & 12.73
  & \hlul{5.67} & 10.23 \\
\cmidrule(lr){1-9}
\textbf{MSE$_{\omega}$} & 119.9 & \hlul{91.8} & 0.02581 & \hlul{0.02025} & 0.0000839
  & \hlul{0.0000234} & 0.13683 & \hlul{0.04132} \\
\textbf{$L_{2,\omega}^{\mathrm{rel}}$ (\%)} & 4.36 & \hlul{3.80} & 7.38 & \hlul{6.53} & 0.91
  & \hlul{0.48} & 0.95 & \hlul{0.52} \\
\bottomrule
\end{tabular}
}
\end{minipage}
}
\vspace{0.9em}

\resizebox{\linewidth}{!}{
\begin{tabular}{p{1.2cm}|*{8}{c}|*{8}{c}}
\toprule
 & \multicolumn{8}{c|}{\textbf{AhmedML}} & \multicolumn{8}{c}{\textbf{SHIFT-Wing}} \\
\cmidrule(lr){2-9}\cmidrule(lr){10-17}
\textbf{Metrics}
& \multicolumn{2}{c}{$p_s$} & \multicolumn{2}{c}{$\boldsymbol{\tau}$} & \multicolumn{2}{c}{$p_v$} &
\multicolumn{2}{c|}{$\mathbf{u}$}
& \multicolumn{2}{c}{$p_s$} & \multicolumn{2}{c}{$\boldsymbol{\tau}$} & \multicolumn{2}{c}{$p_v$}
  & \multicolumn{2}{c}{$\mathbf{u}$} \\
\cmidrule(lr){2-3}\cmidrule(lr){4-5}\cmidrule(lr){6-7}\cmidrule(lr){8-9}
\cmidrule(lr){10-11}\cmidrule(lr){12-13}\cmidrule(lr){14-15}\cmidrule(lr){16-17}
& \textbf{R} & \textbf{M$^3$} & \textbf{R} & \textbf{M$^3$} & \textbf{R} & \textbf{M$^3$} & \textbf{R} & \textbf{M$^3$}
& \textbf{R} & \textbf{M$^3$} & \textbf{R} & \textbf{M$^3$} & \textbf{R} & \textbf{M$^3$}
  & \textbf{R} & \textbf{M$^3$} \\
\midrule
\textbf{MSE}
& 115.3 & \hlul{62.8} & 0.00338 & \hlul{0.00319} & 671 & \hlul{358} & 254 & \hlul{123}
& 7546 & \hlul{6334} & \hlul{9.830} & 9.831 & 39572 & \hlul{35593} & 1192 & \hlul{978} \\
\textbf{$L_2^{\mathrm{rel}}$ (\%)}
& 3.12 & \hlul{2.64} & \hlul{3.90} & 3.97 & 2.13 & \hlul{1.84} & 2.06 & \hlul{1.73}
& 0.357 & \hlul{0.338} & \hlul{18.91} & 18.91 & 0.808 & \hlul{0.774} & 43.1 & \hlul{39.1} \\
\cmidrule(lr){1-17}
\textbf{MSE$_{\omega}$}
& 110.1 & \hlul{61.2} & \hlul{0.00908} & 0.00959 & 89.2 & \hlul{31.0} & 61.7 & \hlul{20.6}
& 7905 & \hlul{5795} & 19.597 & \hlul{19.518} & 2175 & \hlul{776} & 28.50 & \hlul{1.009} \\
\textbf{$L_{2,\omega}^{\mathrm{rel}}$ (\%)}
& 3.07 & \hlul{2.61} & \hlul{3.48} & 3.80 & 0.613 & \hlul{0.455} & 0.510 & \hlul{0.376}
& 0.363 & \hlul{0.317} & 17.71 & \hlul{17.67} & 0.19 & \hlul{0.11} & 1.92 & \hlul{0.39} \\
\bottomrule
\end{tabular}
}
\end{table}

\noindent\textbf{Evaluation under Different Measures.}
We evaluate model performance by full-mesh inference and report mean errors over 50 held-out
cases across three random seeds under both equal-weight and physics-weighted metrics. The latter
weights each point by the area or volume of its associated mesh element, so that its contribution
scales with its physical extent. Accordingly, we report MSE and relative $L_2$ errors, along with
their physics-weighted variants
\begin{equation}
\label{eq:exp-rel-l2-weighted}
\mathrm{MSE}_\omega
=
\frac{\sum_{i\in\mathcal{I}}\lVert \mathbf{y}_i-\hat{\mathbf{y}}_i\rVert_2^2\,\omega_i}{\sum_{i\in\mathcal{I}}\omega_i}\,,
\qquad
L_{2,\omega}^{\mathrm{rel}}
=
\frac{\sqrt{\sum_{i\in\mathcal{I}}\lVert \mathbf{y}_i-\hat{\mathbf{y}}_i\rVert_2^2\,\omega_i}}
{\sqrt{\sum_{i\in\mathcal{I}} \lVert \mathbf{y}_i\rVert_2^2\,\omega_i}}\,,
\end{equation}
with nonnegative weights $\omega_i$.
Here $\mathbf{y}_i$ and $\hat{\mathbf{y}}_i := f_\theta(\mathbf{x}_i)$ denote the ground truth and
the prediction at entity $i$, and $\lVert\cdot\rVert_2$ denotes the Euclidean norm. The index set
$\mathcal{I}$ enumerates all evaluation entities on the full mesh, with weights $\omega_i=\Delta
S_i$ on the surface set $\mathcal{S}$ and $\omega_i=\Delta V_i$ on the volume set $\mathcal{V}$. The
physics-weighted metrics thus act as a quadrature approximation to the continuous-domain error on
the full-mesh discretization. Setting $\omega_i\equiv 1$ recovers the equal-weight metrics.
Together, these two metrics evaluate the same model under different measures (see
Appendix~\ref{sec:app-measure-eval}), providing complementary views for comparing prediction error.

\noindent As shown in Table~\ref{tab:Evaluation-metrics-on-full-resolution-targets}, M$^3$ improves
both equal-weight and physics-weighted $L_2^{\mathrm{rel}}$ on DrivAerML $p_s$ (1.03$\times$;
1.15$\times$), AhmedML $p_s$ (1.18$\times$; 1.18$\times$), $p_v$ (1.16$\times$; 1.35$\times$),
$\mathbf{u}$ (1.19$\times$; 1.36$\times$), and SHIFT-Wing $p_s$ (1.06$\times$; 1.15$\times$), $p_v$
(1.04$\times$; 1.73$\times$), $\mathbf{u}$ (1.10$\times$; 4.92$\times$). These paired gains indicate
improved accuracy under both the equal-weight and physics-weighted evaluation metrics, while even
the weaker channel $\boldsymbol{\tau}$ remains close to random sampling when M$^3$ is not best.
Consistent with the top-5\% error maps in Figure~\ref{fig:error-vis-three-datasets}, the R-trained
model exhibits higher error magnitudes in AhmedML and broader spatial error distributions in
SHIFT-Wing. These results suggest that, for relatively uniform and anisotropic data, the M$^3$
training measure is overall better aligned with the underlying physical-field learning objectives
than the random mesh sampling measure.

\begin{figure}[ht]
  \centering
  \includegraphics[width=\linewidth]{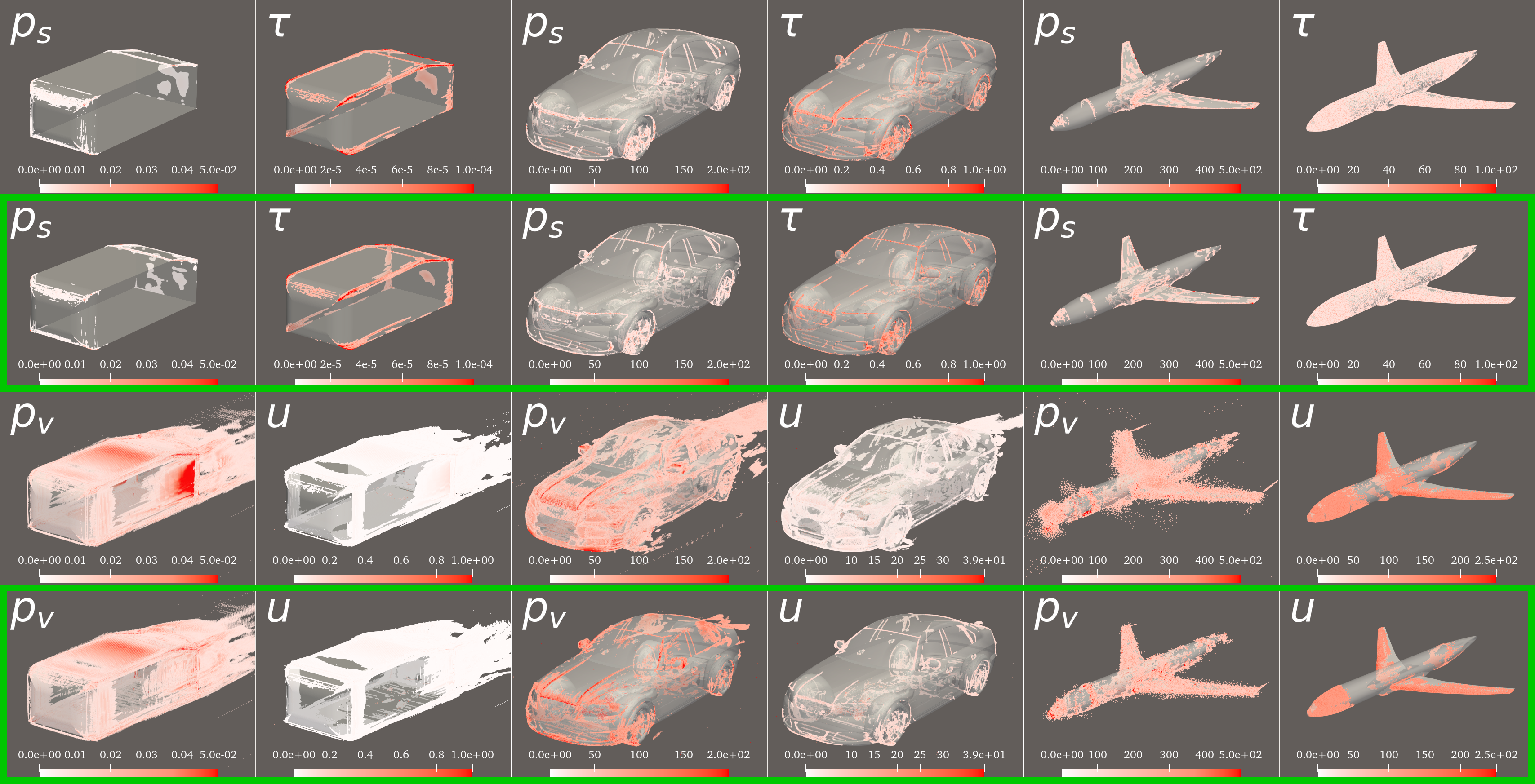}
  \caption{\textbf{Top-5\% error maps (M$^3$ highlighted by green boxes).}
  For each model and field, we keep the 5\% of points with the largest absolute error on the same
  full-mesh query. White denotes near-zero error, while deeper red indicates larger absolute
  error. R-trained models show both broader spatial error support and larger error magnitudes than
  M$^3$. See more corresponding results in Appendix~\ref{sec:app-full-resolution}.}
  \label{fig:error-vis-three-datasets}
\end{figure}

\begin{figure}[ht]
  \centering
  \includegraphics[width=\linewidth]{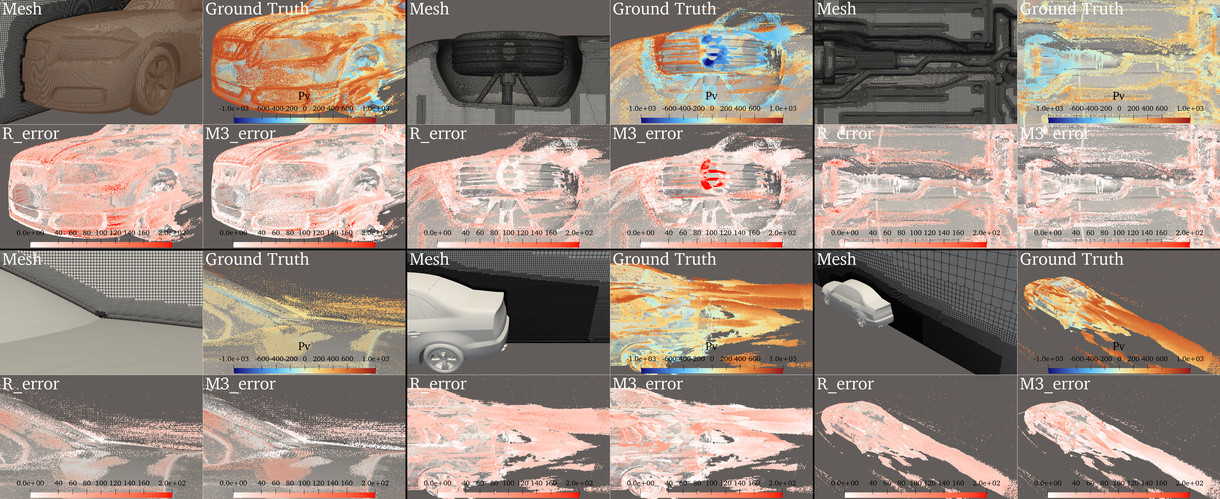}
  \caption{\textbf{Local views of high-error regions in DrivAerML.}
  The displayed query points are selected as the union of the per-model top-5\% error points. M$^3$
  exhibits higher errors than the R-trained model at extremely concentrated mesh regions
  (tire--ground interfaces), but lower errors across most of the spatial domain, including the
  boundary layer, strong separation zones, and the extended wake.}
  \label{fig:local-drivaer}
\end{figure}

On the highly non-uniform DrivAerML dataset ($\sim$160M), however, the two evaluation metrics
exhibit a pronounced inversion. Random sampling achieves lower errors under equal-weight evaluation,
whereas M$^3$ substantially reduces the corresponding physics-weighted errors. In particular, the
physics-weighted $L_2^{\mathrm{rel}}$ improves for $p_s$ ($4.36\%\rightarrow3.80\%$, $1.15\times$),
$\boldsymbol{\tau}$ ($7.38\%\rightarrow6.53\%$, $1.13\times$), $p_v$ ($0.91\%\rightarrow0.48\%$,
$1.90\times$), and $\mathbf{u}$ ($0.95\%\rightarrow0.52\%$, $1.83\times$). As also shown in
Figure~\ref{fig:error-vis-three-datasets}, the R-trained model exhibits broader high-error regions,
particularly in the wake, where the error coverage is substantially larger than that of M$^3$. In
contrast, M$^3$ produces smaller but slightly more concentrated high-error regions. This difference
reflects the fact that training measures define different Monte Carlo optimization objectives,
leading to distinct physical-field approximations. For random downsampling from the mesh measure,
although pointwise learning objectives are unbiased over the discrete point set, they induce a
mesh-dependent empirical measure in the continuous domain. High-density regions contribute more
samples and therefore receive greater effective weight in the empirical loss, leading to spatially
imbalanced learning biased toward dense areas. As a result, predictions tend to be more accurate in
such regions while exhibiting larger errors in sparse ones. Equal-weight metrics may further mask
this imbalance during metric aggregation, potentially overstating performance over the full physical
field.

The design of M$^3$ directly addresses this mesh-induced bias by redistributing the sampling measure
across different physical variation scales, preventing extremely dense mesh regions from dominating
the training budget. It provides a more balanced set of representative points within each batch
update, encouraging the model to capture the underlying field more broadly. The local comparisons in
Figure~\ref{fig:local-drivaer} further support this effect: at the extremely concentrated
tire--ground interfaces, the R-trained model achieves lower local errors than M$^3$, whereas M$^3$
achieves higher accuracy across most of the domain, particularly around flow separation and in the
wake. Since the tire--ground contact regions occupy only a small fraction of the physical domain and
contribute little to the overall force integration (Figure~\ref{fig:app-force-contrib}), emphasizing
them does not necessarily improve global field accuracy. Mitigating such mesh-induced over-concentration from empirical engineering refinement artifacts,
which may be a poor inductive bias for operator learning, is precisely the design goal of
M$^3$. These observations show that directly using the mesh-induced measure is suboptimal for a globally high-fidelity objective. Moreover, pointwise
evaluation alone may not fully reflect performance across the entire physical domain. Multiple
evaluation aspects should therefore be considered.

\begin{table}[ht]
\centering
\small
\caption{\textbf{Data-efficiency comparison on DrivAerML.}
Models are trained on subsets of varying sizes and evaluated on the full-mesh point set.
$\times x$ denotes the sampling factor, where $\times 1$ denotes inputs from the original full CFD
mesh. Shading marks the lowest error in each metric column.}
\label{tab:exp-scale-drivaerml}
\vspace{3pt}
\setlength{\tabcolsep}{2.0pt}
\resizebox{0.99\linewidth}{!}{
\begin{tabular}{@{}c|cc|cccc|cccc|cccc|cccc@{}}
\toprule
\textbf{Training} & \textbf{Surf.} & \textbf{Vol.}
& \multicolumn{4}{c|}{MSE} & \multicolumn{4}{c|}{MSE$_\omega$}
& \multicolumn{4}{c|}{$L_2^{\mathrm{rel}}$ (\%)} & \multicolumn{4}{c}{$L_{2,\omega}^{\mathrm{rel}}$ (\%)} \\
\cmidrule(lr){2-2}\cmidrule(lr){3-3}\cmidrule(lr){4-7}\cmidrule(lr){8-11}\cmidrule(lr){12-15}\cmidrule(lr){16-19}
\textbf{measure} & {\ttfamily\char`\~}9M & {\ttfamily\char`\~}160M
& $p_s$ & $\boldsymbol{\tau}$ & $p_v$ & $\mathbf{u}$
& $p_s$ & $\boldsymbol{\tau}$ & $p_v$ & $\mathbf{u}$
& $p_s$ & $\boldsymbol{\tau}$ & $p_v$ & $\mathbf{u}$
& $p_s$ & $\boldsymbol{\tau}$ & $p_v$ & $\mathbf{u}$ \\
\midrule
\textbf{R} & $\times 1$ & $\times 0.1$
& \hlul{146.6} & \hlul{0.01319} & \hlul{0.0009106} & \hlul{0.59307} & 98.6 & 0.02217 & 0.0000735 & 0.11781
& \hlul{3.30} & \hlul{6.51} & \hlul{5.61} & \hlul{4.92} & 3.95 & 6.86 & 0.85 & 0.88 \\
\textbf{R} & $\times 0.1$ & $\times 0.1$
& 204.2 & 0.01810 & 0.0011611 & 0.777 & 119.9 & 0.02581 & 0.0000839 & 0.13683
& 3.94 & 7.65 & 6.37 & 5.67 & 4.36 & 7.38 & 0.91 & 0.95 \\
\textbf{R} & $\times 0.01$ & $\times 0.01$
& 202.2 & 0.01775 & 0.0012043 & 0.78712 & 122.4 & 0.02640 & 0.0000881 & 0.13363
& 3.93 & 7.60 & 6.49 & 5.71 & 4.41 & 7.47 & 0.93 & 0.94 \\
\midrule
\textbf{M$^3$} & $\times 0.1$ & $\times 0.1$
& 191.1 & 0.01846 & 0.0044956 & 2.419 & \hlul{91.8} & \hlul{0.02025} & 0.0000234 & 0.04132
& 3.82 & 7.76 & 12.73 & 10.23 & \hlul{3.80} & \hlul{6.53} & 0.48 & 0.52 \\
\textbf{M$^3$} & $\times 0.1$ & $\times 0.01$
& 217.2 & 0.02084 & 0.0083998 & 5.155 & 106.3 & 0.02224 & \hlul{0.0000093} & \hlul{0.01046}
& 4.10 & 8.27 & 17.46 & 15.05 & 4.11 & 6.86 & \hlul{0.30} & \hlul{0.26} \\
\textbf{M$^3$} & $\times 0.01$ & $\times 0.01$
& 268.3 & 0.02741 & 0.0088041 & 5.134 & 113.6 & 0.02413 & 0.0000094 & 0.01100
& 4.56 & 9.48 & 17.87 & 15.00 & 4.27 & 7.17 & 0.30 & 0.27 \\
\bottomrule
\end{tabular}
}
\end{table}

\begin{figure}[ht]
  \centering
  \includegraphics[width=\linewidth]{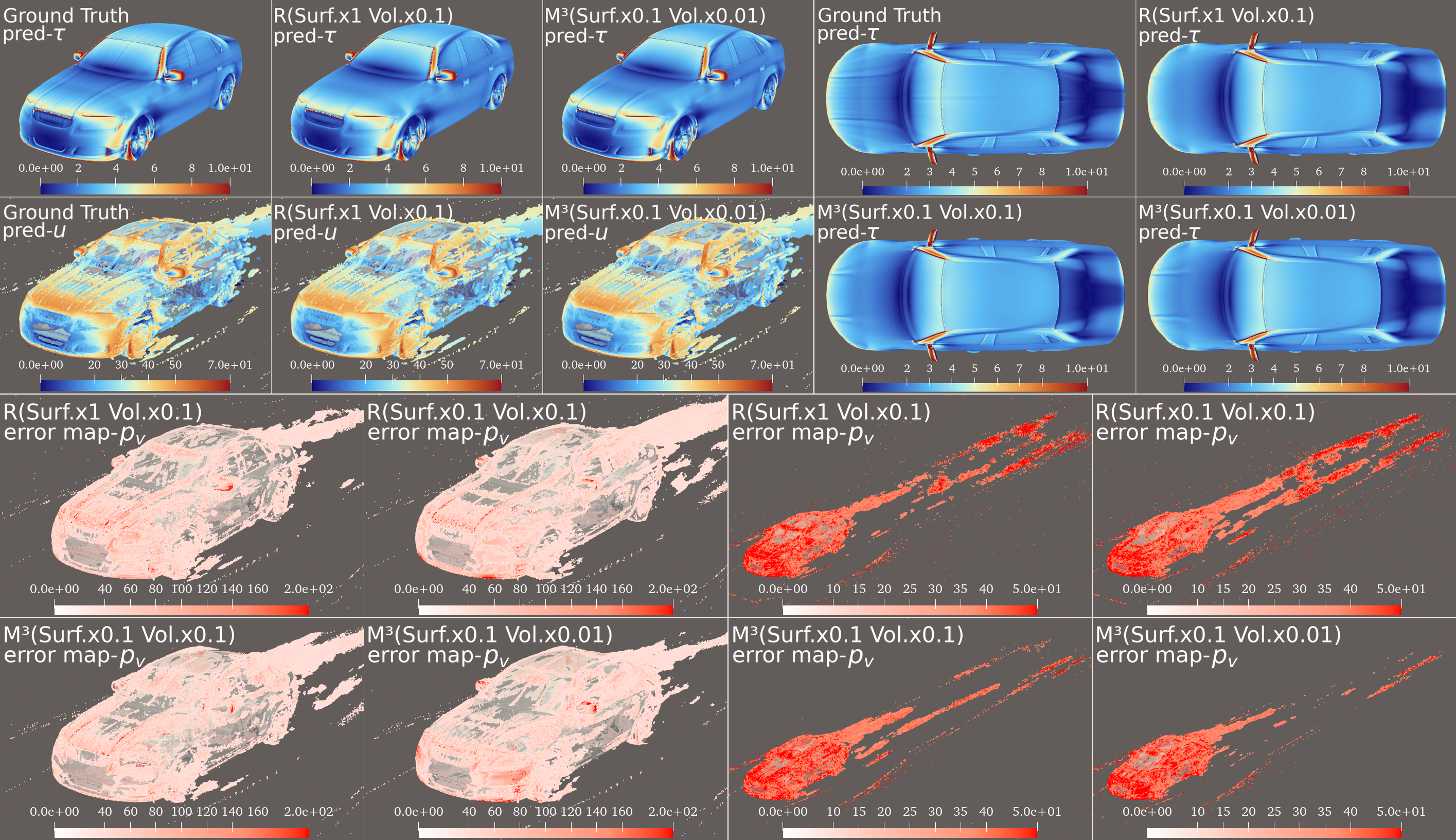}
  \caption{\textbf{Model performance under different training measures and budgets.}
  Blue: field predictions of wall shear stress and volume pressure.
  Pink: top-5\% volume pressure error regions under different models.
  Red: the same regions shown from a broader view with an adjusted color scale for clearer
  comparison. Results for different top-$p\%$ error maps are provided in
  Appendix~\ref{sec:app-full-resolution}.}
  \label{fig:supplement-split-lr}
\end{figure}

\noindent\textbf{Data Efficiency at Reduced Training Scales.} We further evaluate model performance
under different training budgets. As shown in Table~\ref{tab:exp-scale-drivaerml}, across all
physics-weighted metrics, models trained on M$^3$-preprocessed data consistently outperform those
trained on Random (mesh-inherited) measures at the same data scale. The advantage persists even with
substantially less volume data: M$^3$ (Surf.\ $\times 0.1$, Vol.\ $\times 0.01$) surpasses the
larger-budget Random baseline R(Surf.\ $\times 1$, Vol.\ $\times 0.1$) on physics-weighted volume
fields, with $L_{2,\omega}^{\mathrm{rel}}$ of $0.30$/$0.26$ for $p_v$/$\mathbf{u}$, compared with
$0.85$/$0.88$. Relative to R(Surf.\ $\times 0.1$, Vol.\ $\times 0.1$), the same M$^3$ setting
reduces the weighted relative error from $0.91$ to $0.30$ ($3.0\times$) for $p_v$ and from $0.95$ to
$0.26$ ($3.7\times$) for $\mathbf{u}$, with up to $9.0\times$ and $13.1\times$ reductions in
weighted squared error, respectively. Under equal-weight $L_2^{\mathrm{rel}}$, however, M$^3$
performs worse on the volume fields, revealing opposing trends under mesh-based versus
physics-weighted evaluation on highly non-uniform meshes. See Appendix~\ref{app:fidelity} for finer quantitative analyses and detailed mechanisms.

To examine whether this inversion reflects a genuine difference in the predicted physical fields
rather than an artifact of the evaluation measure, Figure~\ref{fig:supplement-split-lr} compares
spatial error patterns using a shared threshold defined as the lower edge of the top-5\% error set
of the R(Surf.\ $\times 0.1$, Vol.\ $\times 0.1$)-trained model and applied identically to all
models. Only points exceeding this threshold are shown. Compared with the R(Surf.\ $\times 1$, Vol.\
$\times 0.1$)-trained model, the M$^3$(Surf.\ $\times 0.1$, Vol.\ $\times 0.01$)-trained model still
recovers more physically consistent flow structures. On the surface, the characteristic high
wall-shear-stress streaks downstream of the nose stagnation point remain sharp under M$^3$ and
closely match the ground truth, whereas they become blurred under the R-trained model. In the
volume, M$^3$ exhibits fewer high-error regions both near the wall and throughout the wake,
providing spatial evidence of improved prediction accuracy and greater consistency with the ground
truth despite using substantially less training data.


\section{Discussion}
\label{Discussion}
We demonstrate that measure-induced bias governs the effective objective learned by neural
operators. Random sampling inherits discretization bias, resulting in spatially imbalanced learning
across the domain. M$^3$ reshapes the data distribution to enhance multi-scale spatial coverage,
yielding more balanced learning and improved performance with fewer samples. These findings suggest
that the training measure remains critical even with abundant data and compute. It is key to
reducing bias and ensuring physically consistent approximation of the underlying continuum.

A limitation of M$^3$ is that it introduces several hyperparameters governing the sampling behavior
through its partitioning, stratification, and allocation procedure. While this design provides
flexibility in shaping the data distribution, it also enlarges the space of induced measures, making
it nontrivial to determine which configuration best reflects the underlying physical
characteristics. As a result, appropriate parameter ranges may vary across datasets, and some
dataset-specific calibration is typically required. A practical strategy is to perform a small pilot
run to identify suitable settings before processing the full dataset. As such preprocessing is
generally unavoidable in large-scale simulation workflows, this limited calibration adds little
overhead while offering an opportunity to improve data quality. More broadly, by reducing data storage and training costs while preserving or improving
operator-learning fidelity, M$^3$ may offer a data-centric approach to scaling physical foundation
models towards larger and more diverse simulation datasets.


\begin{ack}
This work was supported by JSPS KAKENHI Grant Numbers JP25H01454 and JP26K02968, JST PRESTO
Grant Number JPMJPR24T6, and the Qiushi Eagle Program of Zhejiang University. We thank Naoya
Takeishi for careful guidance throughout this project, and for many in-depth discussions on the
method and the theory. We also thank Xiaowen Song for continued support.
The authors declare no competing interests.
\end{ack}

\bibliographystyle{plainnat}
\bibliography{references}

\newpage
\appendix
\newcommand{\tabhl}[1]{\colorbox{gray!35}{#1}\hspace{-\fboxsep}}
\newcommand{\tabup}{\textcolor{red}{\uparrow}}
\newcommand{\tabdn}{\textcolor{green!60!black}{\downarrow}}

\section{Measure-Theoretic Formulation of Subsampled Learning}
\label{sec:app-measure-theory}

\subsection{Training Measures under Subsampling}
\label{sec:app-training-measure}

Subsampling induces a discrete training measure $\mu_S$ over the labeled tuples
$(\mathbf{x}_i,\mathbf{y}_i)$ of Section~\ref{sec:problem-setup}, which fully determines the
empirical risk.
Let $\mathbf{y}:\mathbb{R}^3\rightarrow\mathbb{R}^d$ denote the target field, so that
$\mathbf{y}_i=\mathbf{y}(\mathbf{x}_i)$, and let $f_\theta:\mathbb{R}^3\rightarrow\mathbb{R}^d$ be
the surrogate model.
Since each label is determined by its location, we write measures over spatial locations, e.g.,
$\delta_{\mathbf{x}_i}$ in place of $\delta_{(\mathbf{x}_i,\mathbf{y}_i)}$.
Training minimizes
\begin{equation}
\label{eq:risk-squared-app}
\mathcal{R}_{\mu_S}(\theta)
:=
\mathbb{E}_{\mathbf{x}\sim\mu_S}
\bigl[\|f_\theta(\mathbf{x})-\mathbf{y}(\mathbf{x})\|_2^2\bigr].
\end{equation}
In terms of the training index set $\mathcal{I}_{\mathrm{tr}}$, the training measure is
\begin{equation}
\label{eq:mu-S-multiset-app}
\mu_S \;=\; \frac{1}{|\mathcal{I}_{\mathrm{tr}}|}\sum_{i\in\mathcal{I}_{\mathrm{tr}}}\delta_{\mathbf{x}_i}.
\end{equation}
For M$^3$, $\mathcal{I}_{\mathrm{tr}}$ contains no repeated indices, since Step~3 samples without
replacement within each cell. More generally, sampling with replacement or per-sample loss weights
yield the weighted representation
\begin{equation}
\label{eq:mu-S-weighted-app}
\mu_S = \sum_{i\in\mathcal{I}_{\mathrm{tr}}} w_i\,\delta_{\mathbf{x}_i},
\qquad w_i\ge 0,\quad \sum_{i\in\mathcal{I}_{\mathrm{tr}}} w_i=1,
\end{equation}
where $\mathcal{I}_{\mathrm{tr}}$ denotes the set of distinct indices and $w_i$ encodes sampling
multiplicities or loss weights, and Equation~\eqref{eq:mu-S-multiset-app} corresponds to $w_i\equiv
1/|\mathcal{I}_{\mathrm{tr}}|$. Under squared loss,
\begin{equation}
\label{eq:risk-empirical-weighted}
\mathcal{R}_{\mu_S}(\theta)
=
\int \|f_\theta(\mathbf{x})-\mathbf{y}(\mathbf{x})\|_2^2\,d\mu_S(\mathbf{x})
=
\sum_{i\in\mathcal{I}_{\mathrm{tr}}} w_i\,\|f_\theta(\mathbf{x}_i)-\mathbf{y}(\mathbf{x}_i)\|_2^2.
\end{equation}
Thus, subsampling directly defines $\mu_S$, and different sampling strategies induce different empirical objectives.
Intuitively, over-allocating mass to some regions biases the solution toward them, while
under-sampled regions remain under-optimized, so changing $\mu_S$ changes how approximation error is
distributed across the domain.

In attention-based surrogates, the attention weight $\mu_{\mathrm{attn}}(i\mid\mathbf{x})$ on the
$i$-th sampled token for a query location $\mathbf{x}$ defines a conditional distribution over the
sampled indices and therefore assigns no mass outside $\mathrm{supp}(\mu_S)$.
Attention can thus aggregate information only from the sampled set, whereas predictions outside this
set rely on the generalization of $f_\theta$.
Separately, $\mu_S$ determines which locations enter the empirical risk: locations outside
$\mathrm{supp}(\mu_S)$ contribute no loss and receive little gradient signal, which can lead to
localized errors.

\subsection{Evaluation Measures}
\label{sec:app-measure-eval}

Comparisons of performance are not meaningful unless evaluation is performed under a common evaluation measure.
Suppose evaluation is defined with respect to a measure $\mu_{\mathrm{eval}}$ over spatial
locations, and let $\theta_S^*\in\arg\min_\theta\mathcal{R}_{\mu_S}(\theta)$ denote the parameters
obtained by training on $\mu_S$.
This induces a \emph{measure mismatch}: training minimizes $\mathcal{R}_{\mu_S}$, while reported errors correspond to
\begin{equation}
\mathcal{R}_{\mu_{\mathrm{eval}}}(\theta_S^*)
:=
\mathbb{E}_{\mathbf{x}\sim\mu_{\mathrm{eval}}}
\bigl[\|f_{\theta_S^*}(\mathbf{x})-\mathbf{y}(\mathbf{x})\|_2^2\bigr].
\end{equation}
Different evaluation protocols correspond to different choices of $\mu_{\mathrm{eval}}$.
On a discrete CFD mesh with entities indexed by $\mathcal{I}$, uniform node averaging induces the
mesh-based measure
$\mu_{\mathrm{mesh}}=\frac{1}{|\mathcal{I}|}\sum_{i\in\mathcal{I}}\delta_{\mathbf{x}_i}$, whereas
physically meaningful quantities are defined with respect to volume or surface measures.
With mesh-element volumes $\Delta V_i$ on the volume set $\mathcal{V}$ and face areas $\Delta S_i$
on the surface set $\mathcal{S}$, these correspond to
\[
\mu_{\mathrm{vol}}=\sum_{i\in\mathcal{V}}\frac{\Delta V_i}{\sum_{j\in\mathcal{V}}\Delta V_j}\,\delta_{\mathbf{x}_i},
\qquad
\mu_{\mathrm{surf}}=\sum_{i\in\mathcal{S}}\frac{\Delta S_i}{\sum_{j\in\mathcal{S}}\Delta S_j}\,\delta_{\mathbf{x}_i}.
\]
For each field, the physics-weighted metric $\mathrm{MSE}_\omega$ in
Equation~\eqref{eq:exp-rel-l2-weighted} with $\omega_i=\Delta V_i$ (resp.\ $\omega_i=\Delta S_i$) is
exactly the risk under $\mu_{\mathrm{vol}}$ (resp.\ $\mu_{\mathrm{surf}}$), while $\omega_i\equiv 1$
recovers the risk under $\mu_{\mathrm{mesh}}$.
Equal-weight and physics-weighted metrics thus evaluate the same model under different measures,
which is why their rankings can disagree on strongly non-uniform meshes.

In practice, if $\mathbf{x}_1,\dots,\mathbf{x}_n$ are drawn i.i.d.\ from a sampling distribution
$\mu_{\mathrm{samp}}$ while evaluation is defined with respect to $\mu_{\mathrm{eval}}$, importance
weighting yields
\begin{equation}
    \hat{\mathcal{R}}^{\mathrm{IS}}_n(\theta)
    :=
    \frac{1}{n}\sum_{i=1}^{n}
    \frac{d\mu_{\mathrm{eval}}}{d\mu_{\mathrm{samp}}}(\mathbf{x}_i)\,
    \|f_\theta(\mathbf{x}_i)-\mathbf{y}(\mathbf{x}_i)\|_2^2,
\end{equation}
where $\mu_{\mathrm{eval}}\ll\mu_{\mathrm{samp}}$. The estimator is then unbiased for
$\mathcal{R}_{\mu_{\mathrm{eval}}}(\theta)$.
Importance weighting, however, changes only how the sampled points are weighted, not which locations
are supervised, so for the realized measure $\mu_S$ it cannot recover regions outside its finite
support.
This motivates designing the support of $\mu_S$ itself, and the loss-reweighting baseline in
Appendix~\ref{app:baselines-ahmed} illustrates this distinction empirically.

\subsection{Structured TV Decomposition}
\label{sec:app-measure-shift}

We analyze how the cell-level training allocation induced by Step~3 may deviate from the target
$\mu^\star$ in Equation~\eqref{eq:mu-star-ms}, quantify the deviation by the total variation (TV),
and relate the two stages of Step~3 to the two terms of a TV decomposition.

\paragraph{TV controls the risk gap.}
The TV between probability measures $\mu$ and $\nu$ is
\[
\mathrm{TV}(\mu,\nu):=\sup_A |\mu(A)-\nu(A)|.
\]
It is a natural discrepancy here because it directly bounds the risk gap:
\begin{proposition}
\label{thm:sampling-risk-gap}
    Assume the pointwise loss $\mathcal{L}_\theta(\mathbf{x}) :=
    \|f_\theta(\mathbf{x})-\mathbf{y}(\mathbf{x})\|_2^2$ satisfies $0\le
    \mathcal{L}_\theta(\mathbf{x})\le M$ on the point set.
    Then for any probability measures $\mu,\nu$ on the point set and any $\theta$,
    \[
    \bigl|\mathcal{R}_{\mu}(\theta)-\mathcal{R}_{\nu}(\theta)\bigr|
    \le
    M\,\mathrm{TV}(\mu,\nu).
    \]
\end{proposition}

\begin{proof}
    Since $\mu$ and $\nu$ have equal total mass,
    $\mathcal{R}_\mu(\theta)-\mathcal{R}_\nu(\theta)=\int (\mathcal{L}_\theta-M/2)\,d(\mu-\nu)$.
    As $|\mathcal{L}_\theta-M/2|\le M/2$, the right-hand side is bounded by
    $\frac{M}{2}\|\mu-\nu\|_1=M\,\mathrm{TV}(\mu,\nu)$.
\end{proof}

\paragraph{Reduction to cell-level allocation.}
The risk is defined on points, whereas Step~3 allocates quotas to cells.
On the finite cell set~$\mathcal{C}$, a measure $\mu$ is identified with its mass vector
$p=(p_c)_{c\in\mathcal{C}}$, where $p_c:=\mu(\{c\})$ and $\sum_c p_c=1$.
Let $U_c:=\frac{1}{N_c}\sum_{i\in\mathcal{P}_c}\delta_{\mathbf{x}_i}$ denote the uniform measure
over the points of cell $c$, and define the \emph{lift} of a cell-level measure $\mu$ as
\begin{equation}
\label{eq:lift}
\tilde{\mu}:=\sum_{c\in\mathcal{C}}p_c\,U_c.
\end{equation}
Since the point sets $\{\mathcal{P}_c\}$ are disjoint, lifting preserves TV, i.e.,
$\mathrm{TV}(\tilde{\mu},\tilde{\nu})=\mathrm{TV}(\mu,\nu)$, and $\tilde{\mu}^\star$ spreads the
mass $\alpha_\ell/|\mathcal{C}_\ell|$ of each cell uniformly over its points.
For M$^3$, as in Section~\ref{sec:method-sampling}, the realized budget is
$m':=\min\bigl(m,\sum_{c\in\mathcal{C}}\tilde{N}_c\bigr)$ with $|\mathcal{I}_{\mathrm{tr}}|=m'$.
Let $\bar{\mu}_S$ denote the \emph{cell-level} measure induced by the fixed quotas $\{q_c\}$,
\[
\bar{\mu}_S(\{c\})=\frac{q_c}{m'},
\]
and write $\tilde{\mu}_S$ for its lift under Equation~\eqref{eq:lift}.
The point-level training measure $\mu_S$ is obtained by sampling $q_c$ points uniformly without
replacement within each cell, so each $i\in\mathcal{P}_c$ is selected with probability $q_c/N_c$.
Hence, for any fixed $\theta$, the expected training risk over this within-cell randomness equals
the risk under the lift of the cell-level measure:
\begin{equation}
\label{eq:within-cell-unbiased}
\mathbb{E}\bigl[\mathcal{R}_{\mu_S}(\theta)\bigr]
=
\sum_{c\in\mathcal{C}}\frac{q_c}{m'}\cdot\frac{1}{N_c}\sum_{i\in\mathcal{P}_c}\mathcal{L}_\theta(\mathbf{x}_i)
=
\mathcal{R}_{\tilde{\mu}_S}(\theta).
\end{equation}
Combining Equation~\eqref{eq:within-cell-unbiased}, Proposition~\ref{thm:sampling-risk-gap}, and the
TV preservation of lifting,
\begin{equation}
\label{eq:expected-risk-gap}
\bigl|\mathbb{E}\bigl[\mathcal{R}_{\mu_S}(\theta)\bigr]-\mathcal{R}_{\tilde{\mu}^\star}(\theta)\bigr|
=
\bigl|\mathcal{R}_{\tilde{\mu}_S}(\theta)-\mathcal{R}_{\tilde{\mu}^\star}(\theta)\bigr|
\le
M\,\mathrm{TV}(\tilde{\mu}_S,\tilde{\mu}^\star)
=
M\,\mathrm{TV}(\bar{\mu}_S,\mu^\star).
\end{equation}
The middle inequality is the risk-gap bound on point measures; the last equality follows because lifting preserves TV.
Thus, controlling the cell-level discrepancy $\mathrm{TV}(\bar{\mu}_S,\mu^\star)$ controls the
expected training risk relative to the lifted target, and the allocation problem can be analyzed
through the quotas $\{q_c\}$.

\paragraph{Decomposition of the cell-level TV.}
An idealized formulation of the sampling design under the realized budget $m'$ is
\[
    \min_{\mu \in \mathcal{M}_{m'}}\ \mathrm{TV}(\mu, \mu^\star),
\]
where $\mathcal{M}_{m'}$ denotes the cell-level measures $\mu(\{c\})=q_c/m'$ induced by integer
quotas satisfying $0\le q_c\le\tilde{N}_c$ and $\sum_c q_c=m'$.
Rather than optimizing this objective directly, we separate it into inter-scale mass gaps
$|\mu(\mathcal{C}_\ell)-\alpha_\ell|$, controlled by the scale budgets in
Equation~\eqref{eq:level-capacity}, and intra-scale imbalance within each $\mathcal{C}_\ell$,
controlled by the cell quotas in Equation~\eqref{eq:cell-quota-constraints}.
Proposition~\ref{lem:cell-tv-decomp} establishes this decomposition, and the two stages of Step~3
address the corresponding terms.
For cell-level measures $\mu$ and $\nu$ with mass vectors $p$ and $r=(r_c)_{c\in\mathcal{C}}$, the
variational definition of TV reduces to the discrete identity
\begin{equation}
\label{eq:tv-l1-leaves}
\mathrm{TV}(\mu,\nu)=\frac{1}{2}\sum_{c\in\mathcal{C}}|p_c-r_c|.
\end{equation}
\begin{proposition}
    \label{lem:cell-tv-decomp}
    Let $\mu$ be a probability measure on $\mathcal{C}$ with mass vector $p = (p_c)_{c \in \mathcal{C}}$.
    Then,
    \begin{equation}
        \label{eq:tv-two-term}
        \sum_{\ell\in\Lambda}\bigl|\mu(\mathcal{C}_\ell)-\alpha_\ell\bigr|
        \;\le\;
        2\,\mathrm{TV}(\mu,\mu^\star)
        \;\le\;
        \sum_{\ell\in\Lambda}\bigl|\mu(\mathcal{C}_\ell)-\alpha_\ell\bigr|
        +
        \sum_{\ell\in\Lambda}\sum_{c\in\mathcal{C}_\ell}\Bigl|p_c-\frac{\mu(\mathcal{C}_\ell)}{|\mathcal{C}_\ell|}\Bigr|.
    \end{equation}
\end{proposition}

\begin{proof}
    From the definition in Equation~\eqref{eq:mu-star-ms}, $\mu^\star$ assigns mass $\alpha_\ell$ to
    each stratum and is uniform within strata, i.e., $\mu^\star(\{c\}) =
    \alpha_\ell/|\mathcal{C}_\ell|$ for $c\in\mathcal{C}_\ell$.
    By Equation~\eqref{eq:tv-l1-leaves},
    \[
    2\,\mathrm{TV}(\mu,\mu^\star)=\sum_{\ell\in\Lambda}\sum_{c\in\mathcal{C}_\ell}\Bigl|p_c-\frac{\alpha_\ell}{|\mathcal{C}_\ell|}\Bigr|.
    \]
    \emph{Lower bound.} For each $\ell$, the triangle inequality gives
    $\sum_{c\in\mathcal{C}_\ell}|p_c-\alpha_\ell/|\mathcal{C}_\ell||\ge|\sum_{c\in\mathcal{C}_\ell}(p_c-\alpha_\ell/|\mathcal{C}_\ell|)|=|\mu(\mathcal{C}_\ell)-\alpha_\ell|$.
    Summing over $\ell$ gives the left inequality.

    \emph{Upper bound.} Fix $\ell$ and define $\bar p:=\mu(\mathcal{C}_\ell)/|\mathcal{C}_\ell|$.
    By the triangle inequality,
    \[
    \Bigl|p_c-\frac{\alpha_\ell}{|\mathcal{C}_\ell|}\Bigr|
    \le
    |p_c-\bar p|
    +
    \Bigl|\bar p-\frac{\alpha_\ell}{|\mathcal{C}_\ell|}\Bigr|.
    \]
    Summing over $c\in\mathcal{C}_\ell$ gives
    \[
    \sum_{c\in\mathcal{C}_\ell}\Bigl|p_c-\frac{\alpha_\ell}{|\mathcal{C}_\ell|}\Bigr|
    \le
    \sum_{c\in\mathcal{C}_\ell}|p_c-\bar p|
    +
    \bigl|\mu(\mathcal{C}_\ell)-\alpha_\ell\bigr|,
    \]
    and summing over $\ell$ gives the right inequality.
\end{proof}

\paragraph{Allocation induced by the decomposition.}
For $\mu=\bar{\mu}_S$, we have $p_c=q_c/m'$ and $\bar{\mu}_S(\mathcal{C}_\ell)=m_\ell/m'$, so the
inter-scale term depends only on the scale budgets $\{m_\ell\}$, while, given these budgets, the
intra-scale term depends only on the cell quotas $\{q_c\}$.
The two stages of Step~3 address these terms in turn.
Stage~1 of Algorithm~\ref{alg:cell_equal_level_aware_morton_sampling} allocates $\{m_\ell\}$ to
match the target masses $\{\alpha_\ell m'\}$ (reducing to a near-uniform split when
$\alpha_\ell=|\Lambda|^{-1}$), caps each scale at $\tilde{N}_\ell$, and redistributes overflow to
scales with remaining capacity.
Given these budgets, Stage~2 applies water-filling within each scale, distributing the cell quotas
as evenly as the cell capacities allow.
Since the stages are applied sequentially, this construction addresses the two terms of the
decomposed bound in turn, rather than directly minimizing $2\,\mathrm{TV}(\bar{\mu}_S,\mu^\star)$.

The lower bound in Equation~\eqref{eq:tv-two-term} further shows that no within-scale arrangement
can compensate for incorrect scale masses.
Moreover, since $\mu(\mathcal{C}_\ell)\le\tilde{N}_\ell/m'$ for every $\mu\in\mathcal{M}_{m'}$,
\[
2\,\mathrm{TV}(\mu,\mu^\star)\ge\sum_\ell\bigl(\alpha_\ell-\tilde{N}_\ell/m'\bigr)_+
\]
for any feasible allocation, so any mismatch induced by binding scale capacities is unavoidable.

\clearpage

\section{Implementation Details of M\texorpdfstring{$^3$}{3}}
\label{sec:app-partition-allocation}

\subsection{Algorithms for Multi-Scale Measure Construction and Sampling}
\label{sec:app-algorithms}

This appendix provides pseudocode for the three-stage pipeline.
Algorithm~\ref{alg:physical_gradient_partition} presents a unified formulation of variation-adaptive
partitioning and scale stratification.
Algorithm~\ref{alg:cell_equal_level_aware_morton_sampling} implements capacity-aware scale
allocation and within-cell sampling.

\begin{algorithm}[H]
\small
\caption{\textbf{Partition and grouping}: variation-adaptive octree and scale stratification.}
\label{alg:physical_gradient_partition}

\begin{algorithmic}[1]

\Require Labeled point cloud $\mathcal{P}$; scoring tuple $(\mathcal{D},\{w^{(k)}_{\phi}\},\{w^{(j)}_{\psi}\})$; octree limits $(\varepsilon_{\mathrm{refine}},G_{\max},\kappa)$; stratification tuple $(K,\epsilon_{\log})$ with optional trim $(p_{\mathrm{lo}},p_{\mathrm{hi}})$.
  \Ensure Cell set $\mathcal{C}$ with index sets $\{\mathcal{P}_c\}$, labels $\ell(c)$, and strata $\{\mathcal{C}_\ell\}_{\ell\in\Lambda}$ for $\Lambda\subseteq\{1,\dots,K{+}1\}$.

  \Statex \textbf{Step 1: Variation-adaptive partition (Morton octree).}
  \State Build a Morton-ordered (Z-order) octree over $\mathcal{P}$ and initialize a queue of active cells.
  \While{queue not empty}
    \State Pop cell $c$ with point indices $\mathcal{P}_c$.
    \If{$|\mathcal{P}_c|\le\kappa$ \textbf{or} $G(c)\ge G_{\max}$}
      \State add $c$ to $\mathcal{C}$
    \Else
      \State compute variation score $\delta_c$ for cell $c$ via Equation~\eqref{eq:delta-score}
      \If{$\delta_c\le\varepsilon_{\mathrm{refine}}$}
        \State add $c$ to $\mathcal{C}$
      \Else
        \State subdivide $c$ and enqueue nonempty children
      \EndIf
    \EndIf
  \EndWhile

  \Statex \textbf{Step 2: Scale stratification.}
  \ForAll{$c\in\mathcal{C}$ with $|\mathcal{P}_c|>1$}
    \State compute $\delta_c$ via Equation~\eqref{eq:delta-score} if needed; $h_c\leftarrow$ width of $c$; $\tilde{T}_c\leftarrow\delta_c/h_c$; $s_c\leftarrow\log(\epsilon_{\log}+\tilde{T}_c)$
  \EndFor
  \State assign $\ell(c)\in\{1,\ldots,K\}$ by equal-width bins on $\{s_c\}$ (optional trim $(p_{\mathrm{lo}},p_{\mathrm{hi}})$); set $\ell(c)\leftarrow K{+}1$ if $|\mathcal{P}_c|=1$; form $\mathcal{C}_\ell\leftarrow\{c:\ell(c)=\ell\}$, $\Lambda\leftarrow\{\ell:\mathcal{C}_\ell\neq\emptyset\}$
  \State \Return $\mathcal{C}$, $\{\ell(c)\}$, $\{\mathcal{C}_\ell\}_{\ell\in\Lambda}$

  \end{algorithmic}
\end{algorithm}

\begin{algorithm}[ht]
  \small
  \caption{\textbf{Multi-scale sampling via budgeted capacity-constrained allocation}}
  \label{alg:cell_equal_level_aware_morton_sampling}
  \begin{algorithmic}[1]

  \Require Cell groups $\{\mathcal{C}_\ell\}_{\ell\in\Lambda}$ with point sets $\{\mathcal{P}_c\}$, $N_c=|\mathcal{P}_c|$; budget $m$; fill ratio $\rho\in(0,1]$; target weights $\{\alpha_\ell\}$ with default $\alpha_\ell=|\Lambda|^{-1}$
  \Ensure Training index set $\mathcal{I}_{\mathrm{tr}}$ with $|\mathcal{I}_{\mathrm{tr}}|=m'$

  \Statex \textbf{Capacity computation}
  \ForAll{$c\in\bigcup_{\ell\in\Lambda}\mathcal{C}_\ell$}
      \State $\tilde{N}_c \leftarrow \min\bigl(N_c,\max(1,\lceil\rho N_c\rceil)\bigr)$
  \EndFor
  \ForAll{$\ell\in\Lambda$}
      \State $\tilde{N}_\ell \leftarrow \sum_{c\in\mathcal{C}_\ell}\tilde{N}_c$
  \EndFor
  \State $m' \leftarrow \min\bigl(m,\sum_{\ell\in\Lambda}\tilde{N}_\ell\bigr)$; \textbf{if} $m'=0$ \Return $\emptyset$
  \Statex \textbf{Stage 1: Inter-scale allocation}
  \State Compute $\{m_\ell\}$ to match $\{\alpha_\ell m'\}$ under capacities (near-uniform when $\alpha_\ell=|\Lambda|^{-1}$): clamp to $\tilde{N}_\ell$ and redistribute overflow to scales with residual capacity, s.t.\ $\sum_{\ell\in\Lambda}m_\ell=m'$ and $0\le m_\ell\le\tilde{N}_\ell$

  \Statex \textbf{Stage 2: Intra-scale allocation and sampling}
  \State $\mathcal{I}_{\mathrm{tr}} \leftarrow \emptyset$
  \ForAll{$\ell \in \Lambda$ with $m_\ell > 0$}
      \State Solve $\{q_c\}_{c\in\mathcal{C}_\ell}$ by water-filling under $0\le q_c\le\tilde{N}_c$ with $\sum_{c\in\mathcal{C}_\ell}q_c=m_\ell$
      \ForAll{$c \in \mathcal{C}_\ell$ with $q_c > 0$}
          \State Sample $q_c$ indices uniformly without replacement from $\mathcal{P}_c$; append them to $\mathcal{I}_{\mathrm{tr}}$
      \EndFor
  \EndFor

  \State \Return $\mathcal{I}_{\mathrm{tr}}$

  \end{algorithmic}
  \end{algorithm}

\subsection{Hyperparameters and Design Rationale}
\label{sec:app-implementation}
\label{sec:app-method-details}
This subsection provides a detailed account of the hyperparameter choices throughout the pipeline,
together with the physical motivations that guide their design.

\paragraph{Depth Control and Refinement Balance ($G_{\max}$).}
The depth cap $G_{\max}$ controls the maximum number of recursive subdivisions in the
variation-aware Morton partition and therefore directly limits the finest geometric scale that
Step~1 can represent. Its role is not to prescribe refinement everywhere, but to act as a global
ceiling that prevents unbounded subdivision when local variation remains high.

In practice, boundary and volume branches adopt separate depth caps to reflect their different
computational costs and typical cell counts, but the underlying principle is the same: $G_{\max}$
should be large enough to resolve physically relevant structures, yet not so large that it
introduces excessive fragmentation of the support $\mathcal{C}$. If $G_{\max}$ is set too small,
refinement is prematurely halted by depth rather than variation, causing many heterogeneous regions
to be merged into coarse cells. This collapses the dynamic range of $\delta_c$, and subsequently
compresses the distribution of $s_c$, weakening the multi-scale separation in Step~2. Conversely, an
excessively large $G_{\max}$ does not uniformly increase resolution, since subdivision is still
gated by the variation criterion $\delta_c > \varepsilon_{\mathrm{refine}}$. It enables deep
refinement in a small subset of high-variation regions, leading to a rapid growth in cell count and
preprocessing cost, and potentially over-representing fine-scale cells in the support before
allocation.

In this sense, $G_{\max}$ interacts with $\varepsilon_{\mathrm{refine}}$ and minimum-count
constraints as a structural limiter: increasing $G_{\max}$ only affects cells whose refinement is
depth-limited, while cells already satisfying $\delta_c \le \varepsilon_{\mathrm{refine}}$ remain
unchanged. A well-chosen $G_{\max}$ therefore balances representational capacity and computational
tractability, ensuring that the resulting cell set $\mathcal{C}$ spans the relevant range of spatial
scales without degenerating into either under-resolved coarse partitions or unnecessarily deep trees
that complicate subsequent stratification and allocation.

\paragraph{Refinement Criterion and Scale Sensitivity ($\varepsilon_{\mathrm{refine}}$).}
The threshold $\varepsilon_{\mathrm{refine}}$ determines when subdivision stops by controlling the
sensitivity of the partition to local variation, and thus serves as the primary mechanism that
shapes the geometry of the discrete support $\mathcal{C}$. In Step~1, subdivision proceeds as long
as the normalized variation score $\delta_c$ exceeds this threshold, so
$\varepsilon_{\mathrm{refine}}$ effectively defines the resolution at which variation is considered
sufficiently homogeneous within a cell.

We use fixed, branch-specific thresholds across the whole dataset, with surface branches adopting a
relatively coarser threshold and volume branches a tighter one, reflecting the fact that, after
normalization, volumetric flow fields typically exhibit finer-scale structures and require more
aggressive refinement to be resolved. Decreasing $\varepsilon_{\mathrm{refine}}$ leads to finer
partitions by forcing the algorithm to resolve smaller-scale variations, increasing the number of
cells and expanding the high-intensity tail of $\tilde{T}_c = \delta_c / h_c$, whereas increasing it
produces coarser cells and suppresses small-scale variation, effectively truncating the upper range
of $s_c$.

Through this mechanism, $\varepsilon_{\mathrm{refine}}$ not only determines the stopping condition
locally, but also globally reshapes the empirical distribution of $s_c$ that underlies Step~2. In
particular, since $s_c$ is derived from $\delta_c$ and the cell size $h_c$, the choice of
$\varepsilon_{\mathrm{refine}}$ controls the balance between geometric scale and variation intensity
in the resulting strata: overly large thresholds collapse variation across scales and reduce
separability, while overly small thresholds amplify fine-scale noise and may lead to
over-fragmentation without improving meaningful structure. As a result,
$\varepsilon_{\mathrm{refine}}$ plays a central role in aligning the partition with the intrinsic
multi-scale variation of the data, ensuring that the constructed support $\mathcal{C}$ provides a
faithful and well-conditioned basis for subsequent grouping and budgeted allocation.

\paragraph{Scalar and Vector Channel Weights ($w^{(k)}_{\phi}$, $w^{(j)}_{\psi}$).}
The refinement score $\delta_c$ in Equation~\eqref{eq:delta-score} aggregates extrema-based
variation across scalar and vector channels into a single trigger for subdivision, and the weights
$w^{(k)}_{\phi}$ and $w^{(j)}_{\psi}$ determine their relative influence on the resulting partition.
In the default external-aerodynamics setting, surface branches use pressure and wall shear stress
while volume branches use pressure and velocity, with one scalar and one vector channel in each
case.

Although all channels are z-score normalized so that their ranges are dimensionless and broadly
comparable, their geometric and physical roles differ: scalar fields such as pressure typically
encode dominant structural variations tied to loads and global flow features, whereas vector fields
often exhibit larger directional fluctuations that can produce high projection diameters even when
their contribution to the target quantities is secondary.

To avoid refinement being dominated by such directional effects, we deliberately downweight
vector-channel terms relative to scalar ranges, ensuring that subdivision is primarily driven by
physically salient scalar contrast while still allowing vector variation to trigger refinement when
it is sufficiently strong. The projection-based definition of $\Delta \psi^{(j)}_c$, computed over a
finite set of directions, further reduces axis bias compared to treating vector components
independently, yielding a more isotropic estimate of spread.

Through this design, the weights act as a mechanism for encoding coarse physical priorities into the
partition: they regulate which types of variation are considered important for defining the support
$\mathcal{C}$, and thereby indirectly shape the distribution of $\delta_c$, the induced intensity
scores, and the downstream multi-scale grouping and allocation.

\paragraph{Ablation Study on the Vector-channel Weight $w^{(j)}_{\psi}$.}
\label{sec:app-wing-weight-v-sweep}
To isolate how the vector-channel weight influences support construction, we perform an ablation
study on a representative SHIFT-Wing case at $M\approx0.85$ for both the surface and volume fields.
This transonic regime features shock formation and boundary-layer interactions, leading to large but
spatially distinct variations in pressure and velocity, and thus exposes how scalar and vector
channels compete in driving refinement. We vary only $w^{(j)}_{\psi}\in\{0.1,0.2,0.3,0.4,0.5\}$
while fixing $w^{(k)}_{\phi}=1$ and all other partition hyperparameters, and execute the partition
stage alone. All observed changes therefore arise solely from the relative weighting of scalar and
vector contributions in $\delta_c$.

We track both \emph{latent} refinement signals and \emph{triggered} refinement events.
The counters $N^{\mathrm{thr}}_{\mathrm{s}}$ and $N^{\mathrm{thr}}_{\mathrm{v}}$ count the number of
times the scalar or vector channel exceeds the $\varepsilon_{\mathrm{refine}}$ threshold.
The counters $N^{\mathrm{refine}}_{\mathrm{s}}$ and $N^{\mathrm{refine}}_{\mathrm{v}}$ count the
number of splits triggered by the scalar or vector channel, respectively, i.e., the channel selected
as the maximizer in Equation~\eqref{eq:delta-score}. As $w^{(j)}_{\psi}$ increases, the number of
cells $|\mathcal{C}|$ grows monotonically, $N^{\mathrm{thr}}_{\mathrm{s}}$ remains constant, and
$N^{\mathrm{thr}}_{\mathrm{v}}$ increases, indicating that more cells become eligible for refinement
due to amplified vector contributions. Correspondingly, the share of vector-triggered subdivisions
increases further while scalar-triggered ones decrease, as shown in
Table~\ref{tab:wing-weight-v-sweep}. The weight therefore controls which channel drives refinement
and, through it, the distribution of cells in $\mathcal{C}$ and their variation scales.

\begin{table}[ht]
\centering
\caption{\textbf{SHIFT-Wing: Sensitivity of partitioning to the vector-channel weight $w^{(j)}_{\psi}$}. By increasing the weight, we effectively control whether the scalar or vector channel dominates the refinement process, making the corresponding channel more likely to trigger refinement.}
\label{tab:wing-weight-v-sweep}
\vspace{5pt}
\footnotesize
\setlength{\tabcolsep}{2.5pt}
\begin{tabular}{@{}c*{5}{r}*{5}{r}@{}}
\toprule
& \multicolumn{5}{c}{Surface field} & \multicolumn{5}{c}{Volume field} \\
\cmidrule(lr){2-6}\cmidrule(lr){7-11}
$w^{(j)}_{\psi}$
& $|\mathcal{C}|$ & $N^{\mathrm{thr}}_{\mathrm{s}}$ & $N^{\mathrm{thr}}_{\mathrm{v}}$ &
$N^{\mathrm{refine}}_{\mathrm{s}}$ & $N^{\mathrm{refine}}_{\mathrm{v}}$
& $|\mathcal{C}|$ & $N^{\mathrm{thr}}_{\mathrm{s}}$ & $N^{\mathrm{thr}}_{\mathrm{v}}$
  & $N^{\mathrm{refine}}_{\mathrm{s}}$ & $N^{\mathrm{refine}}_{\mathrm{v}}$ \\
\midrule
0.1 & 30{,}466 & 5{,}258 & 8{,}958 & 3{,}252 & 6{,}161 & 220{,}524 & 38{,}516 & 77{,}751 & 25{,}193 & 58{,}090 \\
0.2 & 32{,}574 & 5{,}258 & 10{,}434 & 1{,}940 & 8{,}552 & 236{,}582 & 38{,}516 & 88{,}882 & 16{,}887 & 73{,}520 \\
0.3 & 32{,}836 & 5{,}258 & 10{,}669 & 1{,}053 & 9{,}628 & 242{,}754 & 38{,}516 & 92{,}479 & 8{,}398 & 84{,}537 \\
0.4 & 32{,}954 & 5{,}258 & 10{,}762 & 629 & 10{,}140 & 246{,}654 & 38{,}516 & 94{,}255 & 3{,}502 & 90{,}883 \\
0.5 & 33{,}035 & 5{,}258 & 10{,}816 & 421 & 10{,}399 & 249{,}629 & 38{,}516 & 95{,}340 & 1{,}406 & 93{,}990 \\
\bottomrule
\end{tabular}
\end{table}

\paragraph{Shared Hyperparameter Configuration.}

We do not perform per-case or per-run tuning of partition and allocation hyperparameters. Instead, a
single configuration is used for each benchmark family (Table~\ref{tab:partition-hparams-datasets})
and applied uniformly across all datasets and all samples within each dataset. The intent is to
decouple evaluation from instance-specific optimization and to reflect a realistic setting where
hyperparameters are chosen once based on coarse physical reasoning rather than exhaustive search.
While per-case tuning would improve performance, we adopt a shared configuration so that observed
gains are attributable to the constructed measure rather than tailored hyperparameter adjustments.
As a result, the reported performance reflects the robustness of the proposed strategy rather than
gains from dataset-specific optimization.

\begin{table}[ht]
\centering
\caption{\textbf{Partition hyperparameters.}
$|\mathcal{D}|$ denotes the number of unit projection directions used to evaluate the vector
projection diameter $\Delta\psi^{(j)}_c$ in Equation~\eqref{eq:Delta-v-proj}. $K$ denotes the number
of multi-scale strata bins. The per-cell point cap is $\kappa=3$ for all datasets. All settings are
shared across datasets for our experiments, except for SHIFT-Wing, where a smaller $w^{(j)}_{\psi}$
mitigates overly aggressive vector-driven refinement under stronger flow variations.}
\label{tab:partition-hparams-datasets}
\vspace{5pt}
\footnotesize
\setlength{\tabcolsep}{3pt}
\begin{tabular}{@{}l*{8}{c}@{}}
\toprule
Dataset & $G_{\max}^{\mathrm{surf}}$ & $G_{\max}^{\mathrm{vol}}$
  & $\varepsilon_{\mathrm{refine}}^{\mathrm{surf}}$
  & $\varepsilon_{\mathrm{refine}}^{\mathrm{vol}}$ & $|\mathcal{D}|$ & $w^{(k)}_{\phi}$
  & $w^{(j)}_{\psi}$ & $K$ \\
\midrule
AhmedML & 8 & 13 & 0.05 & 0.005 & 13 & 1.0 & 0.4 & 64 \\
DrivAerML & 8 & 13 & 0.05 & 0.005 & 13 & 1.0 & 0.4 & 64 \\
SHIFT-Wing & 8 & 13 & 0.05 & 0.005 & 13 & 1.0 & 0.1 & 64 \\
\bottomrule
\end{tabular}
\end{table}

\paragraph{Stratification Bins $K$.}
\label{sec:app-strata}
Once the cell set $\mathcal{C}$ is fixed, Step~2 assigns a scale label $\ell(c)$ to each cell
according to its log-intensity score $s_c$, yielding strata $\{\mathcal{C}_\ell\}$ indexed by
$\Lambda$. This stratification is controlled by the number of bins $K$ together with percentile
bounds $p_{\mathrm{lo}}$ and $p_{\mathrm{hi}}$.

To obtain stable strata boundaries, the range of $s_c$ over multi-point cells is restricted using
percentile clipping. Specifically, $s_{\min}$ and $s_{\max}$ are defined as the $p_{\mathrm{lo}}$-th
and $p_{\mathrm{hi}}$-th percentiles of $\{s_c\}$. Setting $p_{\mathrm{lo}}=0$ and
$p_{\mathrm{hi}}=100$ recovers the full empirical range, while reducing $p_{\mathrm{hi}}$ removes
extreme upper-tail values. In CFD data, $s_c$ typically exhibits a heavy-tailed distribution, where
localized high-variation regions (e.g., shocks or separation zones) produce large scores that expand
the global range and compress the majority of cells into a narrow interval.

The clipped range is then divided into $K$ bins uniformly in $s_c$. Since $s_c$ is log-transformed,
uniform spacing in $s_c$ induces a geometric partition of the underlying variation magnitude. Each
cell is assigned to a stratum based on the interval it falls into, so that strata correspond to
comparable relative variation scales rather than absolute differences.

In particular, $K$ determines the granularity of the multi-scale decomposition: larger $K$ produces
more strata with narrower $s_c$ intervals and thus finer distinctions in physical variation, while
smaller $K$ yields coarser strata that group a wider range of variations. Importantly, regardless of
$K$, the overall span of variation considered is fixed by $p_{\mathrm{lo}}$ and $p_{\mathrm{hi}}$,
and the number of cells is determined solely by the Step~1 partition.

Percentile clipping and uniform binning together define the stratification procedure: clipping
limits the influence of extreme values, while binning controls the resolution of scale separation.
Consequently, the resulting strata are more evenly populated and aligned with the multi-scale
structure of the data, providing a stable basis for the subsequent sampling step.

\paragraph{Sampling Fill Ratio ($\rho$).}
The sampling fill ratio $\rho \in (0,1]$ controls the effective capacity of each cell.
For a cell with raw cardinality $N_c$, its effective capacity is defined as $\tilde{N}_c$.
This capacity limits the maximum number of points that can be drawn from the cell and determines
when the cell is considered sufficiently represented.
A smaller $\rho$ reduces this effective capacity, causing cells to reach saturation earlier during sampling.

This prevents the case where the allocated budget of a cell exceeds its useful capacity (total point
count), so that the subsequent water-filling procedure does not absorb excessive budget and
over-represent individual cells. By imposing the per-cell cap through $\rho$, the excess budget is
redistributed to more cells with nonzero allocation, resulting in a more balanced spatial
distribution across regions and scales.

\section{Scalability and Efficiency}
\label{sec:app-preprocess-feasibility}

\textbf{Preprocessing Feasibility with $N$-point Inputs.}
Preprocessing can become a computational bottleneck at industrial scales (e.g., $10^6$--$10^8$
points). This section evaluates the scalability and processing efficiency of M$^3$ on large-scale
point clouds, comparing it with four alternative methods. As sampling decisions depend on geometry-
or variation-based scores computed over the entire domain, all methods require processing the full
labeled point set $\mathcal{P}$ of size $N$ to construct the same number $m$ of support indices
across increasing input scales.

\begin{itemize}
\setlength{\itemsep}{2pt}

  \item \textbf{Uniform Random:} uniform subsampling without any scoring, serving as a lower-bound cost baseline. This reflects the minimum preprocessing overhead.

  \item \textbf{M$^3$ (Morton-based):} uses Morton-order partitioning followed by scale-aware index construction. The partitioning and indexing (grouping) stages are included in wall time.

  \item \textbf{Grid-based method:} partitions the domain using a uniform voxel grid and computes per-voxel scores based on $(\phi^{(k)}, \boldsymbol{\psi}^{(j)})$, with channel weights analogous to Equation~\eqref{eq:delta-score}. Indexing is then performed via uniform voxel selection to obtain $m$ indices.

  \item \textbf{Proxy-based method:} applies the same scoring mechanism on a coarser voxel grid (with approximately 256 points per voxel). Final sampling indices are obtained via global weighted sampling using voxel-level importance scores, yielding $m$ selected indices.

  \item \textbf{Neighborhood-based method:} constructs nearest-neighbor neighborhoods of size $k_{\mathrm{NN}}=32$. Local scores are computed from $(\phi^{(k)}, \boldsymbol{\psi}^{(j)})$, aggregated using weights $w^{(k)}_{\phi}$ and $w^{(j)}_{\psi}$ as in Equation~\eqref{eq:delta-score}, followed by top-$m$ selection.

\end{itemize}

To ensure a fair comparison of preprocessing efficiency, all baselines use streaming access to avoid
loading the full dataset into memory, which would otherwise lead to out-of-memory issues at large
input scale. We also impose modest computational constraints (e.g., limiting neighborhood sizes) to
keep these methods feasible under the same resource budget.

Our evaluation uses a wall-clock limit of $600$s per case and records the Peak Resident Set Size
(RSS) during execution. At this setting, even if each case only reaches the time limit, a $500$-case
benchmark would still require $3.0\times 10^{5}$\,s ($\approx$83 hours) of preprocessing. Methods
that frequently hit this limit are therefore impractical for large-scale data processing. All
reported wall-clock times measure compute only (excluding memory-mapped data loading), and runs
exceeding the limit are marked as \texttt{Timeout}. All methods are run five times with different
random seeds, and the results are averaged to reduce stochasticity and ensure fair comparison.

\begin{figure}[ht]
\centering
\begin{minipage}[t]{0.48\linewidth}
\centering
\includegraphics[width=\linewidth]{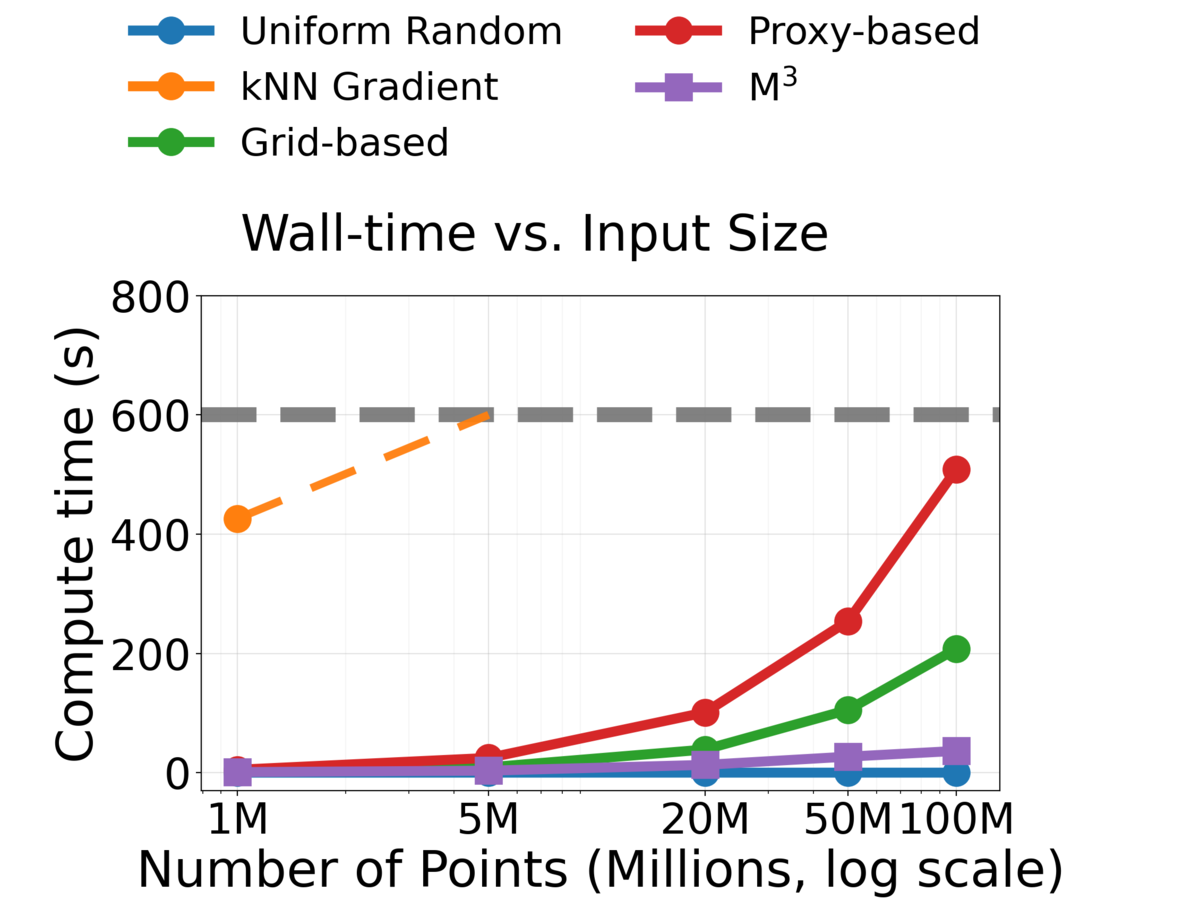}
\end{minipage}\hfill
\begin{minipage}[t]{0.48\linewidth}
\centering
\includegraphics[width=0.97\linewidth]{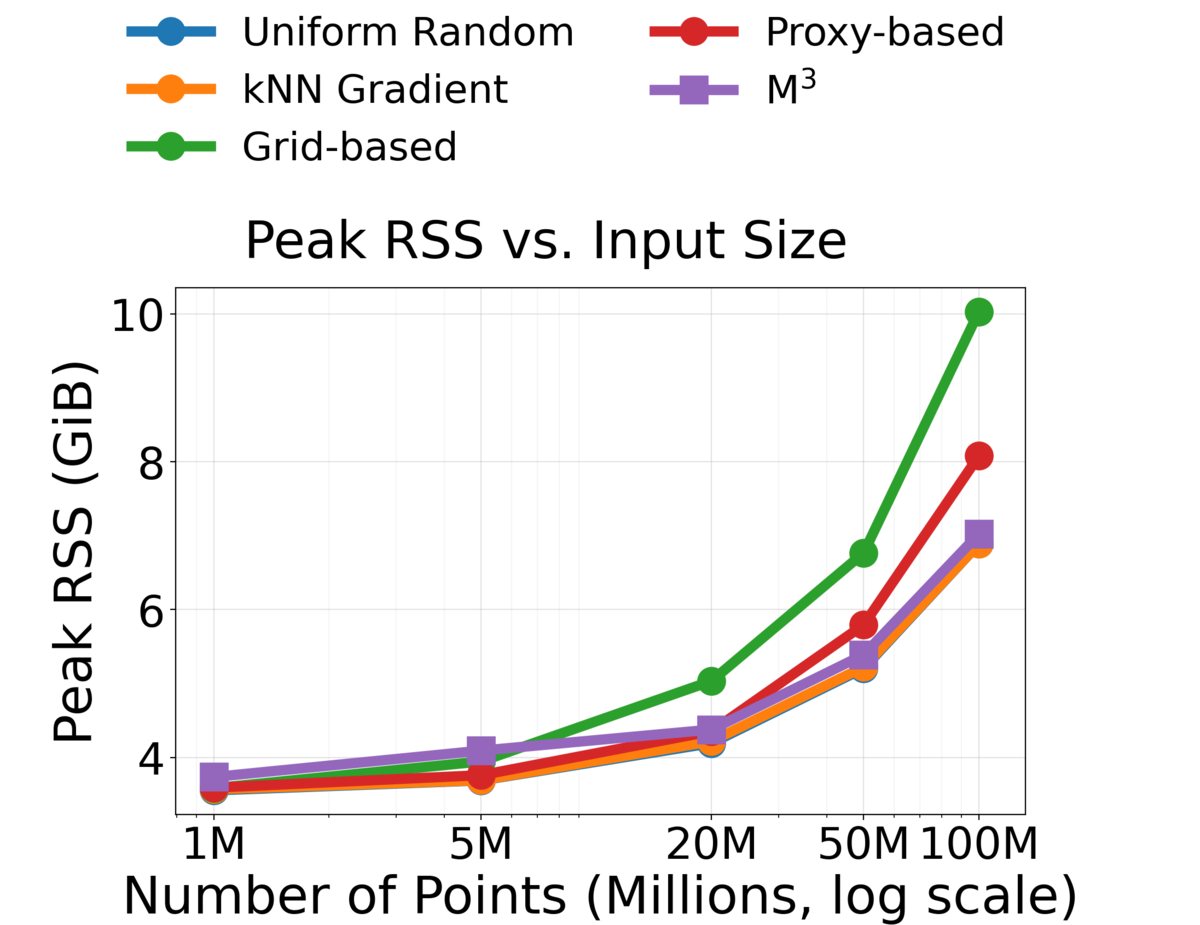}
\end{minipage}
\caption{\textbf{Efficiency comparison across methods under increasing scale $N$.}
\emph{Left:} wall-clock time under a 600\,s limit. The neighborhood-based method ($k$NN) times out
beyond $N\!\approx\!10^6$, while M$^3$ remains efficient across all scales. Grid- and proxy-based
methods remain feasible but become significantly slower at large $N$.
\emph{Right:} peak RSS. Most methods remain within a similar memory range, while grid- and
proxy-based methods grow more rapidly at the largest scales.}
\label{fig:feasibility-wall-peak}
\end{figure}

\begin{table}[ht]
\centering
\footnotesize
\caption{\textbf{Preprocessing time as input size increases (seconds).} Results at $N\!=\!10^6$, $2{\times}10^7$, and $10^8$. Rows are sorted by runtime, and entries exceeding the $600$\,s cap are marked as \texttt{Timeout}.}
\label{tab:exp-scaling-curve}
\vspace{5pt}
\makebox[\linewidth][c]{
\begin{tabular}{@{}lccc@{}}
\toprule
Method & $N\!=\!10^6$ & $N\!=\!2{\times}10^7$ & $N\!=\!10^8$ \\
\midrule
Random & $3.3\!\times\!10^{-4}$ & $4.0\!\times\!10^{-3}$ & $2.4\!\times\!10^{-2}$ \\
M$^3$ & $0.94$ & $13.1$ & $36.4$ \\
Grid-based & $1.67$ & $38.8$ & $208$ \\
Proxy-based & $4.92$ & $101$ & $509$ \\
Neighborhood-based & $426$ & \texttt{Timeout} & \texttt{Timeout} \\
\bottomrule
\end{tabular}
}
\end{table}

Figure~\ref{fig:feasibility-wall-peak} and Table~\ref{tab:exp-scaling-curve} show a clear separation
across methods. In terms of wall-clock time, neighborhood-based methods already require $426$\,s at
$N=10^6$, compared to $0.94$\,s for M$^3$, and reach \texttt{Timeout} at $N=2{\times}10^7$ and
$10^8$. Grid-based and proxy-based methods remain feasible but scale poorly, reaching $208$\,s and
$509$\,s at $10^8$, while M$^3$ completes in tens of seconds ($36.4$\,s). Peak RSS remains within a
comparable range (a few to $\sim$10\,GiB) for random, neighborhood-based, and M$^3$, while
grid-based and proxy-based methods grow more rapidly at large $N$. In these regimes, wall-clock
time, rather than memory, becomes the primary feasibility constraint.

\paragraph{Why Morton Matches Our Method.}
Neighborhood-based methods must explicitly construct and access large per-point neighborhoods (e.g.,
indices, adjacency lists, and repeated queries), leading to computational cost that scales with the
total neighbor volume across the point cloud. This quickly becomes impractical at industrial scale,
as reflected in the observed scaling behavior.

In contrast, M$^3$ leverages linearized Morton (Z-order)
ordering~\cite{morton1966zorder}. A single global sort is followed by efficient
linear scans, enabling interval-based extrema queries within cells without constructing explicit
global $k$NN graphs. This avoids repeated neighborhood queries and keeps M$^3$ efficient even at
$10^8$ points.

While Hilbert~\cite{hilbert1891stetige,moon2001hilbert} and other space-filling curves offer
stronger spatial locality, their benefit is most pronounced in fine-grained point-wise neighborhood
operations. In our setting, computation is performed over coarse spatial cells, where strict
adjacency is not required. Morton ordering provides sufficiently strong locality while being
simpler, faster to compute, and more cache-friendly, making it a practical choice for scalable
preprocessing.

Given these scaling characteristics, we focus our main experiments on comparing uniform random
sampling and M$^3$ under the same budget. In this section, grid-based, proxy-based importance
sampling, and neighborhood-based methods serve as feasibility references, while structured training
baselines, including grid-stratified sampling, are compared in Appendix~\ref{app:baselines}.
Overall, these results establish a clear systems-level separation: M$^3$ preprocessing remains
tractable at $\sim$100M-scale industrial point clouds, whereas neighborhood-based methods fail to
scale and grid- and proxy-based methods incur substantially higher cost.


\section{Experimental Protocol}
\label{sec:app-experimental-setup}

\subsection{Datasets}
\label{sec:app-ext-aero-datasets}

\paragraph{AhmedML.}
\label{sec:app-ahmedml}
AhmedML~\cite{ashton2024ahmedml} is a public dataset of high-resolution CFD simulations for 500
Ahmed body variants, a canonical bluff-body model in automotive aerodynamics. Simulations are
performed with a hybrid RANS--LES scheme in OpenFOAM, capturing key flow phenomena such as
separation and complex 3D vortical structures. Each case contains on the order of 20M mesh cells. As
the dataset does not provide an official split, we partition it into 400/50/50 training, validation,
and test samples following an 80/10/10 ratio.

\paragraph{DrivAerML.}
\label{sec:app-drivaerml}
DrivAerML~\cite{ashton2024drivaerml} is a large-scale dataset for machine learning in high-fidelity
automotive aerodynamics. It consists of 500 parametrically varied DrivAer geometries and addresses
the limited availability of publicly accessible CFD data at industrial resolution. Each case
contains approximately 8.8M surface points and 140--160M volumetric cells, with simulations
performed using hybrid RANS--LES, reflecting state-of-the-art practice in the automotive industry.
Each sample provides surface fields including pressure and wall shear stress, as well as volumetric
fields such as velocity, pressure, and vorticity. As no official split is provided and simulation
results are missing for 16 cases, we use the remaining 484 valid samples and partition them into
387/48/49 for training, validation, and testing following an 80/10/10 split.

\paragraph{Luminary SHIFT-Wing.}
\label{sec:app-shift-wing}
SHIFT-Wing~\cite{luminary2025shiftwing} is an open-source dataset for transonic transport
aerodynamics based on parametrically varied NASA Common Research Model (CRM) geometries, developed
with Otto Aviation. It focuses on high-speed cruise wing--fuselage configurations for
planform-centric design exploration. Geometry and operating conditions are sampled via Latin
hypercube design, with angle of attack in $[0^\circ,4^\circ]$ and Mach numbers in $\{0.5, 0.85\}$,
covering both shock-free and transonic regimes. Simulations use steady RANS with the
Spalart--Allmaras model. A key feature is Luminary Mesh Adaptation, an automated solution-adaptive
meshing procedure that refines anisotropically to capture shocks and sharp gradients. In this work,
we use the $M=0.85$ subset (1714 simulations), as it exhibits stronger variations and a broader
range of physical scales compared to the $M=0.5$ subset. The dataset is split into 1371/171/172
cases for training/validation/test following an 80/10/10 partition. Each case includes approximately
3M surface points and 6M volume points, with surface pressure $p_s$, wall shear stress
$\boldsymbol{\tau}$, and volumetric fields including pressure $p_v$ and velocity $\mathbf{u}$.

\subsection{Backbone Considerations for Measure Comparison}
\label{sec:app-backbone-choice}
\label{app:backbone-mechanism}
Our work evaluates model behavior under shifts of the training measure, with model capacity held
fixed. This imposes two requirements.
\begin{enumerate}
\item The latent context must be constructed consistently during training and inference, so that performance differences can be attributed to the input data.
\item The evaluation is conducted under the same target distribution to ensure a fair comparison.
\end{enumerate}

We therefore use AB-UPT~\cite{alkin2025abupt} as the main backbone. This choice is motivated by its
decoupled structure of latent anchor tokens and query tokens. Anchors provide stable latent
references, while queries retrieve representations by attending to the anchors. Such a design allows
training under varying input distributions while supporting evaluation on arbitrary query tokens. At
inference, a small set of anchors, on the order of \(10^{4}\), can be held fixed and kept the same
as in training. The latent context then stays the same when the queried points change in order,
distribution, or resolution. Models trained on different measures can thus be evaluated on the same
full-mesh query set. This is essential for large-scale simulation, where a full case contains
on the order of $10^6$--$10^8$ volume points and must be processed in chunks.

Other scalable models, including Transolver~\cite{wu2024transolver} and GAOT~\cite{wen2025gaot},
build the latent context through input-conditioned aggregation, such as slicing or multiscale
neighborhood encoding. In these models, changes in discretization or sampling cause the distribution
of input tokens at inference to differ from the distribution seen during training. As the latent
tokens are computed from the input structure, this mismatch alters both the optimization objective
and the resulting representations, and the effect of the input-distribution shift is difficult to
isolate.

For example, in Transolver each forward pass compresses the supplied batch into slice tokens. For a
batch $B$ with features $\mathbf{e}_i$ and soft assignments $a_{ij}$ of point $i$ to slice $j$, the
tokens are
\begin{equation}
\label{eq:app-slice-token}
\mathbf{z}_j
=\frac{\sum_{i\in B}a_{ij}\mathbf{e}_i}{\sum_{i\in B}a_{ij}}
=\frac{\int a_j(\mathbf{x})\mathbf{e}(\mathbf{x})\,d\nu_B(\mathbf{x})}{\int a_j(\mathbf{x})\,d\nu_B(\mathbf{x})},
\qquad
\nu_B=\frac{1}{|B|}\sum_{i\in B}\delta_{\mathbf{x}_i}.
\end{equation}
Each token $\mathbf{z}_j$ is an average under the empirical measure $\nu_B$ of the current batch,
coupling the latent context to the input measure. Consequently, the latent context is shaped based
on the training input measure. At full-mesh inference, the input follows the solver mesh
measure, which may differ from the training measure. This mismatch introduces a train--inference
discrepancy that can degrade prediction accuracy. The resulting error therefore reflects not only
the effect of the training measure, but also the discrepancy between the training and inference
measures.

\subsubsection{Train--Inference Discrepancy under Different Measures}
\label{app:backbone-protocols}
\label{app:backbone-results}

We illustrate this with Transolver by grouping the same full mesh on the surface and volume in
two ways. In both protocols, inference is run on all mesh points.

\paragraph{Native ordering.}
Points follow the storage order of the solver output and are divided into consecutive chunks, which
are spatially localized because the mesh numbering is spatially coherent.

\paragraph{Permuted ordering.}
Surface and volume points are randomly permuted before being divided into chunks. Each chunk
therefore contains points distributed across the full domain.

The permuted protocol matches the mesh-uniform distribution used to train Random, whereas the native
protocol matches neither training measure because each chunk is spatially localized. Neither
protocol reproduces the $\mathrm{M}^{3}$ measure, which is constructed through variation-adaptive
downsampling of the original mesh data.
\begin{table}[t]
\centering
\caption{Transolver on DrivAerML under native and permuted full-mesh inference (relative $L_2$, \%). Equal-weight errors are on the left and $\Delta S$/$\Delta V$-weighted errors are on the right. Results are averaged over three training seeds. Shading marks the lower error within each protocol.}
\label{tab:app-transolver-protocols}
\vspace{5pt}
\small
\setlength{\fboxsep}{1.5pt}
\renewcommand{\tabhl}[1]{\colorbox{gray!35}{#1}\hspace{-\fboxsep}}
\newcommand{\tabhd}[1]{\setbox0\hbox{00.00}\hbox to \wd0{\hfil#1\hfil}}
\begin{tabular*}{\linewidth}{@{\extracolsep{\fill}}ll*{4}{r}@{\hspace{2em}}*{4}{r}@{}}
\toprule
Protocol & Method & \tabhd{$p_s$} & \tabhd{$\boldsymbol{\tau}$} & \tabhd{$p_v$}
  & \tabhd{$\mathbf{u}$} & \tabhd{$p_{s,\omega}$} & \tabhd{$\boldsymbol{\tau}_{\omega}$}
  & \tabhd{$p_{v,\omega}$} & \tabhd{$\mathbf{u}_{\omega}$} \\
\midrule
Native & Random & 71.32 & \tabhl{92.18} & 96.77 & 74.92 & 85.66 & 87.64 & 6.71 & 4.60 \\
Native & $\mathrm{M}^{3}$ & \tabhl{64.29} & 92.77 & \tabhl{81.08} & \tabhl{70.93} & \tabhl{65.26}
  & \tabhl{75.50} & \tabhl{1.19} & \tabhl{0.91} \\
\midrule
Permuted & Random & \tabhl{4.80} & \tabhl{8.89} & \tabhl{8.10} & \tabhl{6.69} & \tabhl{5.02}
  & \tabhl{8.04} & 0.68 & 0.75 \\
Permuted & $\mathrm{M}^{3}$ & 9.86 & 18.02 & 37.66 & 31.19 & 9.69 & 14.00 & \tabhl{0.61} & \tabhl{0.39} \\
\bottomrule
\end{tabular*}
\end{table}

Table~\ref{tab:app-transolver-protocols} shows that the ranking reverses when the same
full mesh is grouped differently for inference, as predicted by
Equation~\eqref{eq:app-slice-token}. The slice tokens, and hence the latent context, depend on
$\nu_B$.
Under native ordering, each chunk is spatially localized, so $\nu_B$ matches neither training
measure. Both models degrade, yet $\mathrm{M}^{3}$ remains lower on seven of the eight metrics,
including the physics-weighted volume errors of $1.19/0.91$ versus $6.71/4.60$ for
$p_{v,\omega}/\mathbf{u}_{\omega}$. Under permuted ordering, each chunk is distributed across the
domain, making $\nu_B$ consistent with the mesh-uniform batches used to train Random. Random then
performs better on the six equal-weight and physics-weighted surface metrics, while the
physics-weighted volume errors remain close, at $0.61/0.39$ for $\mathrm{M}^{3}$ versus $0.68/0.75$.
Thus, the same checkpoints yield opposite conclusions about the training measure depending on the
inference grouping, which changes the latent context. We therefore use AB-UPT for our measure study.
Its latent context is defined by the anchor set, allowing the full-mesh query set to be
evaluated with the same train-inference behavior.

\subsubsection{AB-UPT Anchor Distribution Ablation}
\label{app:backbone-anchor}

We ablate the anchor distribution to check that the gains of $\mathrm{M}^{3}$ come from the training
data rather than from how anchors are placed.
On DrivAerML in the $\times 0.01/\times 0.01$ setting of Table~\ref{tab:exp-scale-drivaerml}, a
$2\times2$ design varies the training data between Random and $\mathrm{M}^{3}$ and the anchor
distribution between mesh and $\mathrm{M}^{3}$.
In the main experiments each forward pass resamples anchors from the current case. For this ablation
we draw one set of $8192$ surface and $8192$ volume anchors once from a single case and reuse it
across all cases and epochs.
Because anchors act as key-value stores, surface anchors need not lie on the body and volume anchors
need not lie in the flow domain, so the model can still learn across cases from a set drawn from one
specific case mesh.
AB-UPT likewise reports that performance is robust to the choice of anchor
distribution~\cite{alkin2025abupt}, which supports this controlled ablation.
All remaining settings follow Section~\ref{sec:app-training-protocol}.

\begin{table}[ht]
\centering
\caption{Anchor distribution versus training measure for AB-UPT on DrivAerML at $\times 0.01/\times 0.01$ (validation relative $L_2$). All models use a single fixed anchor set drawn once from one case (Random or $\mathrm{M}^{3}$ data) and shared across cases and epochs. Since each model is validated on held-out cases sampled from its own training measure, rows with the same training data are directly comparable.}
\label{tab:app-anchor-ablation}
\vspace{5pt}
\small
\begin{tabular*}{\linewidth}{@{\extracolsep{\fill}}ll*{4}{r}@{}}
\toprule
Training data & Anchors & \multicolumn{1}{c}{$p_s$} & \multicolumn{1}{c}{$\boldsymbol{\tau}$}
  & \multicolumn{1}{c}{$p_v$} & \multicolumn{1}{c}{$\mathbf{u}$} \\
\midrule
Random & mesh & 0.0607 & 0.1037 & 0.0839 & 0.0722 \\
Random & $\mathrm{M}^{3}$ & 0.0599 & 0.1035 & 0.0837 & 0.0719 \\
$\mathrm{M}^{3}$ & mesh & 0.0751 & 0.1020 & 0.0377 & 0.0336 \\
$\mathrm{M}^{3}$ & $\mathrm{M}^{3}$ & 0.0708 & 0.0965 & 0.0373 & 0.0332 \\
\bottomrule
\end{tabular*}
\end{table}

\begin{figure}[ht]
\centering
\includegraphics[width=\linewidth]{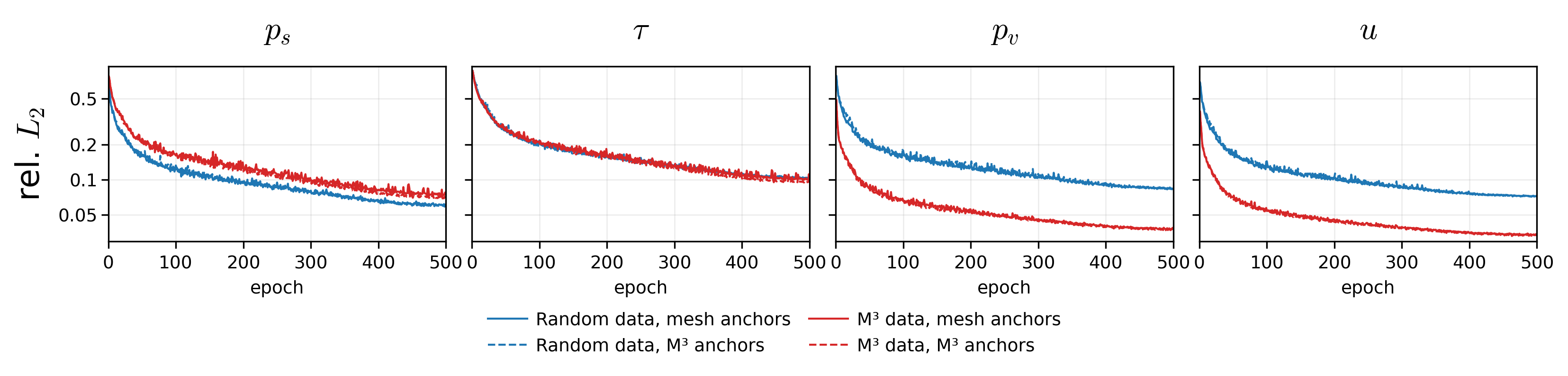}
\caption{Validation relative $L_2$ over 500 training epochs for the four models in Table~\ref{tab:app-anchor-ablation}. Color denotes the training data and line style denotes the anchor set.}
\label{fig:app-anchor-curves}
\end{figure}

Table~\ref{tab:app-anchor-ablation} shows that the error follows the training data rather than the
anchor set. Replacing the mesh anchors with $\mathrm{M}^{3}$ anchors changes each validation error
by at most $6\%$ in relative terms. Absolute errors across different training-data rows are not on a
shared query measure, so we do not treat the gap between Random and $\mathrm{M}^{3}$ rows as a
common-measure ranking. Figure~\ref{fig:app-anchor-curves} further compares the validation
relative $L_2$ curves across different configurations. Curves sharing the training data stay close
from the early epochs through epoch $500$, indicating that the anchor set has little effect on the
monitored validation trajectory. In contrast, curves with different training data separate early
and remain apart throughout all $500$ epochs, most clearly for $p_v$ and $\mathbf{u}$. These
results suggest that the training measure changes the learning trajectory itself under each model's
own-measure validation monitor. Full-mesh inference on held-out cases yields errors of the
same magnitude and preserves the same within-training-data ordering, confirming that the comparison
is not an artifact of the validation queries. Taken together, these findings indicate that the
effects of $\mathrm{M}^{3}$ are attributable to the construction of the training data.

\subsection{Training Setup}
\label{sec:app-training-protocol}

We follow the general training setup of AB-UPT, including the optimizer, learning rate schedule,
architectural design, loss weighting, and numerical precision. The only difference lies in the
training data construction: we compare models trained with uniform random subsampling (the default
AB-UPT strategy) and with data processed by the M$^3$ pipeline.

Across all datasets, the model uses a hidden dimension of 192 with 12 Transformer blocks. We train
for 500 epochs on AhmedML and DrivAerML, and 250 epochs on SHIFT-Wing, using an effective batch size
of 1. Optimization is performed with the LION optimizer~\cite{chen2023lion}, a peak learning rate of
$5{\times}10^{-5}$, weight decay of $0.05$, linear warmup over the first 5\% of training, and cosine
decay to $10^{-6}$. AhmedML and DrivAerML use \texttt{float16} mixed precision, while SHIFT-Wing
uses full \texttt{float32}. For SHIFT-Wing, the angle of attack varies across simulations and is
incorporated via Diffusion Transformer (DiT)-style conditioning applied to all blocks in both
surface and volume branches. The models are trained using a weighted sum of mean squared error (MSE)
losses over multiple physical fields, including surface pressure, wall shear stress, volume
velocity, and volume pressure.

Following the design principles of AB-UPT, all models use symmetric surface and volume tokens
(anchors and queries), with 8192 tokens for AhmedML/DrivAerML and 2048 for SHIFT-Wing. While the
original configuration uses 16384 and 8192 tokens, respectively, we adopt moderately reduced values
guided by AB-UPT's empirical settings, which are sufficient at this scale to support accurate
evaluation in our experiments.

Training is conducted on a single NVIDIA GeForce RTX 5090 GPU, taking approximately 7/7/28 hours for
AhmedML, DrivAerML, and SHIFT-Wing, respectively, and occupies 4GB of GPU memory. The longer runtime
for SHIFT-Wing is due to the larger number of effective samples per epoch.


\section{Baselines and Ablations}
\label{app:baselines}

This section complements the main experiments by comparing $\mathrm{M}^{3}$ with simpler
forms of density correction and by ablating the stages of its multi-scale construction.
We consider five baselines.
Baselines~1--2 correct the mesh density toward the geometric measures $\Delta S$/$\Delta V$,
either through the loss (loss reweighting) or through the sampling density (geometric-measure
sampling).
Baseline~3, Grid-stratified sampling, is the simplest structured alternative to
$\mathrm{M}^{3}$, stratifying space by a fixed Cartesian grid without any physical-variation signal.
Baselines~4--5 form a stepwise ablation of the $\mathrm{M}^{3}$ pipeline: Coverage-floor + mesh
sampling retains only the Step~1 partition, Variation+budget adds the Step~3 allocation, and
adding the Step~2 multi-scale grouping recovers full $\mathrm{M}^{3}$.
Baselines~1--5 are compared with $\mathrm{M}^{3}$ under a matched sampling budget.
All comparisons follow the full-mesh evaluation protocol of
Tables~\ref{tab:Evaluation-metrics-on-full-resolution-targets} and~\ref{tab:exp-scale-drivaerml} in
the main paper.

The baselines are defined as follows.
\begin{enumerate}
\item \textbf{Loss reweighting (RW)} (AhmedML).
If the evaluation target is the geometric (area/volume) measure, a standard correction is to keep
the sampled support fixed and reweight each loss contribution according to that measure, without
changing the data pipeline. We therefore use the same support as Random or $\mathrm{M}^{3}$ and set
$\mathcal{L}_i \leftarrow w_i\mathcal{L}_i$, with
$w_i \propto (\Delta S_i)^\beta$ / $(\Delta V_i)^\beta$.
The tempering parameter interpolates between equal weighting ($\beta{=}0$) and full geometric
weighting ($\beta{=}1$). We report $\beta{=}1$, corresponding to the full geometric weighting
scheme.

\item \textbf{Geometric-measure sampling} (AhmedML).
A closely related alternative corrects the sampling density at the support level rather than through
the loss. Specifically, points are sampled with probability
$p_i \propto (\Delta S_i)^\beta$ / $(\Delta V_i)^\beta$
under the same sampling budget $m$.
As in RW, $\beta$ interpolates between mesh-uniform sampling ($\beta{=}0$) and sampling proportional
to the geometric measure ($\beta{=}1$). We use $\beta{=}1$.
This procedure changes the spatial distribution of supervised points while leaving the loss
formulation unchanged. It therefore provides a direct density-correction comparator to
$\mathrm{M}^{3}$ without introducing variation-aware partitioning or allocation.

\item \textbf{Grid-stratified sampling} (DrivAerML).
The simplest structured alternative enforces spatial coverage using a fixed
Cartesian voxelization in place of the variation-aware partition. Points are stratified with spacing
$h=\max(L_x,L_y,L_z)/G$ ($G{=}64$ shared across cases, with other values of $G$ analyzed in
Appendix~\ref{app:baselines-grid}), with each occupied
voxel receiving an equal budget and samples drawn uniformly within each voxel.
This baseline tests whether spatially uniform coverage alone, without any physical-variation signal,
can replace the $\mathrm{M}^{3}$ construction.

\item \textbf{Coverage-floor + mesh sampling} (DrivAerML).
The first ablation retains only the Step~1 variation-aware partition and uses it to enforce
minimum spatial coverage. Each cell is guaranteed at least one sample, while
the remaining budget is allocated by mesh-uniform random sampling over the full
mesh. This baseline therefore tests whether variation-aware coverage combined
with the underlying mesh density is sufficient, without the multi-scale grouping
and capacity-constrained allocation used by $\mathrm{M}^{3}$.

\item \textbf{Variation+budget} (DrivAerML).
The second ablation builds on the coverage-floor baseline: we retain the Step~1
variation-aware partition and replace mesh-uniform allocation with the
capacity-constrained budget allocation from Step~3, while omitting the
scale grouping in Step~2. This combines the variation-aware partitioning
and budget allocation mechanisms of $\mathrm{M}^{3}$, leaving multi-scale
grouping as the only omitted stage.
\end{enumerate}

Baselines~1--2 require the geometric weights $\Delta V$/$\Delta S$ of every mesh entity
during training. These weights are not stored in the original releases and must be
exported from the raw meshes, which is expensive in both computation and storage for
large-scale meshes: on DrivAerML, the volume weights alone occupy approximately
$1.2\,\mathrm{GB}$ per case, or about $580\,\mathrm{GB}$ across $484$ runs, and a full
export is estimated to exceed ten days of preprocessing.
We therefore evaluate Baselines~1--2 on AhmedML, where the complete $500$-case export is
available. Its smaller meshes keep the export tractable while still allowing
geometric-measure correction through the loss (RW) to be distinguished from correction
through the sampling density (geometric-measure sampling).
By contrast, $\mathrm{M}^{3}$ and Baselines~3--5 do not require these weights during training.
In this work, the weights are exported only for the $50$ held-out cases of each dataset and
are used solely to compute the $\Delta S$/$\Delta V$-weighted evaluation metrics.
Baselines~3--5 are therefore evaluated on DrivAerML, which provides the largest meshes
($\sim\!1.6\times10^8$ volume entities per
case), the strongest non-uniformity, and the richest multi-scale structure among the three
datasets.

\subsection{AhmedML: Geometric-measure Baselines}
\label{app:baselines-ahmed}

\begin{table}[t]
\centering
\caption{AhmedML geometric-measure baselines (relative $L_2$, \%). Sampling budget, AB-UPT backbone, and held-out identifiers follow Table~\ref{tab:Evaluation-metrics-on-full-resolution-targets}. Equal-weight errors are on the left and $\Delta S$/$\Delta V$-weighted errors on the right.}
\label{tab:app-ahmed-baselines}
\vspace{5pt}
\small
\setlength{\fboxsep}{1.5pt}
\renewcommand{\tabhl}[1]{\colorbox{gray!20}{#1}}
\begin{tabular*}{\linewidth}{@{\extracolsep{\fill}}lcccc c cccc@{}}
\toprule
Method & $p_s$ & $\boldsymbol{\tau}$ & $p_v$ & $\mathbf{u}$ & & $p_{s,\omega}$
  & $\boldsymbol{\tau}_{\omega}$ & $p_{v,\omega}$ & $\mathbf{u}_{\omega}$ \\
\midrule
Geometric sampling & 3.93 & 5.44 & 4.19 & 3.47 & & 3.88 & 5.39 & 0.798 & 0.665 \\
RW-Random & 3.16 & \tabhl{3.87} & 3.62 & 2.94 & & 3.11 & \tabhl{3.45} & 0.588 & 0.491 \\
RW-$\mathrm{M}^{3}$ & \tabhl{2.63} & 4.16 & 3.76 & 3.18 & & \tabhl{2.60} & 3.91 & 0.464 & 0.384 \\
Random & 3.12 & 3.90 & 2.13 & 2.06 & & 3.07 & 3.48 & 0.613 & 0.510 \\
$\mathrm{M}^{3}$ & 2.64 & 3.97 & \tabhl{1.84} & \tabhl{1.73} & & 2.61 & 3.80 & \tabhl{0.455} & \tabhl{0.376} \\
\bottomrule
\end{tabular*}
\end{table}

In Table~\ref{tab:app-ahmed-baselines}, the \textbf{geometric sampling} baseline is consistently
worse than both Random and $\mathrm{M}^{3}$ across the reported channels. For the volume fields, its
equal-weight errors are $4.19\%$ and $3.47\%$ for $p_v$ and $\mathbf{u}$,
compared with $2.13\%$ and $2.06\%$ for Random and $1.84\%$ and $1.73\%$
for $\mathrm{M}^{3}$. The same ordering holds under $\Delta V$-weighted
evaluation, with errors of $0.798\%$ and $0.665\%$, compared with
$0.613\%$ and $0.510\%$ for Random and $0.455\%$ and $0.376\%$ for
$\mathrm{M}^{3}$. The gap is consistent with the different quantities used
to construct the sampling measure: geometric-measure sampling allocates points
according to $\Delta S$/$\Delta V$, whereas $\mathrm{M}^{3}$ allocates the
finite budget according to variation-aware cells and multi-scale capacity.
Under a fixed budget, geometric extent alone can therefore allocate more
supervision to large, relatively smooth regions while providing insufficient
coverage of finely resolved, high-variation regions.

For \textbf{loss reweighting}, the effect depends on the underlying support.
On Random, reweighting reduces the physics-weighted volume errors only slightly, from
$0.613\%$ to $0.588\%$ for $p_v$ and from $0.510\%$ to $0.491\%$ for
$\mathbf{u}$, while increasing their equal-weight errors from $2.13\%$ to
$3.62\%$ and from $2.06\%$ to $2.94\%$, respectively. This is expected from
the fixed support: RW changes the relative optimization weight of sampled
points but cannot provide supervision in regions that were not sampled.
The resulting physics-weighted volume errors also remain above those of $\mathrm{M}^{3}$.

Applying the same reweighting to the $\mathrm{M}^{3}$ support further separates
the effects of support construction and loss weighting. The physics-weighted volume
errors change only marginally, from $0.455\%$ to $0.464\%$ for $p_v$ and
from $0.376\%$ to $0.384\%$ for $\mathbf{u}$, whereas the equal-weight errors
increase from $1.84\%$ to $3.76\%$ and from $1.73\%$ to $3.18\%$,
respectively. The equal-weight volume errors of RW-$\mathrm{M}^{3}$ are also
close to those of RW-Random ($p_v{=}3.62\%$, $\mathbf{u}{=}2.94\%$).
Since $\mathrm{M}^{3}$ has already allocated its finite budget across
variation-aware, multi-scale regions, applying additional geometric loss
weights further shifts the optimization emphasis toward geometric measure.
The limited change in physics-weighted volume error together with the degradation
under equal-weight evaluation indicates that the additional weighting does
not provide a missing component of the $\mathrm{M}^{3}$ allocation.

Overall, the two alternatives modify different aspects of the supervision
measure. Geometric-measure sampling changes the sampling density according to geometric extent but
does not account for physical variation, while RW changes the loss on a fixed
support and cannot repair missing coverage. Applying RW on the
$\mathrm{M}^{3}$ support instead reweights an already variation-aware
allocation without improving the corresponding weighted metrics. These
comparisons therefore support the role of variation-aware support construction
and multi-scale budget allocation, rather than geometric-measure weighting
alone, in the observed gains of $\mathrm{M}^{3}$.

\subsection{\texorpdfstring{$\mathrm{M}^{3}$}{M3} module ablation}
\label{app:baselines-drivaer}

\begin{table}[ht]
\centering
\caption{Module ablation on DrivAerML (relative $L_2$, \%). Held-out identifiers follow Table~\ref{tab:exp-scale-drivaerml}. Equal-weight errors are on the left and $\Delta S$/$\Delta V$-weighted errors on the right. The ablation settings use approximately $10^6$ surface and $10^6$ volume points per case, comparable to $\mathrm{M}^{3}$ at surf$\times 0.1$/vol$\times 0.01$. Random uses surf$\times 1$/vol$\times 0.1$, the setting closest to the original mesh measure, and serves as the strongest mesh-measure reference.}
\label{tab:app-drivaer-baselines}
\vspace{5pt}
\small
\setlength{\fboxsep}{1.5pt}
\renewcommand{\tabhl}[1]{\colorbox{gray!35}{#1}\hspace{-\fboxsep}}
\newcommand{\tabhd}[1]{\setbox0\hbox{00.00}\hbox to \wd0{\hfil#1\hfil}}
\begin{tabular*}{\linewidth}{@{\extracolsep{\fill}}l*{4}{r}@{\hspace{2em}}*{4}{r}@{}}
\toprule
\rlap{Method}\hphantom{Coverage-floor + mesh} & \tabhd{$p_s$} & \tabhd{$\boldsymbol{\tau}$} & \tabhd{$p_v$} & \tabhd{$\mathbf{u}$} & \tabhd{$p_{s,\omega}$} & \tabhd{$\boldsymbol{\tau}_{\omega}$} & \tabhd{$p_{v,\omega}$} & \tabhd{$\mathbf{u}_{\omega}$} \\
\midrule
Random & \tabhl{3.30} & \tabhl{6.51} & \tabhl{5.61} & \tabhl{4.92} & 3.95 & 6.86 & 0.85 & 0.88 \\
Grid-stratified & 4.42 & 9.86 & 9.01 & 7.66 & \tabhl{3.84} & \tabhl{6.55} & \tabhl{0.23} & \tabhl{0.16} \\
Coverage-floor + mesh & 3.92 & 7.63 & 22.27 & 19.25 & 4.20 & 7.14 & 1.21 & 1.02 \\
Variation+budget & 4.14 & 8.92 & 21.65 & 18.44 & 3.85 & 6.58 & 1.05 & 0.80 \\
$\mathrm{M}^{3}$ & 4.10 & 8.27 & 17.46 & 15.05 & 4.11 & 6.86 & 0.30 & 0.26 \\
\bottomrule
\end{tabular*}
\end{table}

We conduct a stepwise ablation to isolate the respective contributions of each step of
$\mathrm{M}^{3}$ in Table~\ref{tab:app-drivaer-baselines}. Coverage-floor + mesh sampling combines
minimum spatial coverage with mesh-density sampling for the remaining budget. Across the four
physics-weighted channels it is essentially the worst method in the table. Simple spatial coverage
together with mesh-density sampling is therefore ineffective as a training measure.
Variation+budget adds the capacity-constrained allocation from Step~3 and improves all four
physics-weighted channels
($4.20/7.14/1.21/1.02\!\to\!3.85/6.58/1.05/0.80$). Full $\mathrm{M}^{3}$ then adds the Step~2
multi-scale grouping and further reduces the physics-weighted volume error by approximately
$3$--$4\times$ ($p_{v,\omega}/\mathbf{u}_{\omega}$ from $1.05/0.80$ to $0.30/0.26$), while the
physics-weighted surface errors remain comparable ($3.85/6.58$ versus $4.11/6.86$). The progression
therefore separates three complementary components. Step~1 alone does not yield an effective measure.
Step~2 organizes the support into multiple scales. Step~3 allocates the finite budget across these
scales and cells.

\subsection{DrivAerML: Grid-stratified Baselines}
\label{app:baselines-grid}

Among the baselines, Grid-stratified sampling achieves the lowest aggregate physics-weighted errors.
To better understand its mechanism and how the metrics change with spatial partitioning, we sweep
the voxel resolution $G\in\{32,64,128,256\}$ (Table~\ref{tab:app-drivaer-grid-G}). We do not
consider very large $G$, as finer partitions create more occupied voxels and raise computational
cost. The tested range is sufficient to examine how resolution reshapes the sampling allocation.

\begin{table}[ht]
\centering
\caption{DrivAerML Grid-stratified baselines by $G$ under full-mesh inference (relative $L_2$, \%, with approximately $10^6$ surface and $10^6$ volume points per case). Random is surf$\times 1$/vol$\times 0.1$ and $\mathrm{M}^{3}$ is surf$\times 0.1$/vol$\times 0.01$.}
\label{tab:app-drivaer-grid-G}
\vspace{5pt}
\small
\setlength{\fboxsep}{1.5pt}
\renewcommand{\tabhl}[1]{\colorbox{gray!35}{#1}\hspace{-\fboxsep}}
\newcommand{\tabhd}[1]{\setbox0\hbox{00.00}\hbox to \wd0{\hfil#1\hfil}}
\begin{tabular*}{\linewidth}{@{\extracolsep{\fill}}l*{4}{r}@{\hspace{2em}}*{4}{r}@{}}
\toprule
\rlap{Method}\hphantom{Coverage-floor + mesh} & \tabhd{$p_s$} & \tabhd{$\boldsymbol{\tau}$} & \tabhd{$p_v$} & \tabhd{$\mathbf{u}$} & \tabhd{$p_{s,\omega}$} & \tabhd{$\boldsymbol{\tau}_{\omega}$} & \tabhd{$p_{v,\omega}$} & \tabhd{$\mathbf{u}_{\omega}$} \\
\midrule
Grid ($G{=}32$) & 4.72 & 9.84 & 8.76 & 7.30 & 4.40 & 7.37 & 0.23 & 0.17 \\
Grid ($G{=}64$) & 4.42 & 9.86 & 9.01 & 7.66 & 3.84 & 6.55 & 0.23 & 0.16 \\
Grid ($G{=}128$) & 4.50 & 10.23 & 10.24 & 8.93 & 3.74 & 6.41 & \tabhl{0.21} & \tabhl{0.14} \\
Grid ($G{=}256$) & 4.37 & 9.75 & 13.28 & 10.82 & \tabhl{3.68} & \tabhl{6.30} & 0.23 & 0.15 \\
Random & \tabhl{3.30} & \tabhl{6.51} & \tabhl{5.61} & \tabhl{4.92} & 3.95 & 6.86 & 0.85 & 0.88 \\
$\mathrm{M}^{3}$ & 4.10 & 8.27 & 17.46 & 15.05 & 4.11 & 6.86 & 0.30 & 0.26 \\
\bottomrule
\end{tabular*}
\end{table}

Grid partitions the spatial domain using a fixed Cartesian lattice with resolution $G$. It allocates
an equal sampling budget to each occupied voxel and, within each voxel, samples mesh points without
replacement according to the local mesh measure. The allocation between near-wall and far-field
regions therefore depends on both the voxel partition and the mesh-point distribution within voxels.
As $G$ increases, the domain is divided into more voxels. Because the far field generally occupies a
much larger volume than the near-wall region, it contains more voxels and receives a larger share of
the budget. Sampling thus moves away from the raw mesh-density distribution toward more spatially
uniform coverage, which improves the physics-weighted errors while increasing the equal-weight
errors. This trend is bounded by far-field sparsity. At sufficiently high resolutions, many occupied
voxels contain only a single mesh point. Because sampling is restricted to mesh points and performed
without replacement, the extreme case in which every mesh point occupies its own voxel reduces Grid
sampling to uniform random sampling over the mesh nodes. Increasingly fine partitions therefore
first promote spatial coverage, but ultimately recover the global mesh measure.

Since aggregate metrics summarize errors across the entire field, we further compare spatial
accuracy. We evaluate volume error maps of Grid ($G{=}64$) and $\mathrm{M}^{3}$ on DrivAerML using a
shared subset formed by the union of each model's top-$p$\% absolute-error points for $p_v$ and
$\mathbf{u}$, with color scales shared within each field (Figure~\ref{fig:app-grid-m3-error}). A
full DrivAerML volume mesh contains about $1.6\times10^8$ entities, which is too large for the
memory and compute budget of this pointwise comparison. We therefore draw a uniform random candidate
pool of $5\times10^7$ volume query points, about one third of the full mesh ($\sim$160M),
and form the shared error subset from that pool.
\begin{figure}[ht]
  \centering
  \includegraphics[width=\linewidth]{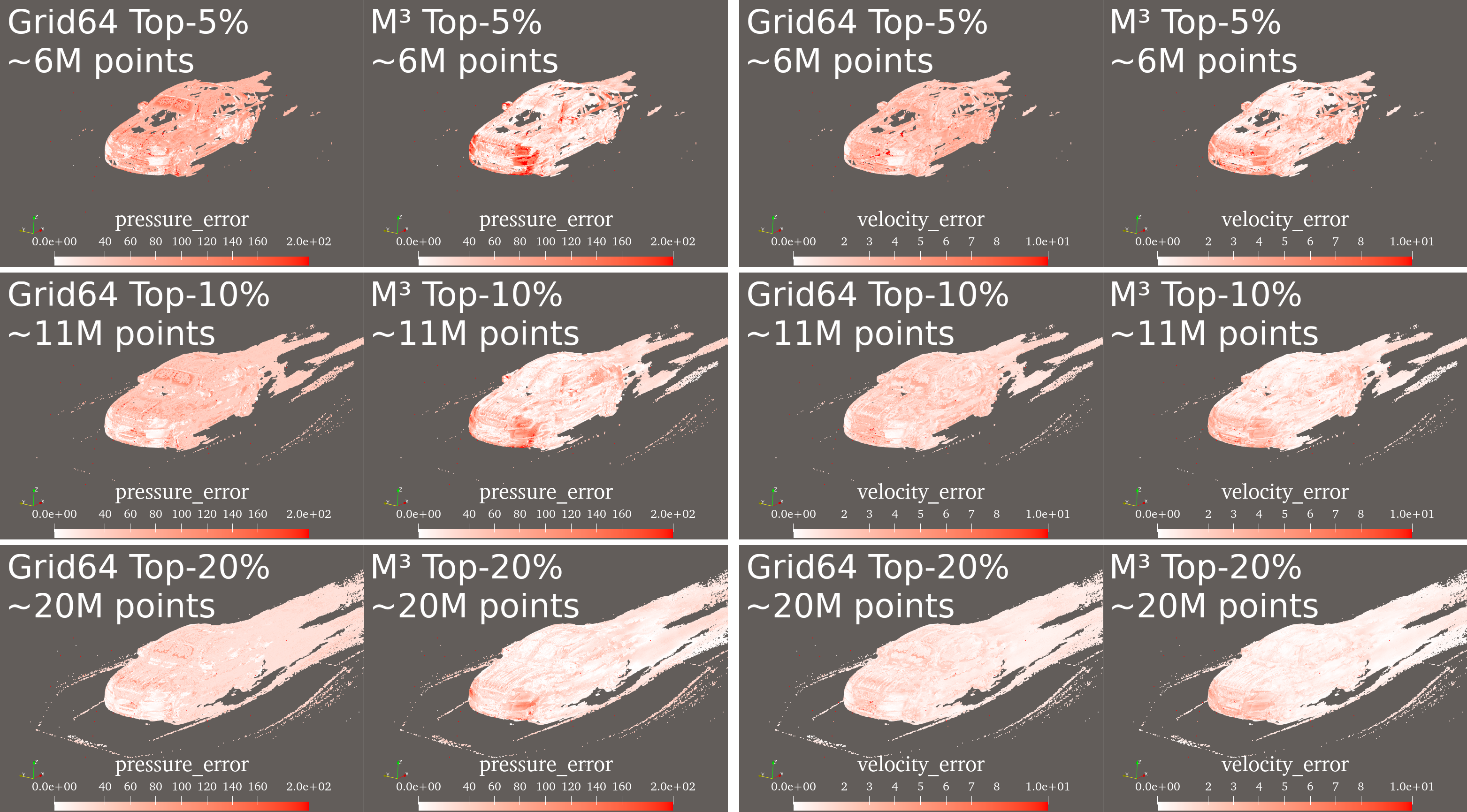}
  \caption{\textbf{Volume error maps of Grid ($G{=}64$) and $\mathrm{M}^{3}$ on DrivAerML run 278.}
  Each row shows the shared union subset of the two models' top-$p$\% absolute-error points for
  $p_v$ and $\mathbf{u}$, with $p=5$, $10$, and $20$.
  The subset size therefore exceeds $p$\% of the mesh.
  The left pair shows pressure errors and the right pair shows velocity errors.
  Color scales are shared within each field.
  Grid distributes moderate errors broadly over the body and wake, whereas the larger errors of
  $\mathrm{M}^{3}$ concentrate near the tire--ground contacts and in the near-ground region.}
  \label{fig:app-grid-m3-error}
\end{figure}
From the pattern of different top-$p$\% error maps, Grid and $\mathrm{M}^{3}$ show distinct
behaviors. Grid tends to spread similar moderate errors across the whole domain, whereas
$\mathrm{M}^{3}$ tends to concentrate larger errors in small dense regions. Compared with Grid,
$\mathrm{M}^{3}$ exhibits larger local errors around the rounded front corners. Physical
variation there is comparatively modest, so the region is assigned fewer scale strata and, under the
extreme volume downsampling ratio of $\times 0.01$, receives relatively little sampling budget.

These patterns reflect complementary alignments with the evaluation measures. By equalizing the
budget across spatial voxels, Grid guarantees that regions of different geometric cell sizes receive
a share of supervision, so its training measure is closer to a spatially stratified geometric
allocation and more aligned with the full-mesh $\Delta S$/$\Delta V$-weighted metrics. Under those
aggregates, Grid's lower errors are therefore expected. $\mathrm{M}^{3}$ instead allocates budgets
according to physical variation and is more closely matched to fidelity organized by variation
scales than by geometric cell size alone. Its higher full-mesh equal-weight volume errors reveal
that localized significant deviations inflate the aggregate and can mask the benefit of overall
spatial fidelity. On the same spatial maps, the $\mathrm{M}^{3}$-trained model keeps large errors
concentrated in small dense regions while recovering sharper structures in variation-sensitive
regions. Grid remains closer to a spatially uniform geometric allocation. Aggregate metrics and
spatial maps therefore capture complementary aspects of prediction fidelity. Grid is stronger on
full-mesh $\Delta S$/$\Delta V$-aligned aggregates, while $\mathrm{M}^{3}$ is better matched to
fidelity organized by physical variation scales.

\subsection{Summary Across Baselines}
\label{app:baselines-summary}

Taken together, the baselines and ablations above suggest that the gains of $\mathrm{M}^{3}$ are not
fully accounted for by any single alternative considered here. Geometric-measure sampling and loss
reweighting correct mesh density toward $\Delta S$/$\Delta V$, but neither closes the gap to
$\mathrm{M}^{3}$. The former reallocates points by geometric extent without tracking physical
variation, and the latter reweights a fixed support without repairing missing coverage. The module
ablations are likewise incomplete on their own. Coverage-floor + mesh sampling shows that minimum
spatial coverage with mesh-density sampling for the remaining budget is ineffective.
Variation+budget improves the aggregates further, but still falls short of full $\mathrm{M}^{3}$
until multi-scale grouping is restored. Grid-stratified sampling can achieve the lowest measure-aligned aggregate
errors by enforcing spatial uniformity, yet it spreads moderate errors over complex near-wall and
wake regions and does not recover the spatially structured fidelity of $\mathrm{M}^{3}$. These
comparisons therefore point to the same mechanism. $\mathrm{M}^{3}$ constructs the supervised
support and allocates the finite budget according to variation and scale, which is not replaced by
geometric-measure correction, loss reweighting, partial ablations, or spatially uniform
stratification alone. Aggregate metrics under a single evaluation protocol remain insufficient, and
the spatial structure of the errors should also be considered.

These baselines also clarify why the main experiments focus the primary comparison on
$\mathrm{M}^{3}$ versus uniform random sampling. Under equal-weight, average pointwise evaluation,
the metric is defined directly on the mesh measure itself. Random sampling draws from that same
measure, so the training distribution is naturally aligned with the evaluation measure and provides
a particularly competitive baseline. It is also the default practice in current industrial pipelines
for large-scale simulations.

\section{Robustness of \texorpdfstring{$\mathrm{M}^{3}$}{M3}}
\label{app:robustness}

We provide additional analyses to support the robustness and effectiveness of
the $\mathrm{M}^{3}$ pipeline across its full end-to-end scope, from underlying
data organization and the induced training measure to downstream model
performance. The results provide evidence that the observed gains are not specific
to a narrowly tuned parameter or a particular construction choice.

\subsection{Octree-Frame Robustness}
\label{app:octree-frame}

Morton/octree partitioning depends on the bounding box, origin, padding, and
orientation of the root cube. Changes in this construction frame modify cell
boundaries and can alter the sampled set even when the physical domain is
unchanged. We test this sensitivity on SHIFT-Wing, whose anisotropic AMR is the
most demanding setting for an isotropic octree among the three datasets. Keeping
the cases and training protocol fixed, we retrain under a joint stack of three
octree-frame transformations:

(i)~a $+10^\circ$ rotation of the octree frame about the $Z$ axis.

(ii)~$10\%$ padding of the Morton root cube.

(iii)~a random center shift with amplitude equal to $10\%$ of the padded cube size.

\begin{table}[ht]
\centering
\caption{Octree-frame perturbation on SHIFT-Wing. $\mathrm{M}^{3}$ uses the default frame and modification (mod) uses the jointly perturbed frame. $\Delta$ denotes the relative change $(\text{mod}-\mathrm{M}^{3})/\mathrm{M}^{3}$. Positive values indicate higher error under the perturbed frame.}
\label{tab:app-modframe}
\vspace{5pt}
\scriptsize
\setlength{\tabcolsep}{2pt}
\begin{tabular*}{\linewidth}{@{\extracolsep{\fill}}l ccc ccc ccc ccc@{}}
\toprule
 & \multicolumn{3}{c}{$p_s$} & \multicolumn{3}{c}{$\boldsymbol{\tau}$} & \multicolumn{3}{c}{$p_v$}
   & \multicolumn{3}{c@{}}{$\mathbf{u}$} \\
\cmidrule(lr){2-4}\cmidrule(lr){5-7}\cmidrule(lr){8-10}\cmidrule(l){11-13}
 & $\mathrm{M}^{3}$ & mod & $\Delta$ & $\mathrm{M}^{3}$ & mod & $\Delta$ & $\mathrm{M}^{3}$ & mod
   & $\Delta$ & $\mathrm{M}^{3}$ & mod & $\Delta$ \\
\midrule
\textbf{MSE} & 6334 & 7354 & $+16.11\%$ & 9.831 & 9.839 & $+0.08\%$ & 35593 & 33620 & $-5.54\%$
  & 978.3 & 979.3 & $+0.11\%$ \\
\textbf{MSE$_{\omega}$} & 5795 & 5967 & $+2.98\%$ & 19.52 & 19.53 & $+0.05\%$ & 776.1 & 870.7
  & $+12.19\%$ & 1.009 & 0.685 & $-32.06\%$ \\
\textbf{$L_2^{\mathrm{rel}}$ (\%)} & 0.338 & 0.359 & $+6.20\%$ & 18.91 & 18.92 & $+0.04\%$ & 0.774
  & 0.749 & $-3.24\%$ & 39.1 & 39.2 & $+0.02\%$ \\
\textbf{$L_{2,\omega}^{\mathrm{rel}}$ (\%)} & 0.317 & 0.319 & $+0.90\%$ & 17.67 & 17.67 & $+0.03\%$
  & 0.113 & 0.114 & $+0.65\%$ & 0.388 & 0.327 & $-15.65\%$ \\
\bottomrule
\end{tabular*}
\end{table}

Table~\ref{tab:app-modframe} compares the multi-seed default-frame mean with this
jointly trained modification (mod) model. Absolute shifts remain small overall, and several
physics-weighted metrics improve. The largest relative movement is the
$\mathrm{MSE}_{\omega}$ of $\mathbf{u}$ ($-32.06\%$), which is a change on a very
small absolute scale ($1.009\to 0.685$) and does not indicate degradation. Even
under the compounded setting, the gap stays small, which supports robustness of
$\mathrm{M}^{3}$ to octree-frame transformations.

This follows from the separation between spatial partitioning and budget
allocation in $\mathrm{M}^{3}$. The octree provides a scalable spatial
representation for localizing variation, while the allocation is determined by
the variation associated with the resulting cells rather than by fixed coordinate
locations. After the frame is perturbed, high-variation regions are still refined
into finer cells, grouped into the corresponding scale strata, and processed by
the same capacity-constrained water-filling procedure. Consequently, the
particular cell boundaries and sampled points may change, while the
variation-driven allocation remains comparatively stable. Robustness therefore
does not require an invariant discrete partition. Different valid partitions can
still preserve the variation structure that drives multi-scale allocation, so the
octree acts as a scalable carrier for the variation signal rather than as a
geometry-aligned substitute for the mesh. This is particularly relevant for
anisotropic meshes such as SHIFT-Wing.

\subsection{\texorpdfstring{Hyperparameter Sensitivity
($\varepsilon_{\mathrm{refine}}$, $G_{\max}$, $K$)}{Hyperparameter Sensitivity (epsilon refine, G max, K)}}
\label{app:hyperparam}

\begin{table}[ht]
\centering
\caption{Hyperparameter sensitivity on SHIFT-Wing volume.
Coarse~$\to$~Def~$\to$~Fine lists the swept values. SW$_1$ reports the
distance of the Coarse and Fine settings to the default.}
\label{tab:app-hyperparam}
\vspace{5pt}
\small
\begin{tabular*}{\linewidth}{@{\extracolsep{\fill}}lccc@{}}
\toprule
Param & Coarse $\to$ Def $\to$ Fine & SW$_1$ vs.\ default (Coarse / Fine) & $n_{\mathrm{eff}}$ ($K$) \\
\midrule
$\varepsilon_{\mathrm{refine}}$ & $0.01\to 0.005\to 0.001$ &
$0.076{\pm}0.022$ / $0.097{\pm}0.058$ & --- \\
$G_{\max}$ & $11\to 13\to 15$ &
$0.170{\pm}0.007$ / $0.131{\pm}0.063$ & --- \\
$K$ & $32\to 64\to 128$ &
$0.096{\pm}0.012$ / $0.085{\pm}0.011$ &
$2.55$ / $\mathbf{4.96}$ / $8.46$ \\
\bottomrule
\end{tabular*}
\end{table}

The $\mathrm{M}^{3}$ pipeline contains hyperparameters that control
variation-aware partitioning, refinement depth, and multi-scale grouping.
We characterize their sensitivity on SHIFT-Wing volume fields using three
cases with budget $m{=}600\mathrm{k}$ ($\approx 10\%$ of $\sim 6\mathrm{M}$
points). The default setting is
$(\varepsilon_{\mathrm{refine}},K,G_{\max})=(0.005,64,13)$. Each sweep is
centered on this default with the other parameters fixed.
Results are reported as mean $\pm$ standard deviation across cases.
As a diagnostic of the induced sampling measure rather than CFD prediction
error, we use the sliced Wasserstein-1 distance (SW$_1$) on normalized
position, velocity, and total pressure to quantify changes relative to the
default configuration. For the $K$ sweep, we additionally report the
effective number of occupied strata $n_{\mathrm{eff}}$, which measures how
broadly the sampling budget is distributed across the resulting scale
strata.
Table~\ref{tab:app-hyperparam} summarizes these diagnostics.

Across the tested range, $\varepsilon_{\mathrm{refine}}$ and $K$ produce
relatively small changes in the induced measure. Decreasing
$\varepsilon_{\mathrm{refine}}$ from $0.01$ to $0.001$ allows finer partitioning,
while the resulting SW$_1$ distances remain moderate. Increasing $K$ from
$32$ to $128$ increases $n_{\mathrm{eff}}$ from $2.55$ to $8.46$, corresponding
to a progressively finer multi-scale grouping, while the SW$_1$ distances
remain within a similar range. Thus, changes in refinement granularity and
stratification do not substantially alter the overall induced measure within
the tested range. In contrast, $G_{\max}$ shows greater sensitivity: both
shallower refinement ($G_{\max}{=}11$) and deeper refinement
($G_{\max}{=}15$) produce larger distributional deviations, with
$G_{\max}{=}15$ also exhibiting higher between-case variability. This indicates
that the maximum refinement depth has a stronger effect on the resulting
measure than the tested changes in $\varepsilon_{\mathrm{refine}}$ or $K$,
and that the default configuration lies in a stable operating region within
the tested range.

\subsection{\texorpdfstring{Prediction-Accuracy Ablation of the Stratum Count $K$}{Prediction-Accuracy Ablation of the Stratum Count K}}
\label{app:ablate-K}

Section~\ref{app:baselines-drivaer} ablates the $\mathrm{M}^{3}$ modules.
Here we also vary the stratum count $K\in\{32,64,128\}$ on DrivAerML under
the same $\times 0.01/\times 0.01$ budget as Table~\ref{tab:exp-scale-drivaerml},
with default $K{=}64$ corresponding to the main paper's $\mathrm{M}^{3}$ column.
As the resolution of multi-scale stratification, $K$ controls how finely the
sampling budget is split across scale strata, and is therefore central to the
$\mathrm{M}^{3}$ construction. This ablation asks how different scale splits
affect downstream model learning under a fixed budget.

\begin{table}[ht]
\centering
\caption{Ablation of the stratum count $K$ on DrivAerML.}
\label{tab:app-K-ablation}
\small
\newcommand{\tabar}[1]{\setbox0\hbox{$\mathbf{0.000}\,\downarrow$}\hbox to \wd0{\hfil#1}}
\begin{tabular*}{0.8\linewidth}{@{\extracolsep{\fill}}lcccc@{}}
\toprule
Metric & Random $\times 0.01/\times 0.01$ & $K{=}32$ & $K{=}64$ (default) & $K{=}128$ \\
\midrule
\textbf{$p_s$ $L_2^{\mathrm{rel}}$} & 3.934 & \tabar{$4.082\,\tabup$} & \tabar{$4.557\,\tabup$}
  & \tabar{$4.125\,\tabup$} \\
\textbf{$p_s$ $L_{2,\omega}^{\mathrm{rel}}$} & 4.408 & \tabar{$\mathbf{3.663}\,\tabdn$}
  & \tabar{$\mathbf{4.268}\,\tabdn$} & \tabar{$\mathbf{3.734}\,\tabdn$} \\
\textbf{$\boldsymbol{\tau}$ $L_2^{\mathrm{rel}}$} & 7.601 & \tabar{$8.724\,\tabup$}
  & \tabar{$9.476\,\tabup$} & \tabar{$8.781\,\tabup$} \\
\textbf{$\boldsymbol{\tau}$ $L_{2,\omega}^{\mathrm{rel}}$} & 7.471
  & \tabar{$\mathbf{6.387}\,\tabdn$} & \tabar{$\mathbf{7.165}\,\tabdn$}
  & \tabar{$\mathbf{6.456}\,\tabdn$} \\
\textbf{$p_v$ $L_2^{\mathrm{rel}}$} & 6.487 & \tabar{$16.85\,\tabup$} & \tabar{$17.87\,\tabup$}
  & \tabar{$17.07\,\tabup$} \\
\textbf{$p_v$ $L_{2,\omega}^{\mathrm{rel}}$} & 0.931 & \tabar{$\mathbf{0.290}\,\tabdn$}
  & \tabar{$\mathbf{0.302}\,\tabdn$} & \tabar{$\mathbf{0.306}\,\tabdn$} \\
\textbf{$\mathbf{u}$ $L_2^{\mathrm{rel}}$} & 5.713 & \tabar{$14.37\,\tabup$}
  & \tabar{$15.00\,\tabup$} & \tabar{$14.18\,\tabup$} \\
\textbf{$\mathbf{u}$ $L_{2,\omega}^{\mathrm{rel}}$} & 0.938 & \tabar{$\mathbf{0.242}\,\tabdn$}
  & \tabar{$\mathbf{0.268}\,\tabdn$} & \tabar{$\mathbf{0.250}\,\tabdn$} \\
\bottomrule
\end{tabular*}
\end{table}
Table~\ref{tab:app-K-ablation} reports equal-weight and
$\Delta S$/$\Delta V$-weighted relative $L_2$ (\%) under the same protocol
as the main experiments, with full-mesh evaluation. The three tested values of $K$ exhibit the same qualitative behavior relative
to Random. For all $K\in\{32,64,128\}$, $\Delta S$/$\Delta V$-weighted surface
and volume errors are consistently lower than Random, whereas equal-weight
volume errors are consistently higher. Thus, the same separation between
physics-weighted and equal-weight evaluation is preserved across
the entire tested range of stratum counts.

Importantly, $K{=}64$ does not consistently yield the lowest error. Both
$K{=}32$ and $K{=}128$ outperform $K{=}64$ on several physics-weighted metrics. For
example, the weighted wall-shear-stress error is $6.387$, $7.165$, and $6.456$ for
$K{=}32$, $64$, and $128$, respectively, while the corresponding velocity
errors are $0.242$, $0.268$, and $0.250$. The observed behavior therefore
does not rely on a narrowly tuned stratum count, and we use $K{=}64$ as a
representative intermediate setting within this fourfold range.

\subsection{Multi-Seed Robustness}
\label{app:multiseed}

To further assess the robustness and reproducibility of the reported gains, we re-evaluate the
Random versus $\mathrm{M}^{3}$ comparisons under the same full-mesh protocol with three
training seeds ($0/1/2$). This provides a multi-seed view of the performance differences across
datasets and fields, particularly for cases where the differences are small.
Tables~\ref{tab:app-seed-ahmed}, \ref{tab:app-seed-drivaer}, and \ref{tab:app-seed-wing}
report AhmedML, DrivAerML, and SHIFT-Wing, respectively. AhmedML follows the paper scaling
(MSE $\times10^{-6}$, MAE $\times10^{-3}$). In these tables, arrows indicate whether the
$\mathrm{M}^{3}$ mean is lower ($\downarrow$) or higher ($\uparrow$) than Random.

\begin{table}[ht]
\centering
\caption{Multi-seed robustness on AhmedML (seeds $0/1/2$). Both R and $\mathrm{M}^{3}$ are mean$\pm$std.}
\label{tab:app-seed-ahmed}
\scriptsize
\renewcommand{\tabup}{\textcolor{red}{\uparrow}}
\renewcommand{\tabdn}{\textcolor{green!60!black}{\downarrow}}
\setlength{\tabcolsep}{3pt}
\resizebox{\linewidth}{!}{%
\begin{tabular}{l rr rr rr rr}
\toprule
& \multicolumn{2}{c}{$p_s$} & \multicolumn{2}{c}{$\boldsymbol{\tau}$} & \multicolumn{2}{c}{$p_v$}
  & \multicolumn{2}{c}{$\mathbf{u}$} \\
\cmidrule(lr){2-3}\cmidrule(lr){4-5}\cmidrule(lr){6-7}\cmidrule(lr){8-9}
& \multicolumn{1}{c}{R} & \multicolumn{1}{c}{$\mathrm{M}^{3}$} & \multicolumn{1}{c}{R}
  & \multicolumn{1}{c}{$\mathrm{M}^{3}$} & \multicolumn{1}{c}{R}
  & \multicolumn{1}{c}{$\mathrm{M}^{3}$} & \multicolumn{1}{c}{R}
  & \multicolumn{1}{c}{$\mathrm{M}^{3}$} \\
\midrule
\textbf{MAE} & $2.56{\pm}0.30$ & $2.09{\pm}0.15\,\tabdn$ & $0.0130{\pm}0.0007$
  & $0.01478{\pm}0.00023\,\tabup$ & $6.11{\pm}0.17$ & $5.43{\pm}0.11\,\tabdn$ & $3.64{\pm}0.14$
  & $3.10{\pm}0.02\,\tabdn$ \\
\textbf{MAE$_{\omega}$} & $2.32{\pm}0.28$ & $1.91{\pm}0.13\,\tabdn$ & $0.0122{\pm}0.0007$
  & $0.01467{\pm}0.00014\,\tabup$ & $1.84{\pm}0.09$ & $1.39{\pm}0.01\,\tabdn$ & $0.951{\pm}0.055$
  & $0.679{\pm}0.021\,\tabdn$ \\
\textbf{MSE} & $115{\pm}32$ & $62.8{\pm}11.6\,\tabdn$ & $0.00338{\pm}0.00025$
  & $0.00319{\pm}0.00013\,\tabdn$ & $671{\pm}77$ & $358{\pm}44\,\tabdn$ & $254{\pm}35$
  & $123{\pm}14\,\tabdn$ \\
\textbf{MSE$_{\omega}$} & $110{\pm}27$ & $61{\pm}8\,\tabdn$ & $0.00908{\pm}0.00066$
  & $0.00959{\pm}0.00036\,\tabup$ & $89.2{\pm}10.8$ & $31.0{\pm}3.5\,\tabdn$ & $61.7{\pm}7.6$
  & $20.6{\pm}2.5\,\tabdn$ \\
\textbf{$L_2^{\mathrm{rel}}$ (\%)} & $3.12{\pm}0.29$ & $2.64{\pm}0.17\,\tabdn$ & $3.90{\pm}0.14$
  & $3.97{\pm}0.02\,\tabup$ & $2.13{\pm}0.08$ & $1.84{\pm}0.06\,\tabdn$ & $2.06{\pm}0.11$
  & $1.73{\pm}0.03\,\tabdn$ \\
\textbf{$L_{2,\omega}^{\mathrm{rel}}$ (\%)} & $3.07{\pm}0.28$ & $2.61{\pm}0.16\,\tabdn$
  & $3.48{\pm}0.13$ & $3.80{\pm}0.02\,\tabup$ & $0.613{\pm}0.023$ & $0.455{\pm}0.003\,\tabdn$
  & $0.510{\pm}0.021$ & $0.376{\pm}0.004\,\tabdn$ \\
\bottomrule
\end{tabular}%
}
\end{table}

\begin{table}[ht]
\centering
\caption{Multi-seed robustness on DrivAerML (seeds $0/1/2$). Both R and $\mathrm{M}^{3}$ are mean$\pm$std.}
\label{tab:app-seed-drivaer}
\scriptsize
\renewcommand{\tabup}{\textcolor{red}{\uparrow}}
\renewcommand{\tabdn}{\textcolor{green!60!black}{\downarrow}}
\setlength{\tabcolsep}{2.5pt}
\resizebox{\linewidth}{!}{%
\begin{tabular}{l rr rr rr rr}
\toprule
& \multicolumn{2}{c}{$p_s$} & \multicolumn{2}{c}{$\boldsymbol{\tau}$} & \multicolumn{2}{c}{$p_v$}
  & \multicolumn{2}{c}{$\mathbf{u}$} \\
\cmidrule(lr){2-3}\cmidrule(lr){4-5}\cmidrule(lr){6-7}\cmidrule(lr){8-9}
& \multicolumn{1}{c}{R} & \multicolumn{1}{c}{$\mathrm{M}^{3}$} & \multicolumn{1}{c}{R}
  & \multicolumn{1}{c}{$\mathrm{M}^{3}$} & \multicolumn{1}{c}{R}
  & \multicolumn{1}{c}{$\mathrm{M}^{3}$} & \multicolumn{1}{c}{R}
  & \multicolumn{1}{c}{$\mathrm{M}^{3}$} \\
\midrule
\textbf{MAE} & $8.41{\pm}0.12$ & $8.38{\pm}0.17\,\tabdn$ & $0.0689{\pm}0.0002$
  & $0.0735{\pm}0.0014\,\tabup$ & $0.0179{\pm}0.0002$ & $0.0299{\pm}0.0009\,\tabup$
  & $0.500{\pm}0.005$ & $0.815{\pm}0.019\,\tabup$ \\
\textbf{MAE$_{\omega}$} & $6.33{\pm}0.11$ & $5.50{\pm}0.12\,\tabdn$ & $0.0524{\pm}0.0003$
  & $0.0457{\pm}0.0009\,\tabdn$ & $0.0050{\pm}0.0002$ & $0.0027{\pm}0.0001\,\tabdn$
  & $0.132{\pm}0.004$ & $0.072{\pm}0.006\,\tabdn$ \\
\textbf{MSE} & $204.2{\pm}6.1$ & $191.1{\pm}7.7\,\tabdn$ & $0.01810{\pm}0.00032$
  & $0.01846{\pm}0.00057\,\tabup$ & $0.0011611{\pm}0.0000226$ & $0.0044956{\pm}0.0002953\,\tabup$
  & $0.777{\pm}0.009$ & $2.419{\pm}0.111\,\tabup$ \\
\textbf{MSE$_{\omega}$} & $119.9{\pm}4.6$ & $91.8{\pm}4.0\,\tabdn$ & $0.02581{\pm}0.00062$
  & $0.02025{\pm}0.00085\,\tabdn$ & $0.0000839{\pm}0.0000096$ & $0.0000234{\pm}0.0000007\,\tabdn$
  & $0.13683{\pm}0.01238$ & $0.04132{\pm}0.00409\,\tabdn$ \\
\textbf{$L_2^{\mathrm{rel}}$ (\%)} & $3.94{\pm}0.05$ & $3.82{\pm}0.08\,\tabdn$ & $7.65{\pm}0.05$
  & $7.76{\pm}0.13\,\tabup$ & $6.37{\pm}0.07$ & $12.73{\pm}0.42\,\tabup$ & $5.67{\pm}0.04$
  & $10.23{\pm}0.24\,\tabup$ \\
\textbf{$L_{2,\omega}^{\mathrm{rel}}$ (\%)} & $4.36{\pm}0.07$ & $3.80{\pm}0.07\,\tabdn$
  & $7.38{\pm}0.07$ & $6.53{\pm}0.13\,\tabdn$ & $0.91{\pm}0.05$ & $0.48{\pm}0.01\,\tabdn$
  & $0.95{\pm}0.04$ & $0.52{\pm}0.02\,\tabdn$ \\
\bottomrule
\end{tabular}%
}
\end{table}

\begin{table}[ht]
\centering
\caption{Multi-seed robustness on SHIFT-Wing (seeds $0/1/2$). Both R and $\mathrm{M}^{3}$ are mean$\pm$std.}
\label{tab:app-seed-wing}
\scriptsize
\renewcommand{\tabup}{\textcolor{red}{\uparrow}}
\renewcommand{\tabdn}{\textcolor{green!60!black}{\downarrow}}
\setlength{\tabcolsep}{2.5pt}
\resizebox{\linewidth}{!}{%
\begin{tabular}{l rr rr rr rr}
\toprule
& \multicolumn{2}{c}{$p_s$} & \multicolumn{2}{c}{$\boldsymbol{\tau}$} & \multicolumn{2}{c}{$p_v$}
  & \multicolumn{2}{c}{$\mathbf{u}$} \\
\cmidrule(lr){2-3}\cmidrule(lr){4-5}\cmidrule(lr){6-7}\cmidrule(lr){8-9}
& \multicolumn{1}{c}{R} & \multicolumn{1}{c}{$\mathrm{M}^{3}$} & \multicolumn{1}{c}{R}
  & \multicolumn{1}{c}{$\mathrm{M}^{3}$} & \multicolumn{1}{c}{R}
  & \multicolumn{1}{c}{$\mathrm{M}^{3}$} & \multicolumn{1}{c}{R}
  & \multicolumn{1}{c}{$\mathrm{M}^{3}$} \\
\midrule
\textbf{MAE} & $34.2{\pm}0.2$ & $36.3{\pm}1.4\,\tabup$ & $1.483{\pm}0.001$
  & $1.483{\pm}0.002\,\tabup$ & $69.2{\pm}3.1$ & $67.4{\pm}1.2\,\tabdn$ & $17.52{\pm}1.38$
  & $15.33{\pm}0.24\,\tabdn$ \\
\textbf{MAE$_{\omega}$} & $39.8{\pm}0.7$ & $38.0{\pm}1.8\,\tabdn$ & $1.133{\pm}0.001$
  & $1.129{\pm}0.002\,\tabdn$ & $32.9{\pm}1.6$ & $18.9{\pm}0.7\,\tabdn$ & $1.007{\pm}0.298$
  & $0.322{\pm}0.016\,\tabdn$ \\
\textbf{MSE} & $7546{\pm}622$ & $6334{\pm}517\,\tabdn$ & $9.830{\pm}0.008$
  & $9.831{\pm}0.008\,\tabup$ & $39572{\pm}3629$ & $35593{\pm}403\,\tabdn$ & $1192{\pm}176$
  & $978{\pm}30\,\tabdn$ \\
\textbf{MSE$_{\omega}$} & $7905{\pm}733$ & $5795{\pm}546\,\tabdn$ & $19.597{\pm}0.023$
  & $19.518{\pm}0.026\,\tabdn$ & $2175{\pm}173$ & $776{\pm}121\,\tabdn$ & $28.50{\pm}26.79$
  & $1.009{\pm}0.190\,\tabdn$ \\
\textbf{$L_2^{\mathrm{rel}}$ (\%)} & $0.357{\pm}0.013$ & $0.338{\pm}0.013\,\tabdn$
  & $18.91{\pm}0.01$ & $18.91{\pm}0.01\,\tabup$ & $0.808{\pm}0.032$ & $0.774{\pm}0.009\,\tabdn$
  & $43.1{\pm}3.2$ & $39.1{\pm}0.6\,\tabdn$ \\
\textbf{$L_{2,\omega}^{\mathrm{rel}}$ (\%)} & $0.363{\pm}0.010$ & $0.317{\pm}0.015\,\tabdn$
  & $17.71{\pm}0.01$ & $17.67{\pm}0.01\,\tabdn$ & $0.19{\pm}0.01$ & $0.11{\pm}0.01\,\tabdn$
  & $1.92{\pm}1.06$ & $0.39{\pm}0.03\,\tabdn$ \\
\bottomrule
\end{tabular}%
}
\end{table}

Across all three datasets, seed-to-seed variation remains small relative to the
Random--$\mathrm{M}^{3}$ gap.

\section{Fine-Grained Quantitative and Localized Analysis}
\label{app:fidelity}

On highly non-uniform meshes such as DrivAerML, equal-weight and
physics-weighted metrics can disagree in ranking. Both are aggregate scores,
so we turn to finer-grained quantitative and local analyses to see where the
gains and trade-offs arise across mesh scales and flow structures, including
boundary layers, wakes, and separation zones. We provide a size-stratified
error analysis under distinct mesh regimes, followed by localized visualizations.

To examine whether the overall improvements obscure errors at finer spatial scales, we perform a
scale-resolved analysis of prediction errors across different mesh-entity sizes. This stratification
reveals which spatial scales contribute to the observed improvements, rather than relying solely on
aggregate metrics whose values can be dominated by the geometric weighting of mesh entities. In
particular, $\Delta V$/$\Delta S$-weighted metrics give greater influence to larger geometric
entities, so aggregate improvements can mask substantial errors concentrated in smaller entities,
including those in near-wall and boundary-layer regions.

We therefore regroup all evaluation points by their per-point geometric weights $\Delta S$ on the
surface and $\Delta V$ in the volume, and examine the prediction errors within each size range. On
the surface, smaller $\Delta S$ generally corresponds to locally refined regions with stronger
curvature or rapid geometric variation, whereas larger $\Delta S$ is more common in smoother,
less-refined regions. In the volume, the smallest $\Delta V$ typically corresponds to near-wall
boundary layers, while other small $\Delta V$ points often lie in locally refined wakes. Larger
$\Delta V$ is more prevalent farther from the wall in the outer flow. This size-based stratification
therefore provides a more fine-grained quantitative view of prediction accuracy across point-wise
$\Delta S$/$\Delta V$ scales and associated flow structures, complementing the aggregate evaluation
metrics.

Based on that, we divide surface and volume points into logarithmic $\Delta S$/$\Delta V$ bins, with
V0/S0 corresponding to the smallest weights and the index increasing with $\Delta S$ or $\Delta V$.
The number of bins is adjusted by dataset. SHIFT-Wing uses finer binning because AMR refinement
produces a strong concentration of points within the same size range.
Tables~\ref{tab:app-sizebin-wing}--\ref{tab:app-sizebin-dri} report the mean results across the
three cases for each dataset. For each bin, both equal-weight and $\Delta V$/$\Delta S$-weighted
relative $L_2$ errors are computed using only the points within that bin. Size is the bin's
$\Delta S$ or $\Delta V$ range. Frac.\ is the percentage of points in that bin among all surface or
volume points on the mesh, reported for a representative case (1801, 216, and 137 for SHIFT-Wing,
AhmedML, and DrivAerML, respectively). Eq.\ and Wt.\ denote equal-weight and physics-weighted
evaluation. The win column marks whether $\mathrm{M}^{3}$ is better (green checkmark), worse (red
cross), or equal (--) relative to Random under that metric.


\begin{table}[ht]
\centering
\caption{Size-stratified evaluation on SHIFT-Wing.
Cases: 1801, 1802, 1803.}
\label{tab:app-sizebin-wing}
\scriptsize
\setlength{\tabcolsep}{3pt}
\resizebox{0.75\linewidth}{!}{%
\begin{tabular}{llrrrrcrrc}
\toprule
Field & Bin & Size & Frac.\ & R-Eq. & $\mathrm{M}^{3}$-Eq. & win & R-Wt. & $\mathrm{M}^{3}$-Wt. & win \\
\midrule
$\boldsymbol{\tau}$ & overall & $10^{-3}{\sim}10^{-2}$ & 100\% & 0.189 & 0.189 & -- & 0.177 & 0.177 & -- \\
$\boldsymbol{\tau}$ & S0 & $10^{-3}$ & 0.36\% & 0.249 & 0.246
  & \textcolor{green!60!black}{\checkmark} & 0.244 & 0.244 & -- \\
$\boldsymbol{\tau}$ & S1 & $10^{-3}$ & 2.05\% & 0.219 & 0.219 & -- & 0.218 & 0.218 & -- \\
$\boldsymbol{\tau}$ & S2 & $10^{-3}{\sim}10^{-2}$ & 16.08\% & 0.202 & 0.202 & -- & 0.201 & 0.201 & -- \\
$\boldsymbol{\tau}$ & S3 & $10^{-2}$ & 38.17\% & 0.184 & 0.184 & -- & 0.183 & 0.184 & \textcolor{red}{$\times$} \\
$\boldsymbol{\tau}$ & S4 & $10^{-2}$ & 28.61\% & 0.186 & 0.186 & -- & 0.185 & 0.185 & -- \\
$\boldsymbol{\tau}$ & S5 & $10^{-2}$ & 13.94\% & 0.166 & 0.166 & -- & 0.165 & 0.164
  & \textcolor{green!60!black}{\checkmark} \\
$\boldsymbol{\tau}$ & S6 & $10^{-2}$ & 0.79\% & 0.132 & 0.129
  & \textcolor{green!60!black}{\checkmark} & 0.117 & 0.114
  & \textcolor{green!60!black}{\checkmark} \\
\midrule
$\mathbf{u}$ & overall & $10^{-4}{\sim}10^{1}$ & 100\% & 0.424 & 0.377
  & \textcolor{green!60!black}{\checkmark} & 0.0191 & 0.00378
  & \textcolor{green!60!black}{\checkmark} \\
$\mathbf{u}$ & V0 & $10^{-4}{\sim}10^{-3}$ & 45.08\% & 1.216 & 1.050
  & \textcolor{green!60!black}{\checkmark} & 0.942 & 0.824
  & \textcolor{green!60!black}{\checkmark} \\
$\mathbf{u}$ & V1 & $10^{-3}{\sim}10^{-2}$ & 32.37\% & 0.343 & 0.348 & \textcolor{red}{$\times$}
  & 0.193 & 0.239 & \textcolor{red}{$\times$} \\
$\mathbf{u}$ & V2 & $10^{-2}$ & 7.73\% & 0.053 & 0.079 & \textcolor{red}{$\times$} & 0.036 & 0.049
  & \textcolor{red}{$\times$} \\
$\mathbf{u}$ & V3 & $10^{-2}{\sim}10^{-1}$ & 6.99\% & 0.026 & 0.016
  & \textcolor{green!60!black}{\checkmark} & 0.023 & 0.009
  & \textcolor{green!60!black}{\checkmark} \\
$\mathbf{u}$ & V4 & $10^{-1}{\sim}10^{0}$ & 6.18\% & 0.013 & 0.0047
  & \textcolor{green!60!black}{\checkmark} & 0.0085 & 0.0039
  & \textcolor{green!60!black}{\checkmark} \\
$\mathbf{u}$ & V5 & $10^{0}$ & 1.24\% & 0.008 & 0.0034 & \textcolor{green!60!black}{\checkmark}
  & 0.0072 & 0.0031 & \textcolor{green!60!black}{\checkmark} \\
$\mathbf{u}$ & V6 & $10^{0}{\sim}10^{1}$ & 0.42\% & 0.013 & 0.0034
  & \textcolor{green!60!black}{\checkmark} & 0.019 & 0.0038
  & \textcolor{green!60!black}{\checkmark} \\
\bottomrule
\end{tabular}
}
\end{table}

\begin{table}[ht]
\centering
\caption{Size-stratified evaluation on AhmedML.
Cases: 11, 150, 216.}
\label{tab:app-sizebin-ah}
\scriptsize
\setlength{\tabcolsep}{3pt}
\resizebox{0.75\linewidth}{!}{%
\begin{tabular}{llrrrrcrrc}
\toprule
Field & Bin & Size & Frac.\ & R-Eq. & $\mathrm{M}^{3}$-Eq. & win & R-Wt. & $\mathrm{M}^{3}$-Wt. & win \\
\midrule
$\boldsymbol{\tau}$ & overall & $10^{-4}{\sim}10^{-3}$ & 100\% & 0.028 & 0.026
  & \textcolor{green!60!black}{\checkmark} & 0.020 & 0.020 & -- \\
$\boldsymbol{\tau}$ & S0 & $10^{-4}$ & 3.86\% & 0.063 & 0.051
  & \textcolor{green!60!black}{\checkmark} & 0.048 & 0.039
  & \textcolor{green!60!black}{\checkmark} \\
$\boldsymbol{\tau}$ & S1 & $10^{-4}{\sim}10^{-3}$ & 16.55\% & 0.029 & 0.035
  & \textcolor{red}{$\times$} & 0.030 & 0.037 & \textcolor{red}{$\times$} \\
$\boldsymbol{\tau}$ & S2 & $10^{-3}$ & 40.23\% & 0.019 & 0.020 & \textcolor{red}{$\times$} & 0.019
  & 0.020 & \textcolor{red}{$\times$} \\
$\boldsymbol{\tau}$ & S3 & $10^{-3}$ & 38.57\% & 0.015 & 0.015 & -- & 0.015 & 0.015 & -- \\
$\boldsymbol{\tau}$ & S4 & $10^{-3}$ & 0.78\% & 0.016 & 0.016 & -- & 0.016 & 0.016 & -- \\
\midrule
$\mathbf{u}$ & overall & $10^{-4}{\sim}10^{-2}$ & 100\% & 0.00992 & 0.00989
  & \textcolor{green!60!black}{\checkmark} & 0.00256 & 0.00225
  & \textcolor{green!60!black}{\checkmark} \\
$\mathbf{u}$ & V0 & $10^{-4}{\sim}10^{-3}$ & 5.21\% & 0.024 & 0.025 & \textcolor{red}{$\times$}
  & 0.023 & 0.024 & \textcolor{red}{$\times$} \\
$\mathbf{u}$ & V1 & $10^{-3}$ & 26.06\% & 0.014 & 0.015 & \textcolor{red}{$\times$} & 0.014 & 0.015
  & \textcolor{red}{$\times$} \\
$\mathbf{u}$ & V2 & $10^{-3}$ & 44.75\% & 0.00908 & 0.00829
  & \textcolor{green!60!black}{\checkmark} & 0.00751 & 0.00733
  & \textcolor{green!60!black}{\checkmark} \\
$\mathbf{u}$ & V3 & $10^{-3}{\sim}10^{-2}$ & 16.39\% & 0.00252 & 0.00247
  & \textcolor{green!60!black}{\checkmark} & 0.00244 & 0.00238
  & \textcolor{green!60!black}{\checkmark} \\
$\mathbf{u}$ & V4 & $10^{-2}$ & 7.60\% & 0.0021 & 0.00179 & \textcolor{green!60!black}{\checkmark}
  & 0.00219 & 0.00182 & \textcolor{green!60!black}{\checkmark} \\
\bottomrule
\end{tabular}
}
\end{table}

\begin{table}[ht]
\centering
\caption{Size-stratified evaluation on DrivAerML.
Cases: 137, 186, 278.}
\label{tab:app-sizebin-dri}
\scriptsize
\setlength{\tabcolsep}{3pt}
\resizebox{0.75\linewidth}{!}{%
\begin{tabular}{llrrrrcrrc}
\toprule
Field & Bin & Size & Frac.\ & R-Eq. & $\mathrm{M}^{3}$-Eq. & win & R-Wt. & $\mathrm{M}^{3}$-Wt. & win \\
\midrule
$\boldsymbol{\tau}$ & overall & $10^{-5}{\sim}10^{-2}$ & 100\% & 0.078 & 0.075
  & \textcolor{green!60!black}{\checkmark} & 0.071 & 0.062
  & \textcolor{green!60!black}{\checkmark} \\
$\boldsymbol{\tau}$ & S0 & $10^{-5}{\sim}10^{-4}$ & 11.35\% & 0.076 & 0.072
  & \textcolor{green!60!black}{\checkmark} & 0.077 & 0.072
  & \textcolor{green!60!black}{\checkmark} \\
$\boldsymbol{\tau}$ & S1 & $10^{-4}$ & 29.51\% & 0.084 & 0.085 & \textcolor{red}{$\times$} & 0.094
  & 0.096 & \textcolor{red}{$\times$} \\
$\boldsymbol{\tau}$ & S2 & $10^{-4}{\sim}10^{-3}$ & 23.94\% & 0.073 & 0.071
  & \textcolor{green!60!black}{\checkmark} & 0.071 & 0.067
  & \textcolor{green!60!black}{\checkmark} \\
$\boldsymbol{\tau}$ & S3 & $10^{-3}$ & 31.76\% & 0.079 & 0.071
  & \textcolor{green!60!black}{\checkmark} & 0.082 & 0.072
  & \textcolor{green!60!black}{\checkmark} \\
$\boldsymbol{\tau}$ & S4 & $10^{-3}{\sim}10^{-2}$ & 3.44\% & 0.059 & 0.050
  & \textcolor{green!60!black}{\checkmark} & 0.051 & 0.043
  & \textcolor{green!60!black}{\checkmark} \\
\midrule
$\mathbf{u}$ & overall & $10^{-5}{\sim}10^{0}$ & 100\% & 0.058 & 0.103 & \textcolor{red}{$\times$}
  & 0.00941 & 0.0056 & \textcolor{green!60!black}{\checkmark} \\
$\mathbf{u}$ & V0 & $10^{-5}{\sim}10^{-4}$ & 28.81\% & 0.072 & 0.179 & \textcolor{red}{$\times$}
  & 0.084 & 0.167 & \textcolor{red}{$\times$} \\
$\mathbf{u}$ & V1 & $10^{-4}{\sim}10^{-3}$ & 38.13\% & 0.073 & 0.086 & \textcolor{red}{$\times$}
  & 0.073 & 0.072 & \textcolor{green!60!black}{\checkmark} \\
$\mathbf{u}$ & V2 & $10^{-3}{\sim}10^{-2}$ & 31.50\% & 0.034 & 0.030
  & \textcolor{green!60!black}{\checkmark} & 0.020 & 0.018
  & \textcolor{green!60!black}{\checkmark} \\
$\mathbf{u}$ & V3 & $10^{-2}{\sim}10^{-1}$ & 1.42\% & 0.013 & 0.00917
  & \textcolor{green!60!black}{\checkmark} & 0.014 & 0.00874
  & \textcolor{green!60!black}{\checkmark} \\
$\mathbf{u}$ & V4 & $10^{-1}{\sim}10^{0}$ & 0.14\% & 0.012 & 0.00683
  & \textcolor{green!60!black}{\checkmark} & 0.00938 & 0.00558
  & \textcolor{green!60!black}{\checkmark} \\
\bottomrule
\end{tabular}
}
\end{table}

\paragraph{SHIFT-Wing (AMR meshes).}
For SHIFT-Wing, the size-stratified results show that $\mathrm{M}^{3}$ preserves the finest volume
scale while improving the larger spatial scales. The finest volume cells (V0), which account for
45.08\% of all cells, improve under both evaluation measures. Although the intermediate V1 and V2
bins show degradation, the magnitude of the change is small relative to the improvements at other
scales. In particular, V1 accounts for 32.37\% of the cells, yet its equal-weight error changes only
from 0.343 to 0.348. V2 accounts for only 7.73\% of the cells, further limiting its contribution to
the overall error. In contrast, all larger volume scales V3--V6 improve under both evaluation
measures. The surface field is similarly stable, with nearly all size bins remaining close to Random
and only a negligible increase at S3 under $\Delta S$ weighting (0.183 to 0.184). Overall,
the $\Delta V$-weighted velocity error decreases substantially from 0.0191 to 0.00378. These observations are consistent with Figures~\ref{fig:partition-sampling-8192}
and~\ref{fig:error-vis-three-datasets}. The $\mathrm{M}^{3}$ support better retains the dual-wing
wake and tip-vortex structures than random sampling from mesh, while the near-wall error maps show
lower inference errors. 

\paragraph{AhmedML (relatively uniform meshes).}
For AhmedML, the effect of $\mathrm{M}^{3}$ is comparatively small at the aggregate level and
remains moderate across individual mesh scales. The finest volume bins V0 and V1 show slight
degradation under both evaluation measures, with an absolute increase of only 0.001. V0 accounts for
5.21\% of the cells, while V1 accounts for 26.06\%, limiting their overall contribution relative to
the improvements at larger scales. In contrast, V2--V4 consistently improve, with V2 alone
accounting for 44.75\% of the cells. As a result, the overall equal-weight velocity error remains
nearly unchanged, from 0.00992 to 0.00989, while the $\Delta V$-weighted error decreases from
0.00256 to 0.00225. The surface field is similarly stable, with improvement at the finest scale,
some degradation at S1, and errors close to Random at the remaining scales.

\paragraph{DrivAerML (non-uniform meshes).}
For DrivAerML, the surface field remains comparatively stable, with improvements in S0, S2, S3, and
S4 and only slight degradation in S1. The volume field, however, exhibits a substantially stronger
scale dependence. The finest volume bin V0, which accounts for 28.81\% of all cells, shows a clear
increase in error, from 0.072 to 0.179 under equal-weight evaluation and from 0.084 to 0.167 under
$\Delta V$ weighting. V1 also degrades under equal-weight evaluation, while remaining nearly
unchanged under $\Delta V$ weighting. In contrast, the larger volume scales V2--V4 consistently
improve under both evaluation measures, with V2 alone accounting for 31.50\% of the cells.
Consequently, the overall equal-weight velocity error increases from 0.058 to 0.103, whereas the
$\Delta V$-weighted error decreases substantially from 0.00941 to 0.0056. To further investigate whether the degradation at the finest volume scale corresponds to a
systematic deterioration of near-wall flow prediction, we visualize the regions with the largest
prediction errors, as shown in Figure~\ref{fig:app-nearwall-1}. Additional multi-view volume
error maps are provided in Figures~\ref{fig:app-vis-27c} and~\ref{fig:app-vis-27d}.

\begin{figure}[ht]
  \centering
  \includegraphics[width=\linewidth]{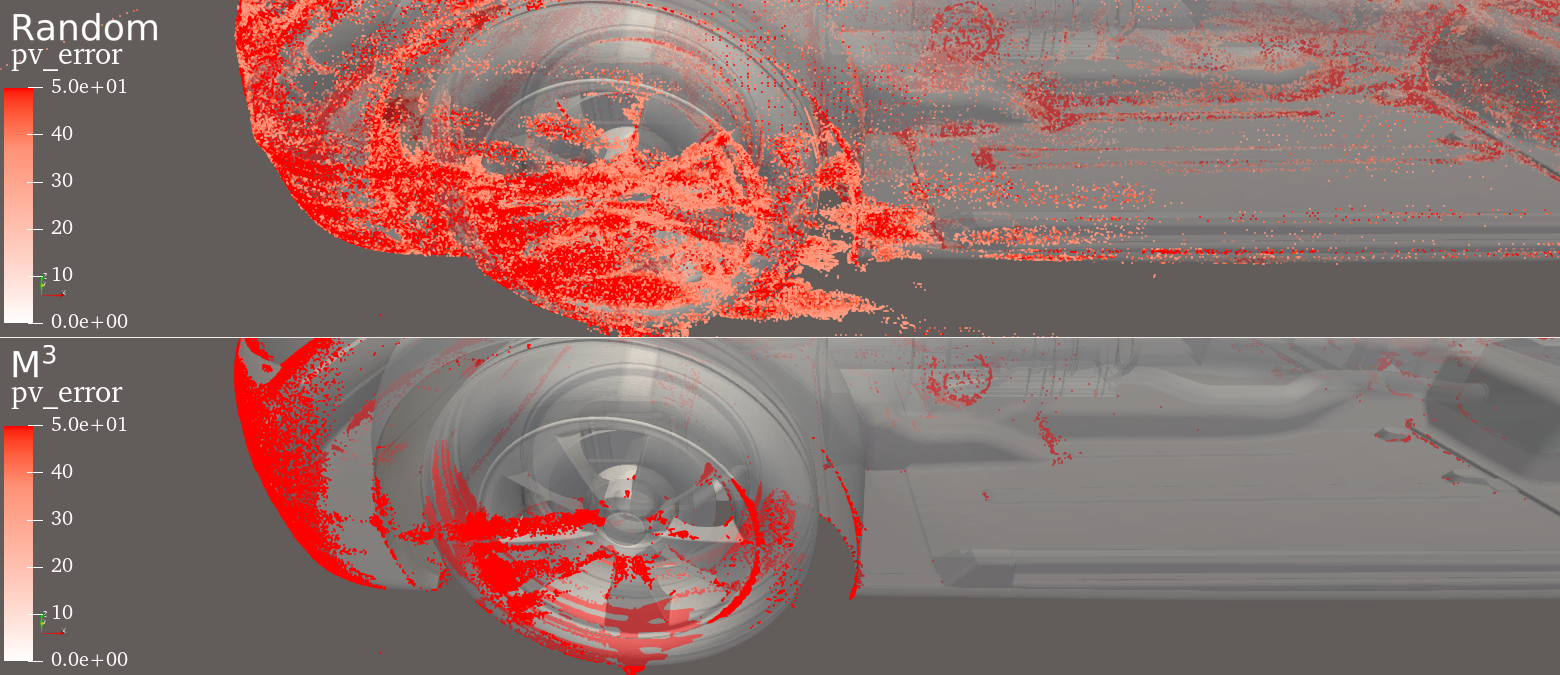}\\[0.6em]
  \includegraphics[width=\linewidth]{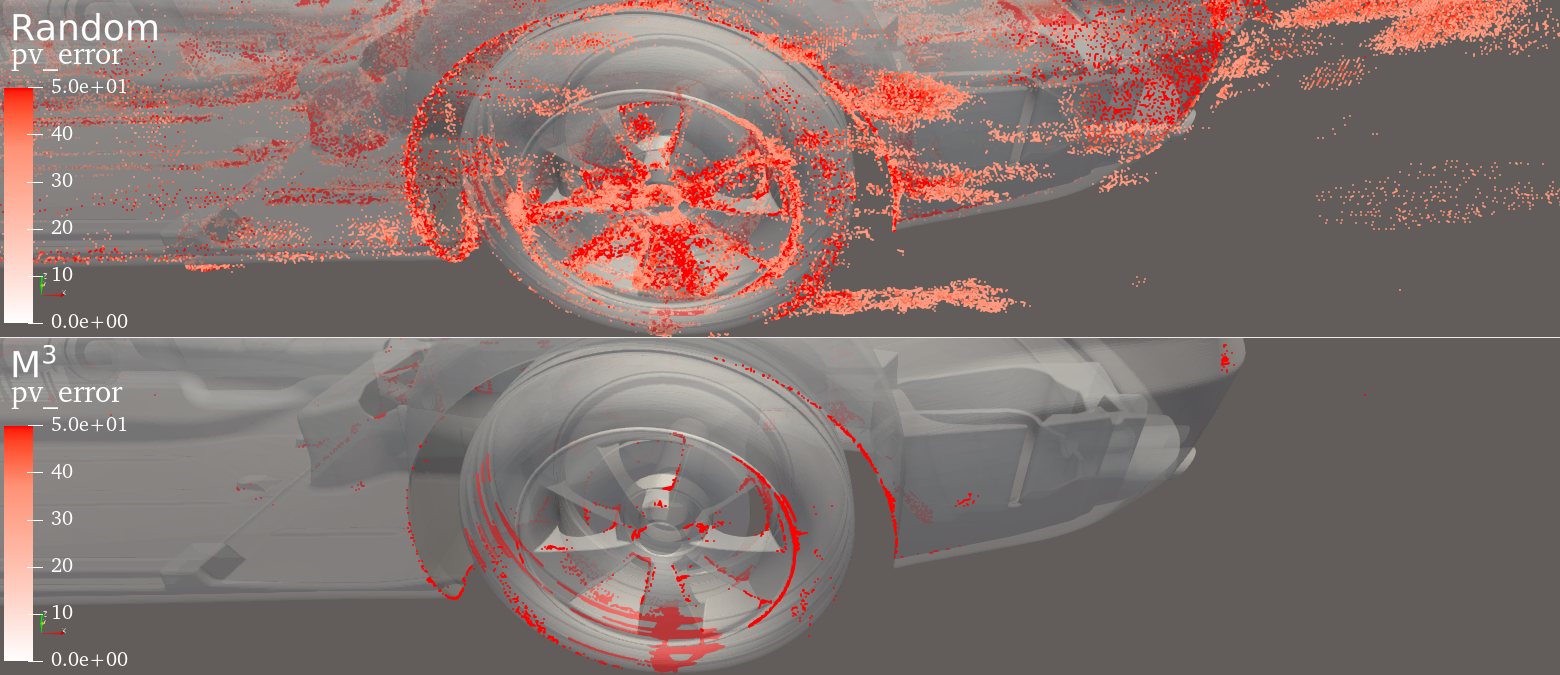}
  \caption{\textbf{Top 5\% error maps of near-wall/complex regions for DrivAerML.}}
  \label{fig:app-nearwall-1}
\end{figure}

Notably, the high-error regions under $\mathrm{M}^{3}$ largely overlap with those under Random,
with similar error levels. In addition, the errors under Random are more spatially scattered. This
matches the point-density pattern in Figures~\ref{fig:app-vis-13}--\ref{fig:app-vis-16}, where the
meshes are extremely dense around the tire--ground contact regions and errors tend to appear in
these high-density places. Random sampling from the mesh therefore allocates a large fraction of
samples to these regions during per-batch optimization, at the expense of other parts of the flow
field. In contrast, $\mathrm{M}^{3}$ distributes samples across physical variation scales throughout
the domain, improving overall field fidelity while retaining localized high-error regions in these
densely meshed areas.

This trade-off explains the opposing aggregate metrics on DrivAerML. Equal-weight evaluation counts
every point equally, so the numerous fine points in the V0 and V1 bins dominate and inflate the
overall volume error of $\mathrm{M}^{3}$. Physics-weighted evaluation downweights the same points by
their small $\Delta V$, so the aggregate instead reflects the improved fidelity across multiple
scales. Meanwhile, these small $\Delta S$ and $\Delta V$ also make those points contribute little to
the global force coefficients. This is reflected in the ground-truth surface force map, where the
tires and tire--ground contact regions are nearly white (Figure~\ref{fig:app-force-contrib}), while
the integrated lift and drag coefficients remain accurate (Figure~\ref{fig:app-vis-20}).

Overall, $\mathrm{M}^{3}$ shifts supervision away from highly over-resolved regions, improving
field fidelity across broader regions and physical scales, although substantial high-error overlap
with Random remains in the densely meshed areas. These results highlight the value of complementary evaluation metrics for
large-scale point-cloud regression, as aggregate metrics alone may obscure the relationship between
local prediction errors and their global physical contributions.


\section{Visualization}
\label{sec:app-visualization}
This section presents additional results and visualizations of the M$^3$ pipeline. We organize the content as follows:
\begin{enumerate}
\item Simulation meshes of three datasets.
\item Partitioning results from Step 1.
\item Grouping results from Step 2.
\item Sampling results from Step 3.
\item Inference results, including force coefficients , full-mesh predictions, and spatial error maps.
\end{enumerate}

\subsection{Simulation Meshes of the Three Datasets}
\label{sec:dataset-discretization}
\begin{figure}[ht]
  \centering
  \includegraphics[width=1\linewidth]{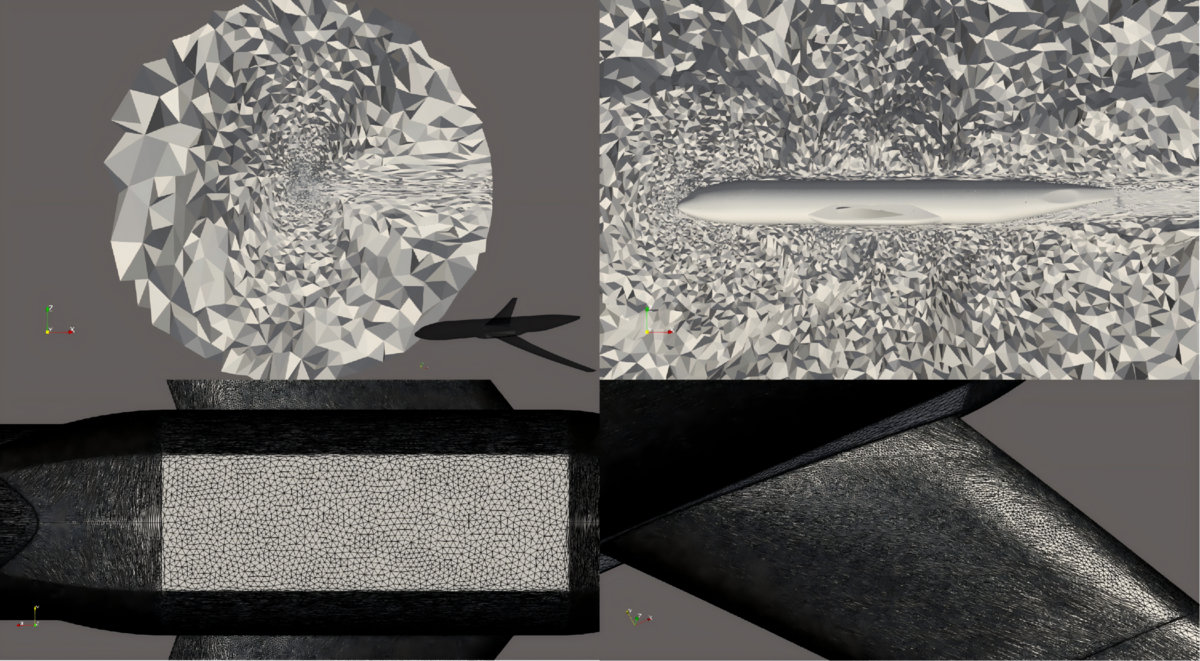}
  \caption{\textbf{Mesh of SHIFT-Wing.} The mesh exhibits strong anisotropy, originating from AMR, where high-gradient regions are densely resolved.}
  \label{fig:app-vis-1}
\end{figure}

\begin{figure}[ht]
  \centering
  \includegraphics[width=1\linewidth]{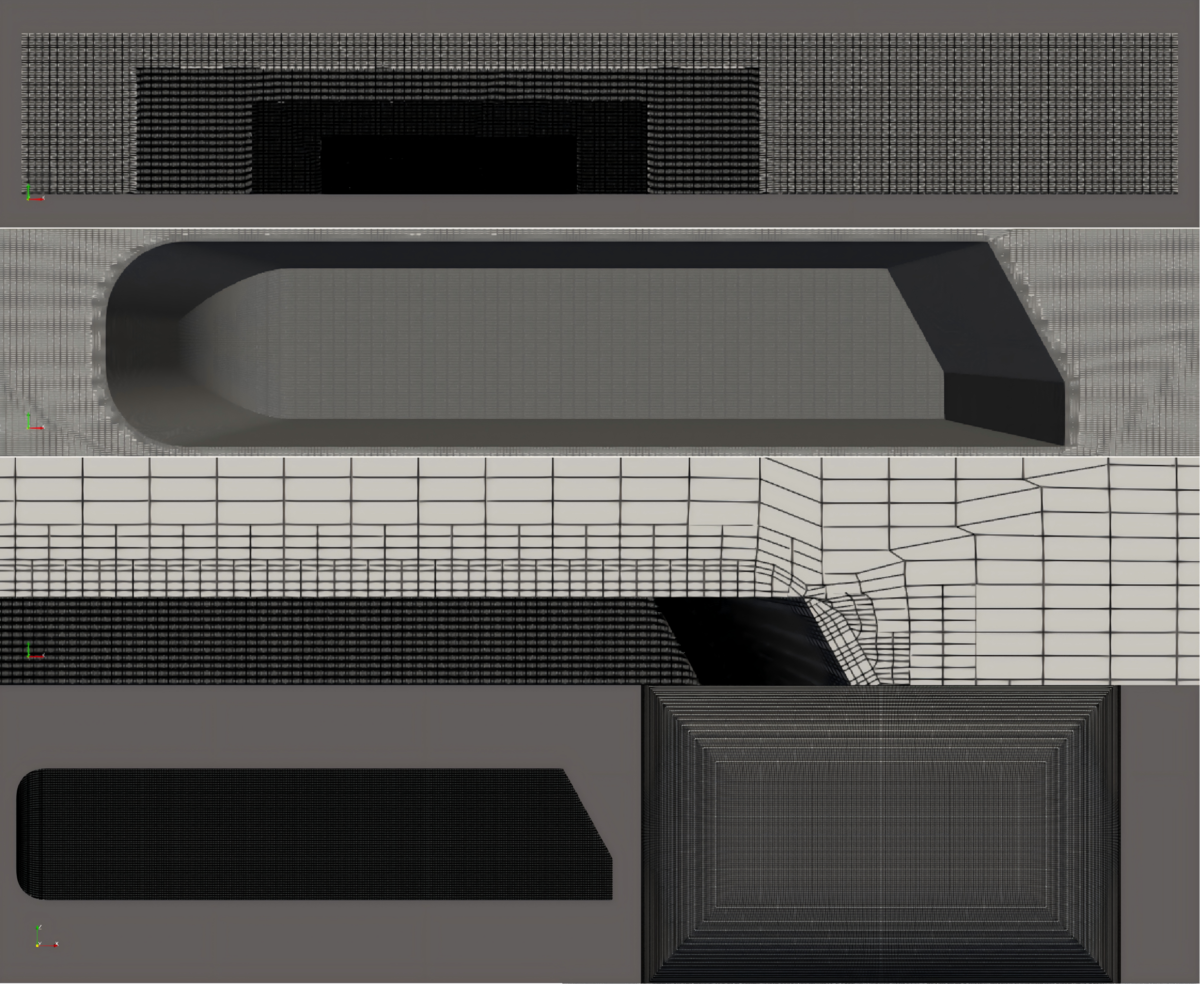}
  \caption{\textbf{Mesh of AhmedML.} The surface mesh is nearly uniform, and variation in volume cell size is relatively mild.}
  \label{fig:app-vis-2}
\end{figure}

\begin{figure}[ht]
  \centering
  \includegraphics[width=1\linewidth]{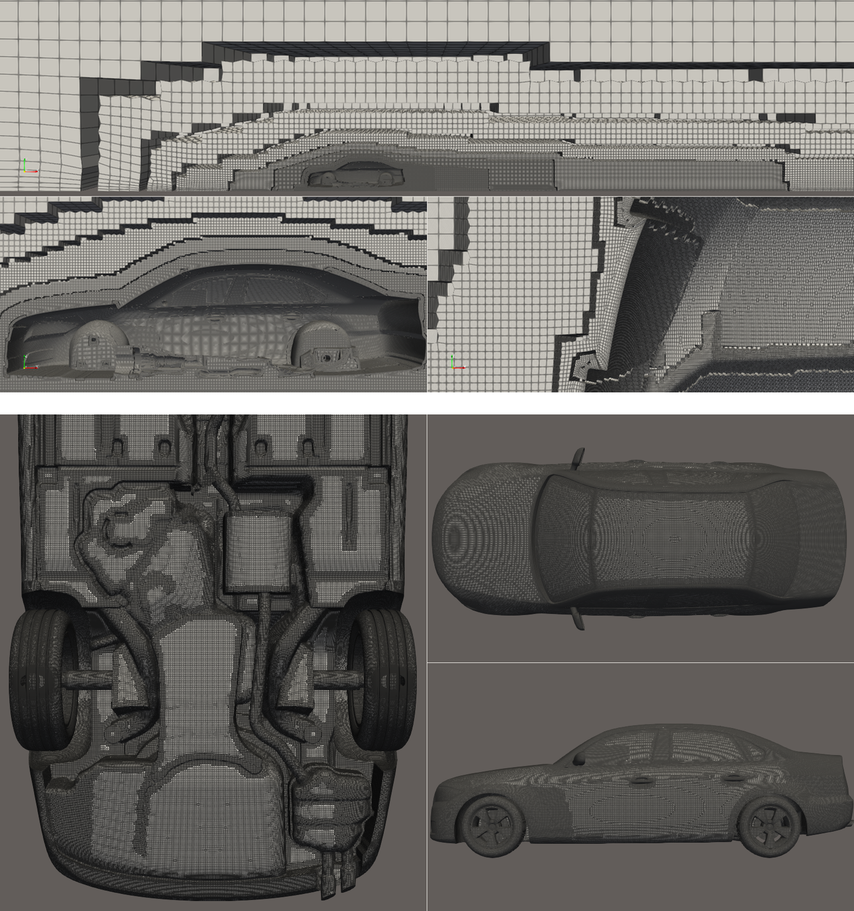}
  \caption{\textbf{Mesh of DrivAerML.} The contrast between coarse and fine cells is highly pronounced, resulting in a strongly non-uniform data distribution.}
  \label{fig:app-vis-3}
\end{figure}

\clearpage
\subsection{Geometry Partition Results for Surface and Volume Fields}
\label{sec:geom-partition}
All subsequent sampling is based on these cell supports. To better illustrate multi-scale layering,
rendering transparency may be adjusted.

\begin{figure}[ht]
  \centering
  \includegraphics[width=\linewidth]{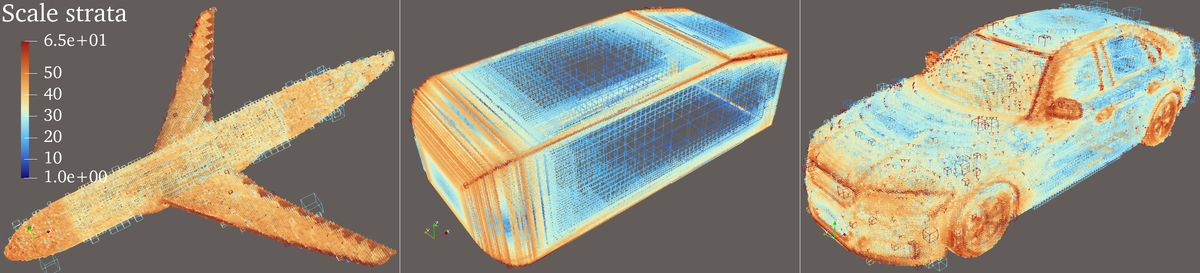}
  \caption{\textbf{Surface cells for the three datasets using a fixed scale bin $K{=}64$.} The final bin (65) contains cells with only one point.}
  \label{fig:app-vis-4}
\end{figure}

\begin{figure}[ht]
  \centering
  \includegraphics[width=\linewidth]{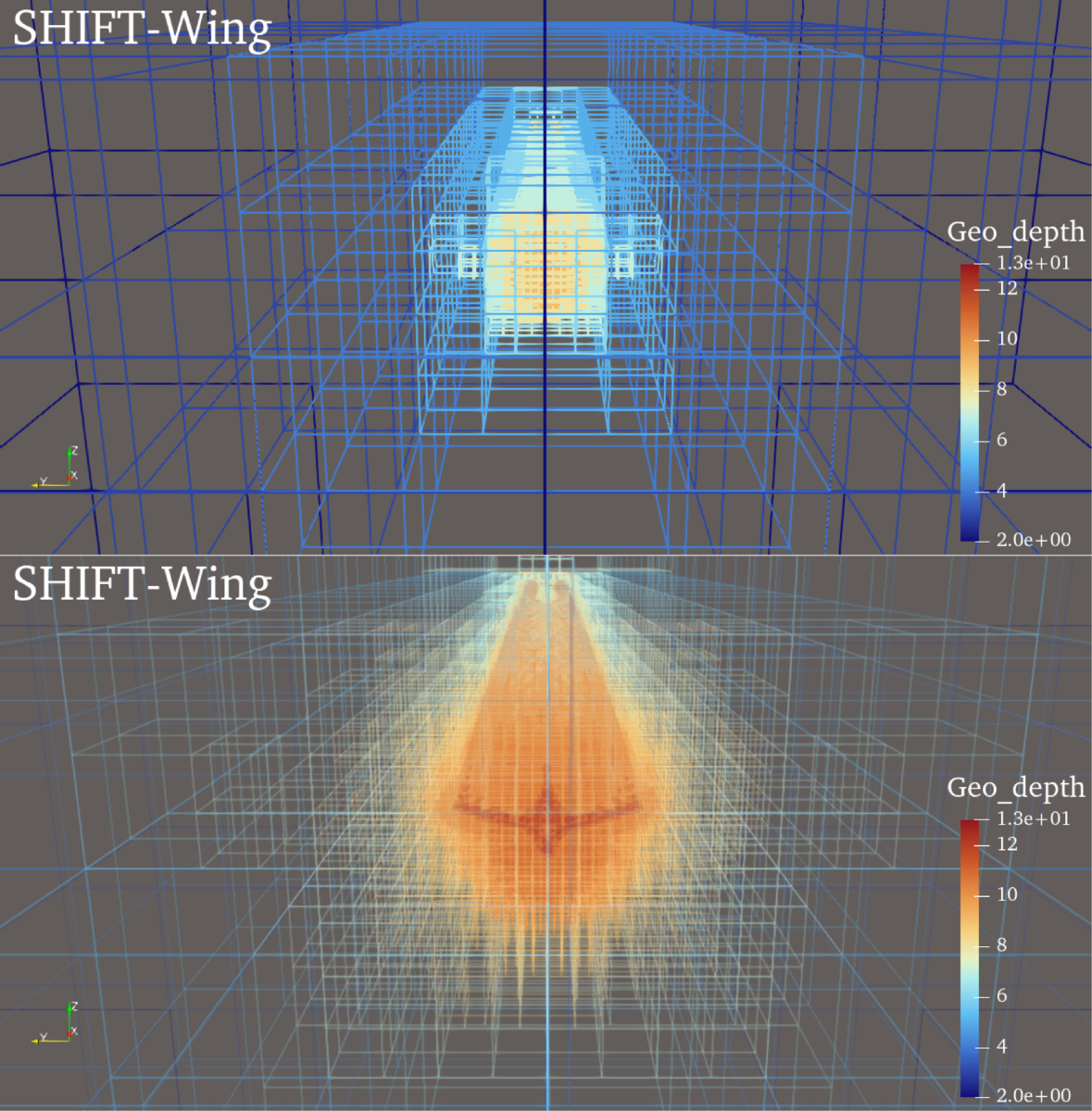}
  \caption{\textbf{Volume cells for SHIFT-Wing, colored by geometric size.}}
  \label{fig:app-vis-6}
\end{figure}

\begin{figure}[ht]
  \centering
  \includegraphics[width=\linewidth]{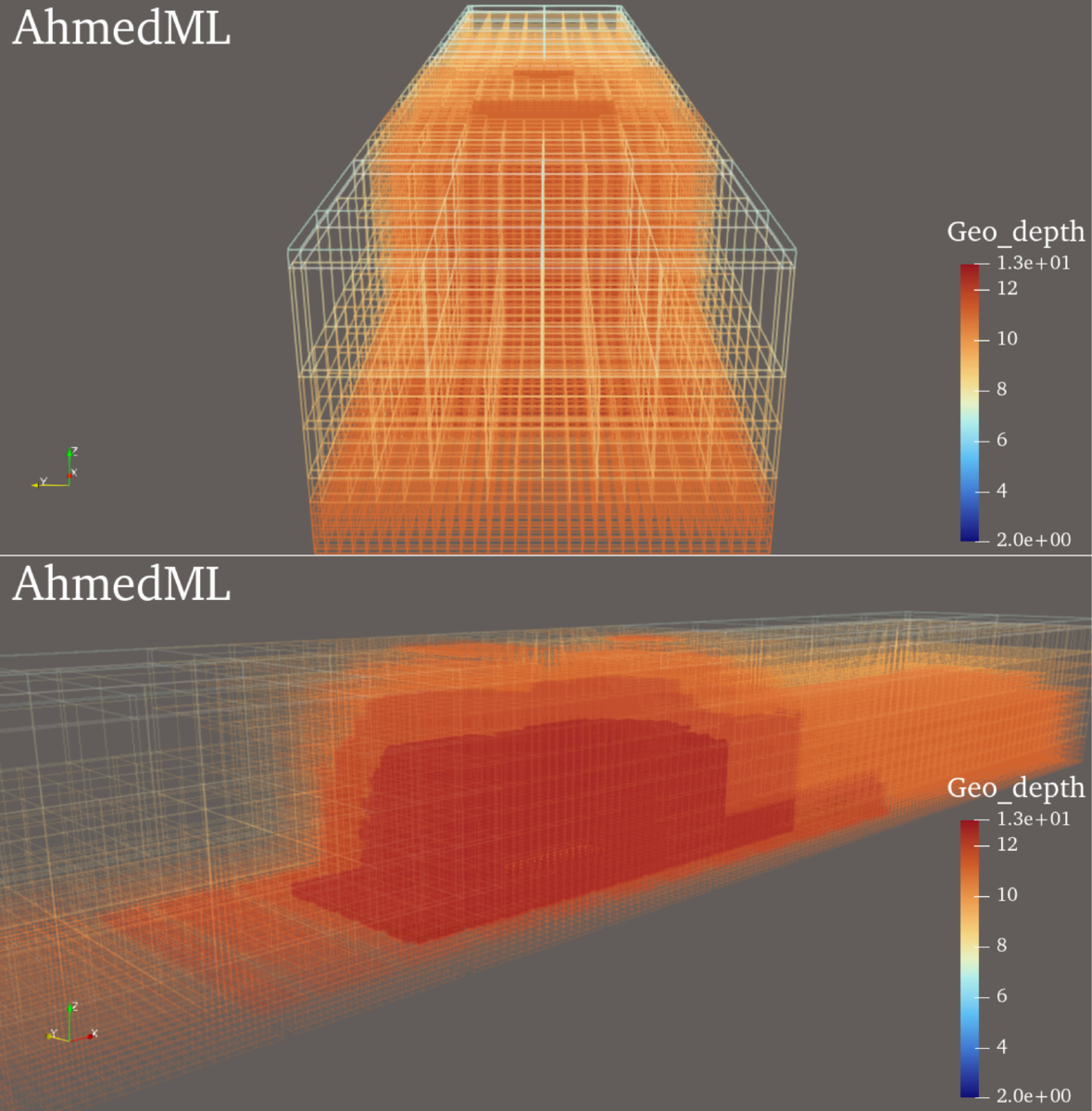}
  \caption{\textbf{Volume cells for AhmedML, colored by geometric size.}}
  \label{fig:app-vis-7}
\end{figure}

\begin{figure}[ht]
  \centering
  \includegraphics[width=\linewidth]{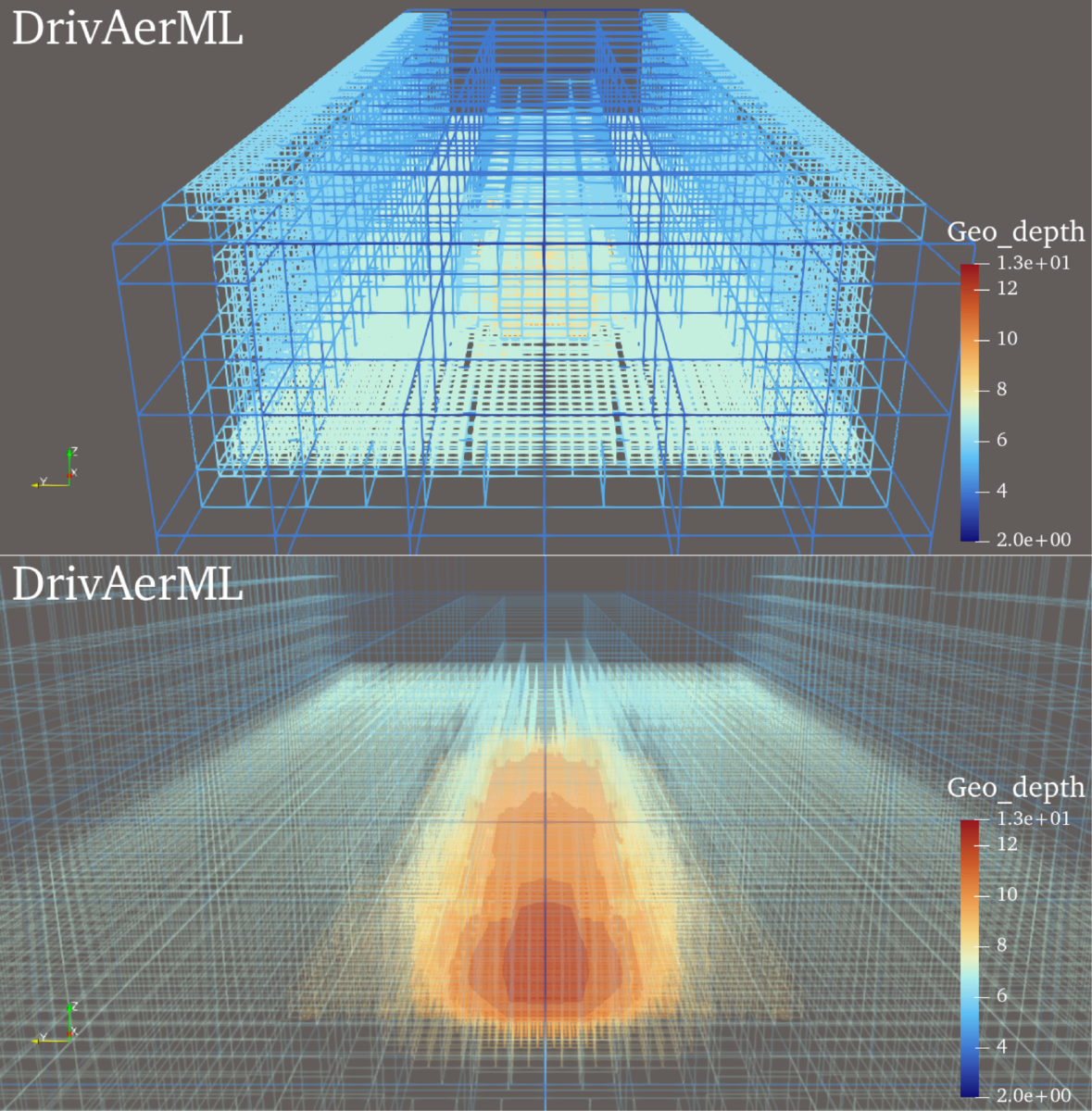}
  \caption{\textbf{Volume cells for DrivAerML, colored by geometric size.}}
  \label{fig:app-vis-8}
\end{figure}

\begin{figure}[ht]
  \centering
  \includegraphics[width=\linewidth]{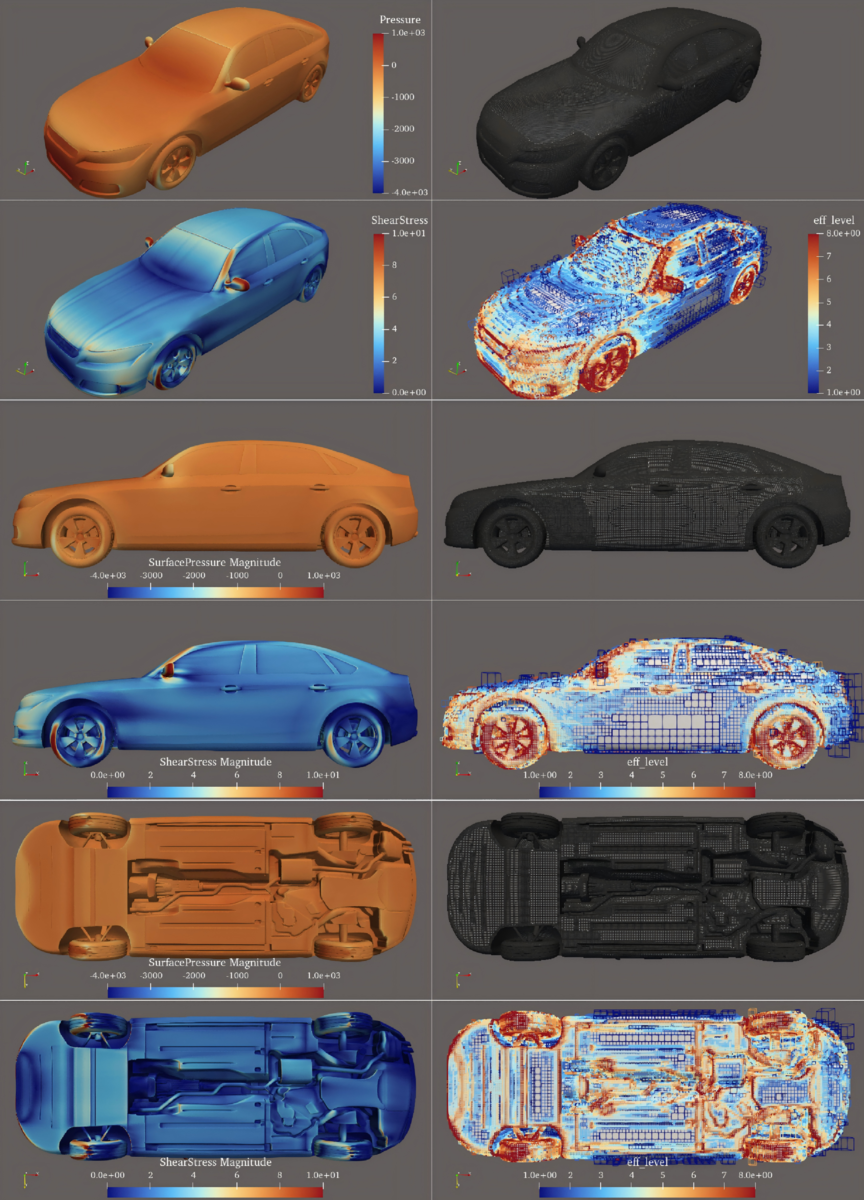}
  \caption{\textbf{Surface cell visualization for DrivAerML with scale bin $K{=}8$ for improved clarity.} Cells are adaptively refined in regions with large physical variation, resulting in a finer partition where needed. The resulting cells jointly form a coarse-grained representation of the full mesh.}
  \label{fig:app-vis-5}
\end{figure}
\clearpage
\subsection{Step 2 Scale Stratification Results}
\label{sec:step2-scale-stratification}
These figures show that cells are reallocated from geometry-driven distributions to a more balanced,
scale-aware distribution. In Step 3, the sampling budget is assigned based on stratified cell
supports, which provide more uniform coverage than geometry-based depth distributions.

\begin{figure}[ht]
  \centering
  \includegraphics[width=0.98\linewidth]{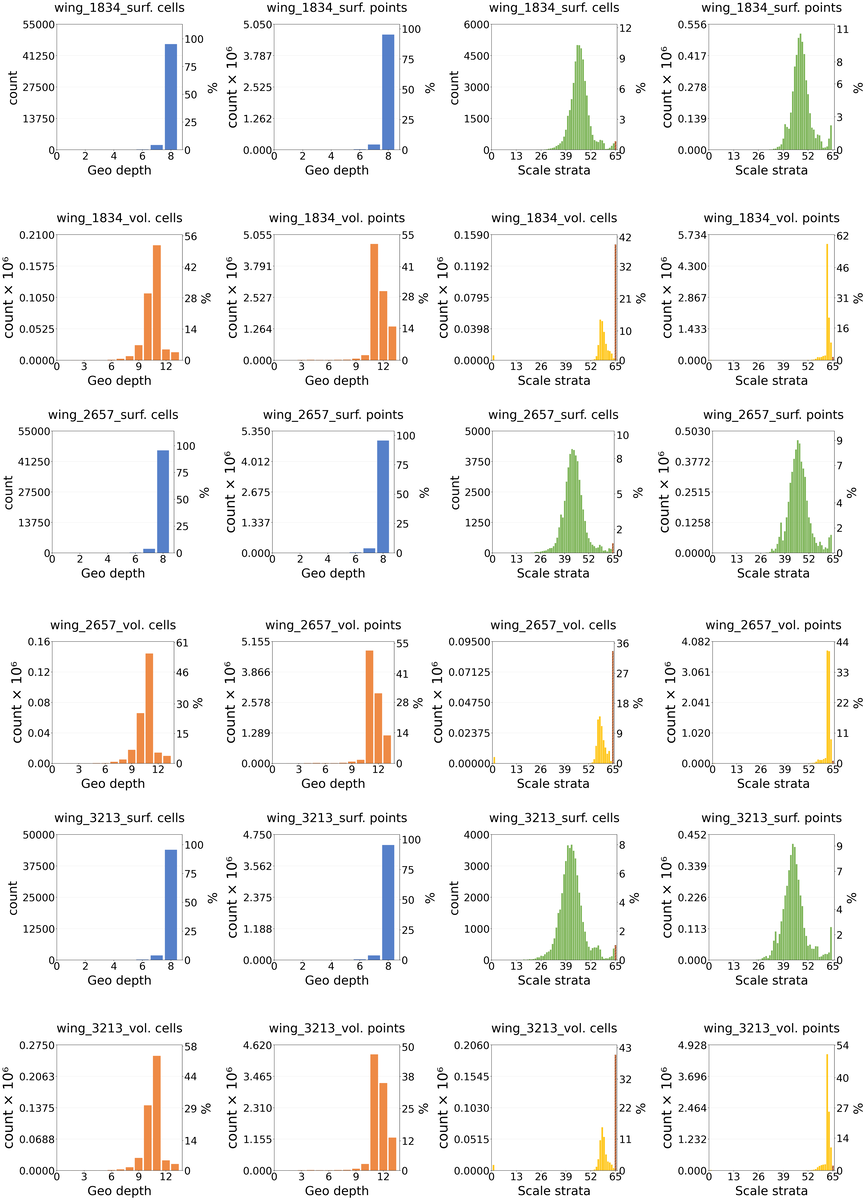}
  \caption{\textbf{Statistical summaries for SHIFT-Wing.} Left axis: counts of cells and points. Right axis: percentage of the total.}
  \label{fig:app-vis-9}
\end{figure}

\begin{figure}[ht]
  \centering
  \includegraphics[width=\linewidth]{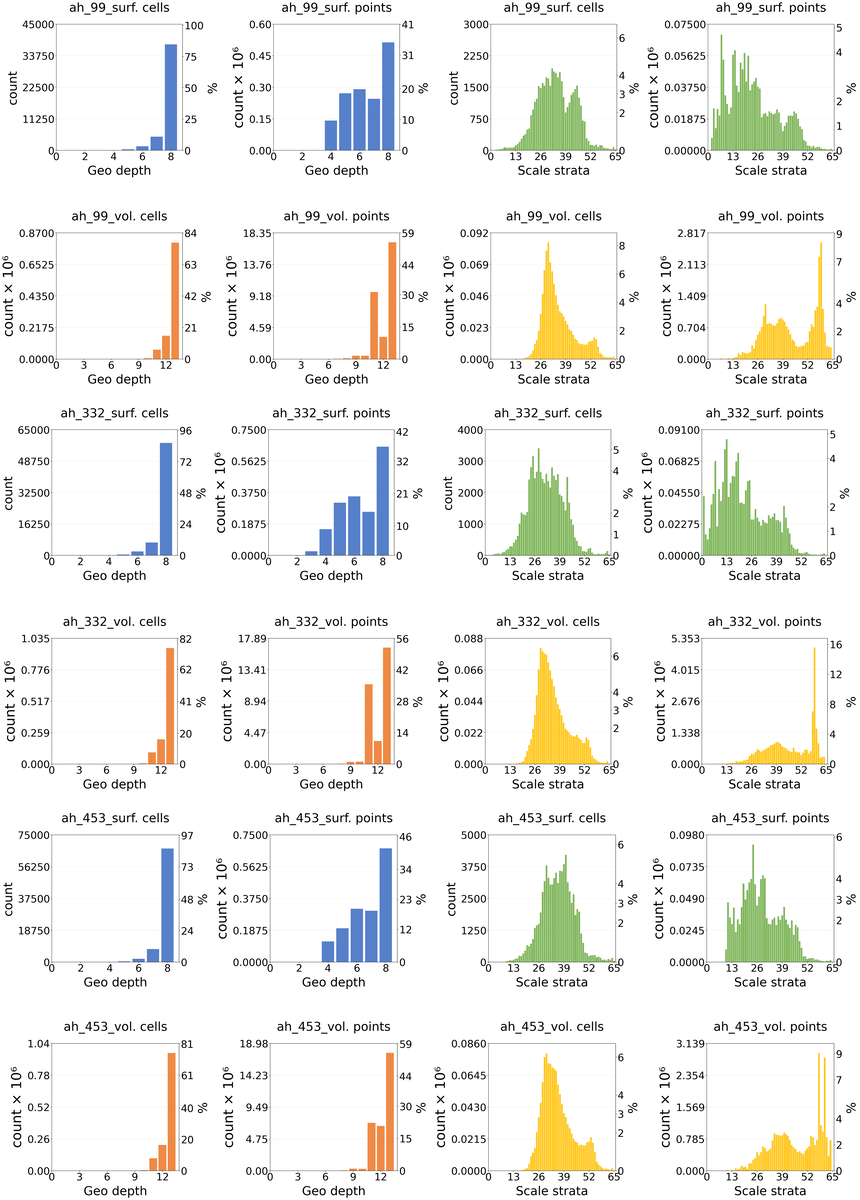}
  \caption{\textbf{Statistical summaries for AhmedML.} Left axis: counts of cells and points. Right axis: percentage of the total.}
  \label{fig:app-vis-10}
\end{figure}

\begin{figure}[ht]
  \centering
  \includegraphics[width=\linewidth]{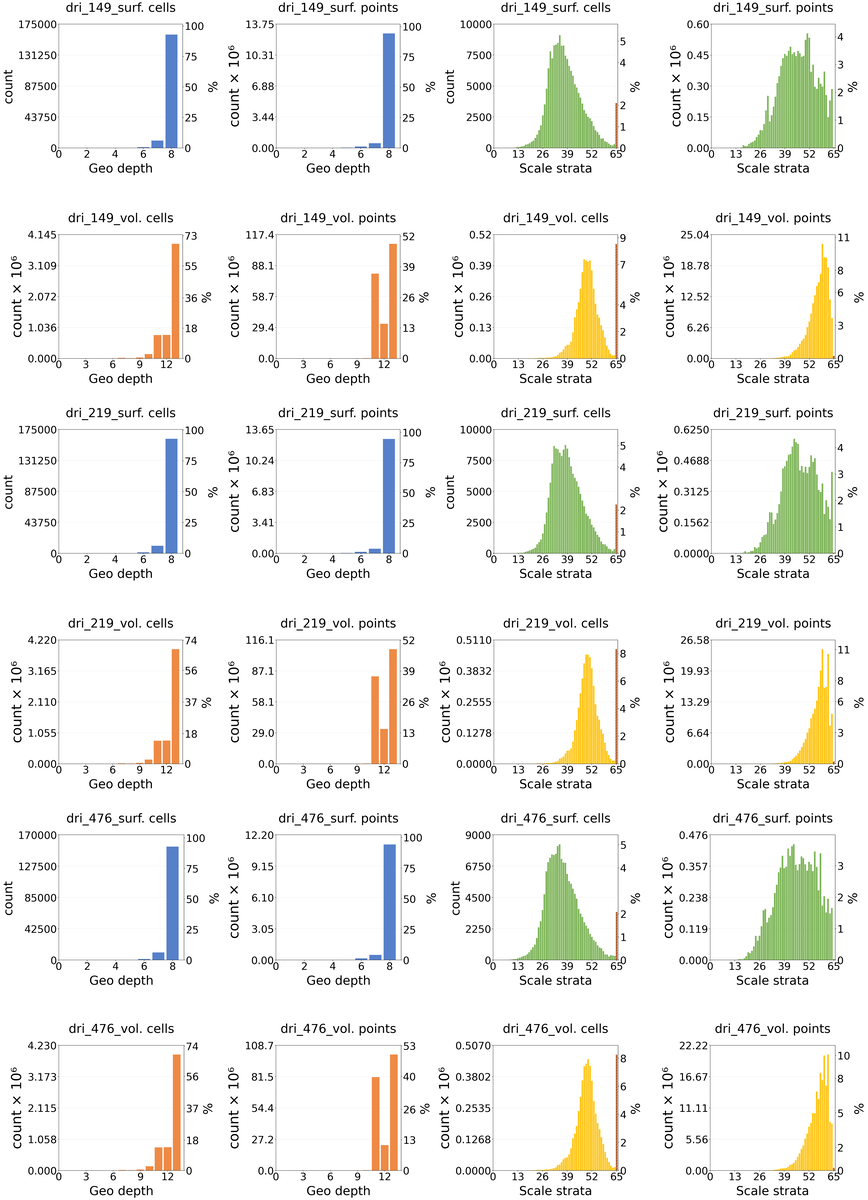}
  \caption{\textbf{Statistical summaries for DrivAerML.} Left axis: counts of cells and points. Right axis: percentage of the total.}
  \label{fig:app-vis-11}
\end{figure}

\begin{figure}[ht]
  \centering
  \includegraphics[width=\linewidth]{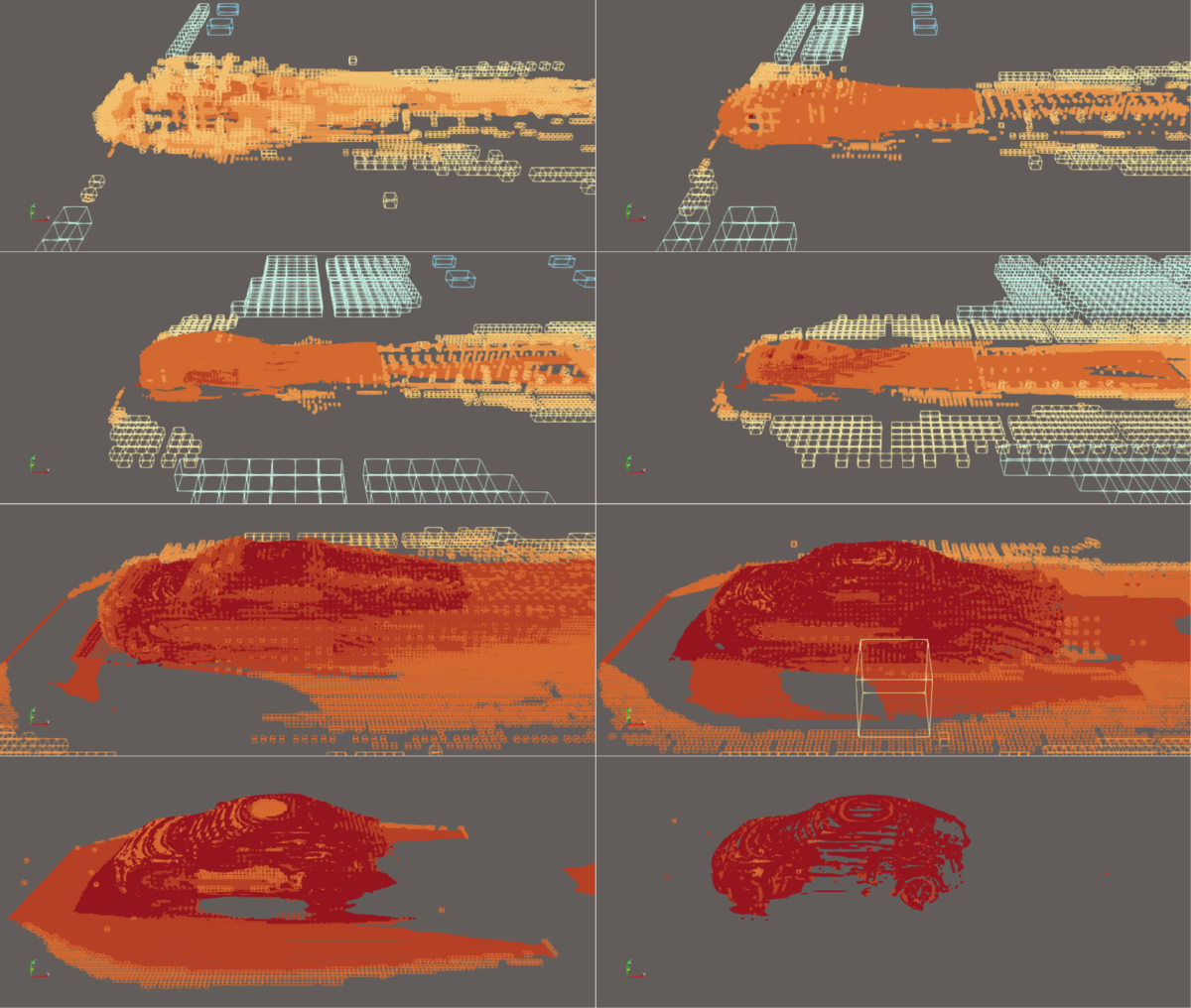}
  \caption{\textbf{Scale-stratified cell visualization for DrivAerML.} Eight consecutive scales are selected from the Step 2 multi-scale grouping results. Each scale groups cells whose points share similar physical variations, enabling geometrically different cells to be assigned to the same physical scale.}
  \label{fig:app-vis-12}
\end{figure}

\clearpage
\subsection{Step 3 Sampling Results for DrivAerML}
\label{sec:drivaerml-step3}
MP denotes M$^3$-preprocessed data, and $x\%$ indicates the downsampling ratio. M$^3$ produces a
more balanced spatial distribution than uniform random sampling, which over-samples dense boundary
regions and leads to uneven coverage.

\begin{figure}[ht]
  \centering
  \includegraphics[width=\linewidth]{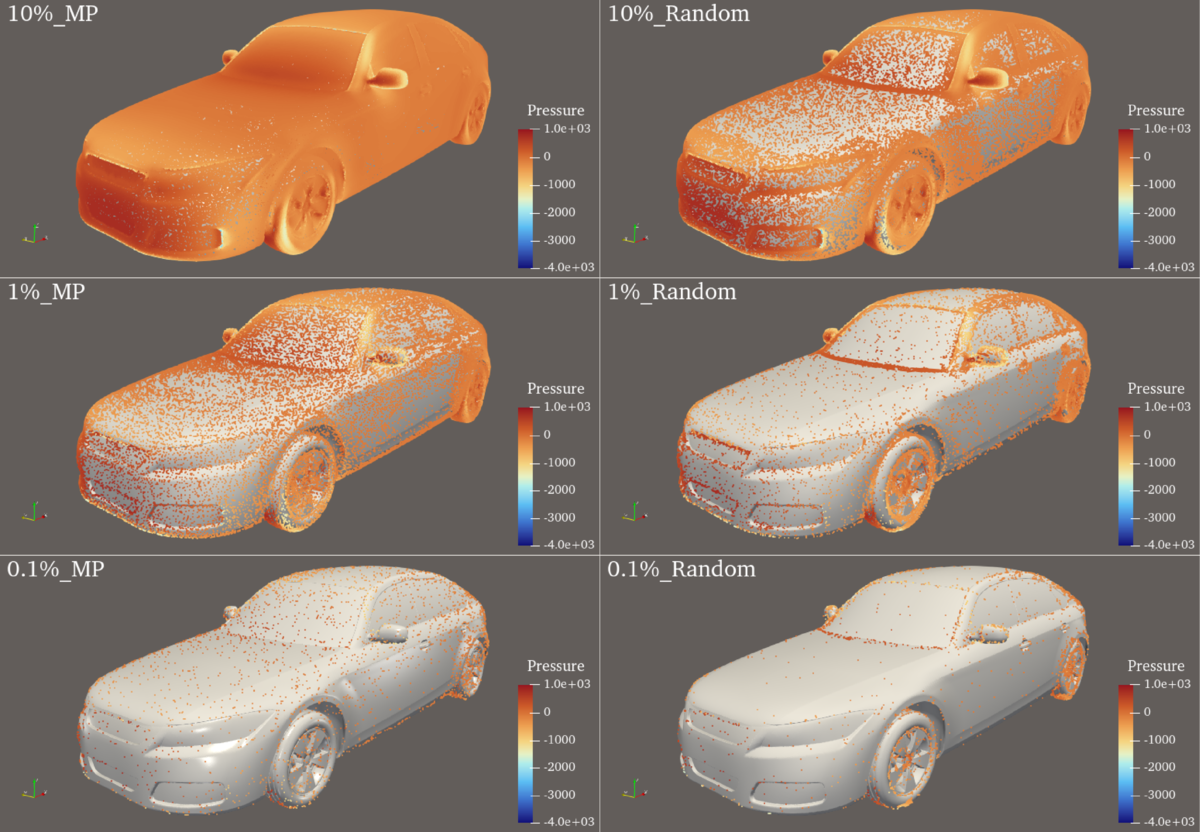}
  \caption{\textbf{Sampling comparison on surface pressure (1).}}
  \label{fig:app-vis-13}
\end{figure}

\begin{figure}[ht]
  \centering
  \includegraphics[width=\linewidth]{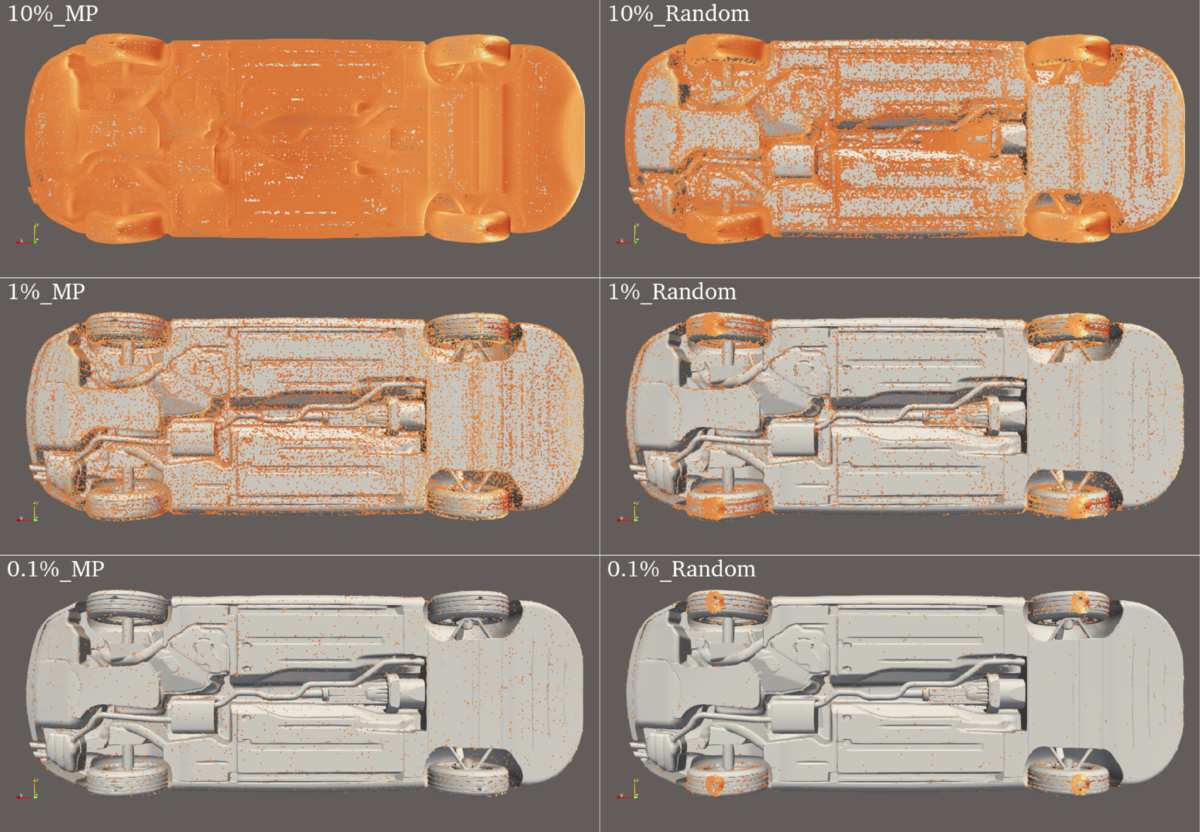}
  \caption{\textbf{Sampling comparison on surface pressure (2).}}
  \label{fig:app-vis-14}
\end{figure}

\begin{figure}[ht]
  \centering
  \includegraphics[width=\linewidth]{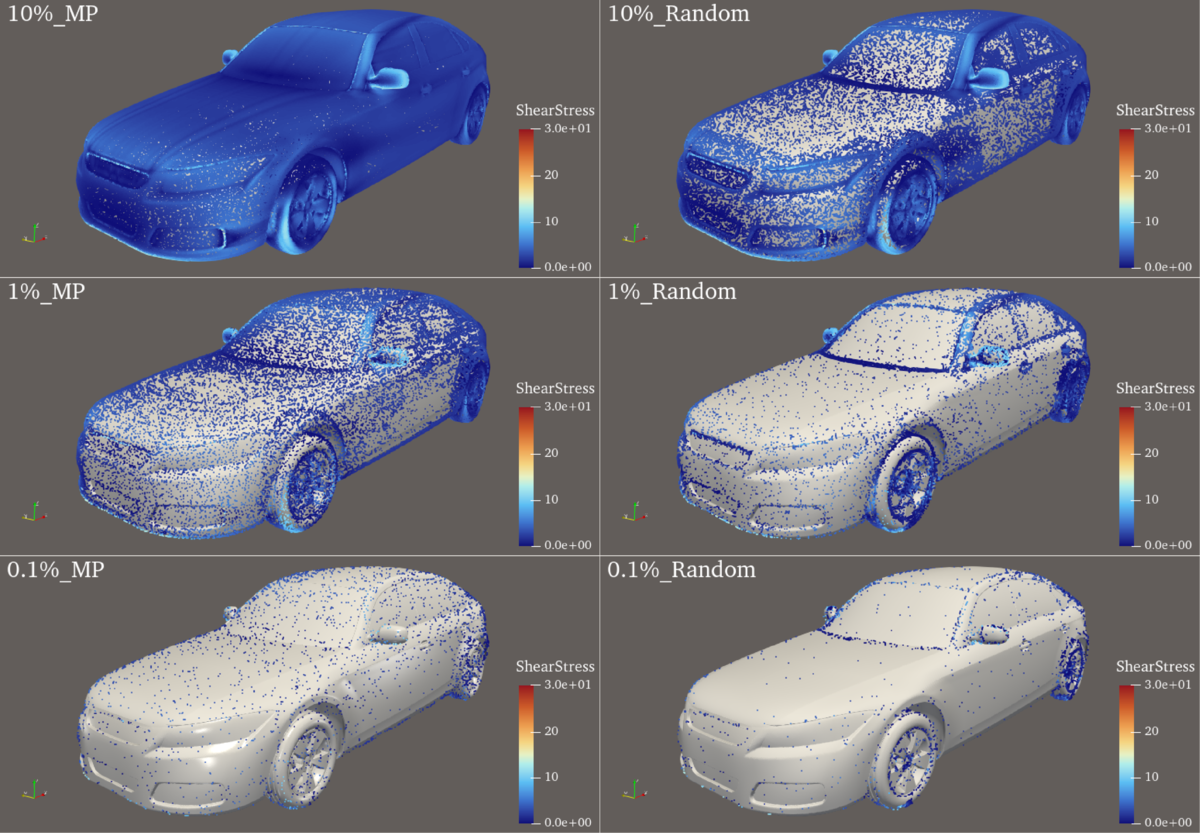}
  \caption{\textbf{Sampling comparison on surface shear stress (1).}}
  \label{fig:app-vis-15}
\end{figure}

\begin{figure}[ht]
  \centering
  \includegraphics[width=\linewidth]{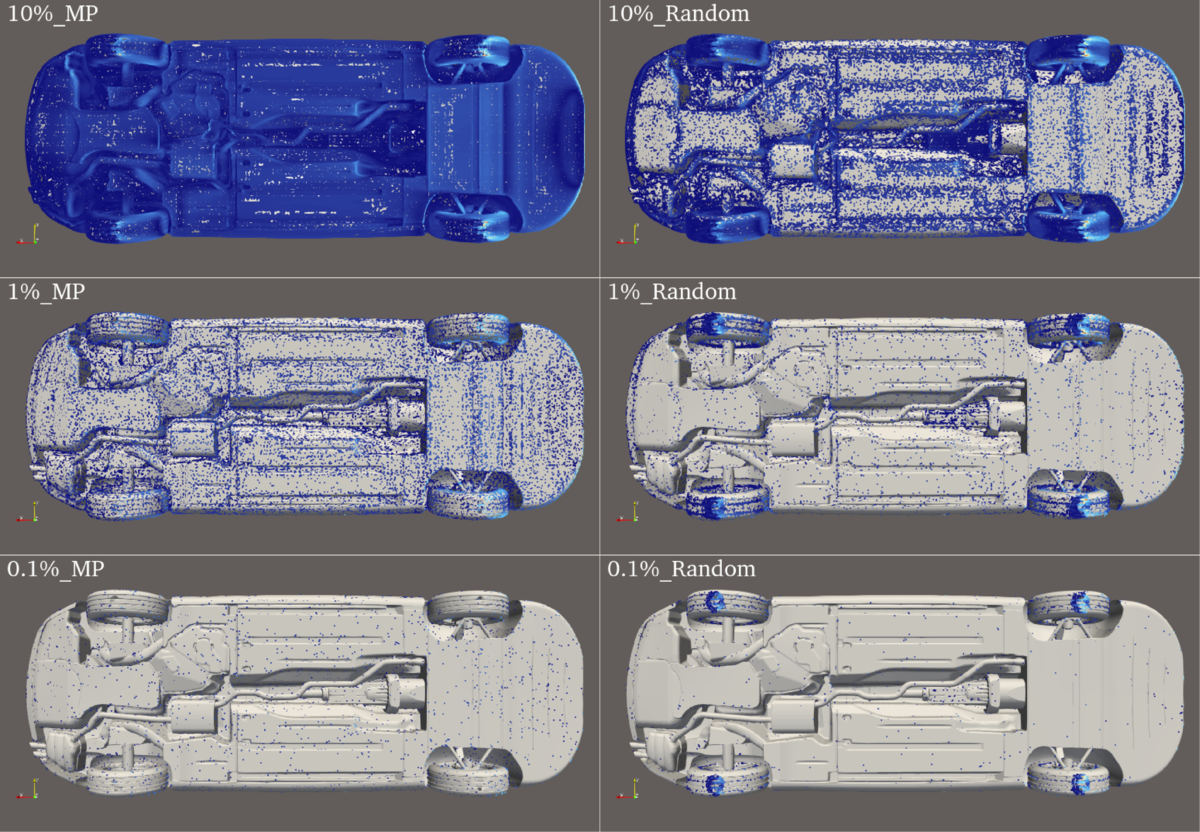}
  \caption{\textbf{Sampling comparison on surface shear stress (2).} Point concentration is observed in the tire region under random sampling due to locally dense mesh discretization.}
  \label{fig:app-vis-16}
\end{figure}

\begin{figure}[ht]
  \centering
  \includegraphics[width=\linewidth]{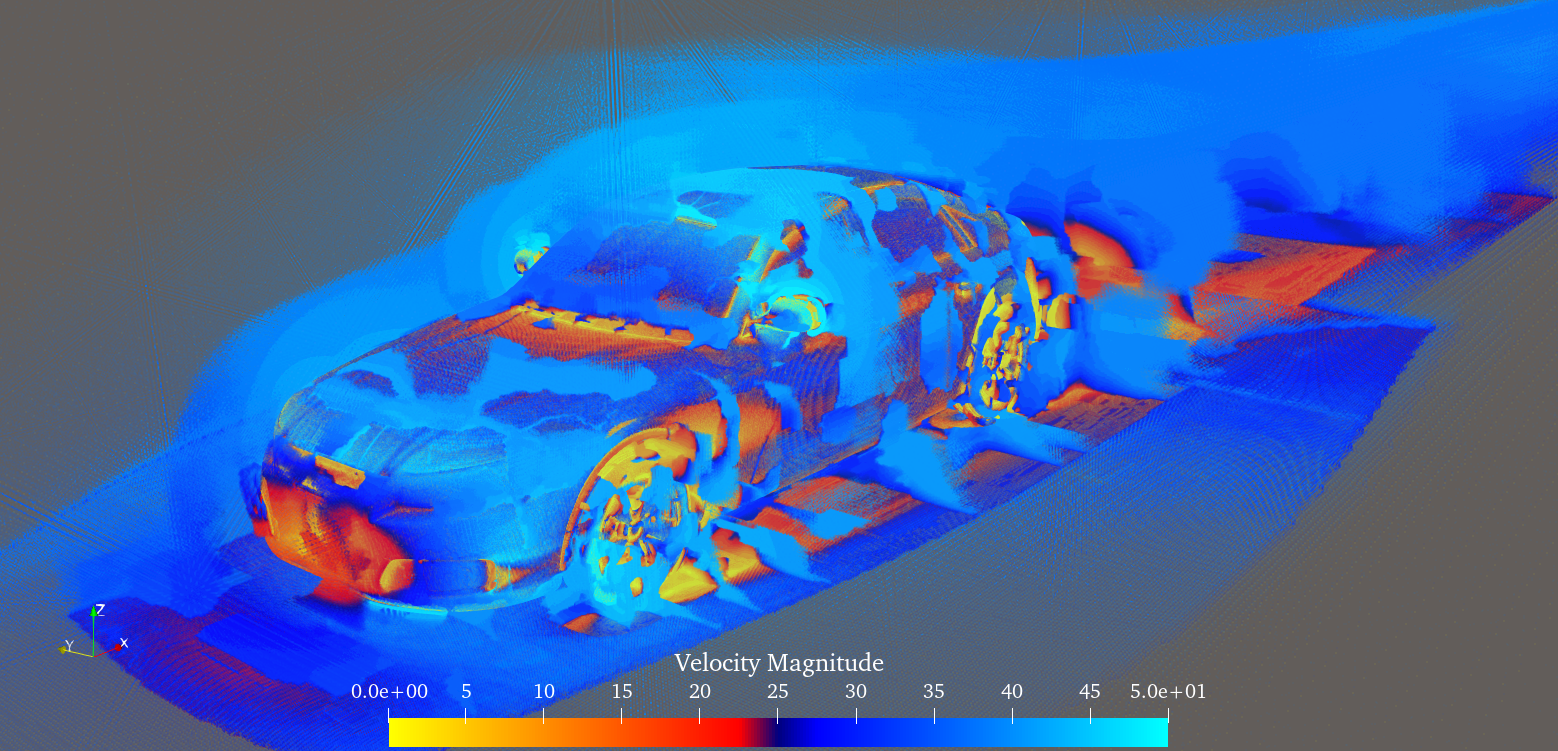}
  \caption{\textbf{Ground truth with 160M volume points.} Minor visual differences in the sampling results are attributable to rendering.}
  \label{fig:app-vis-17}
\end{figure}

\begin{figure}[ht]
  \centering
  \includegraphics[width=\linewidth]{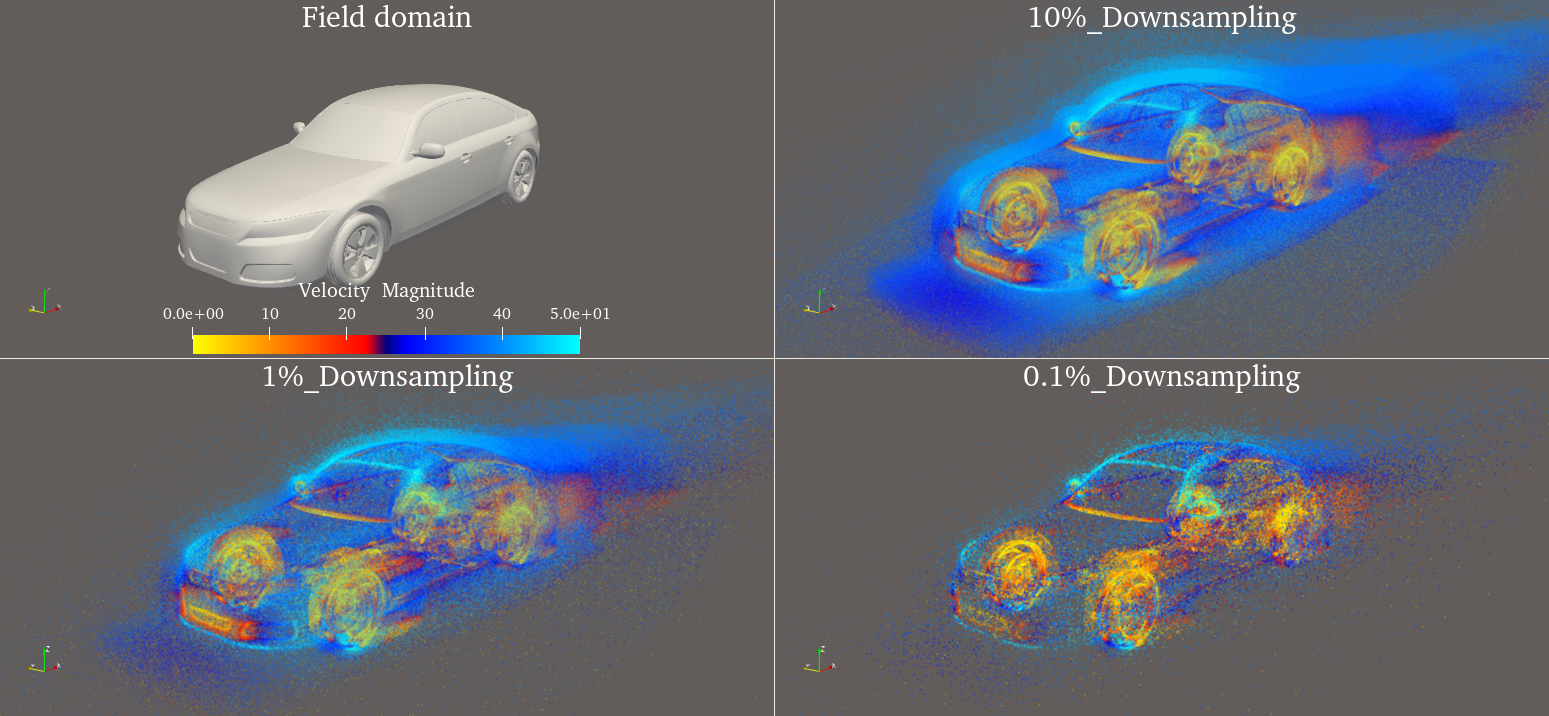}
  \caption{\textbf{Uniform random sampling results in the volume field.} The distribution follows the mesh density, leading to concentration near boundary layers (region highlighted in yellow).}
  \label{fig:app-vis-18}
\end{figure}

\begin{figure}[ht]
  \centering
  \includegraphics[width=\linewidth]{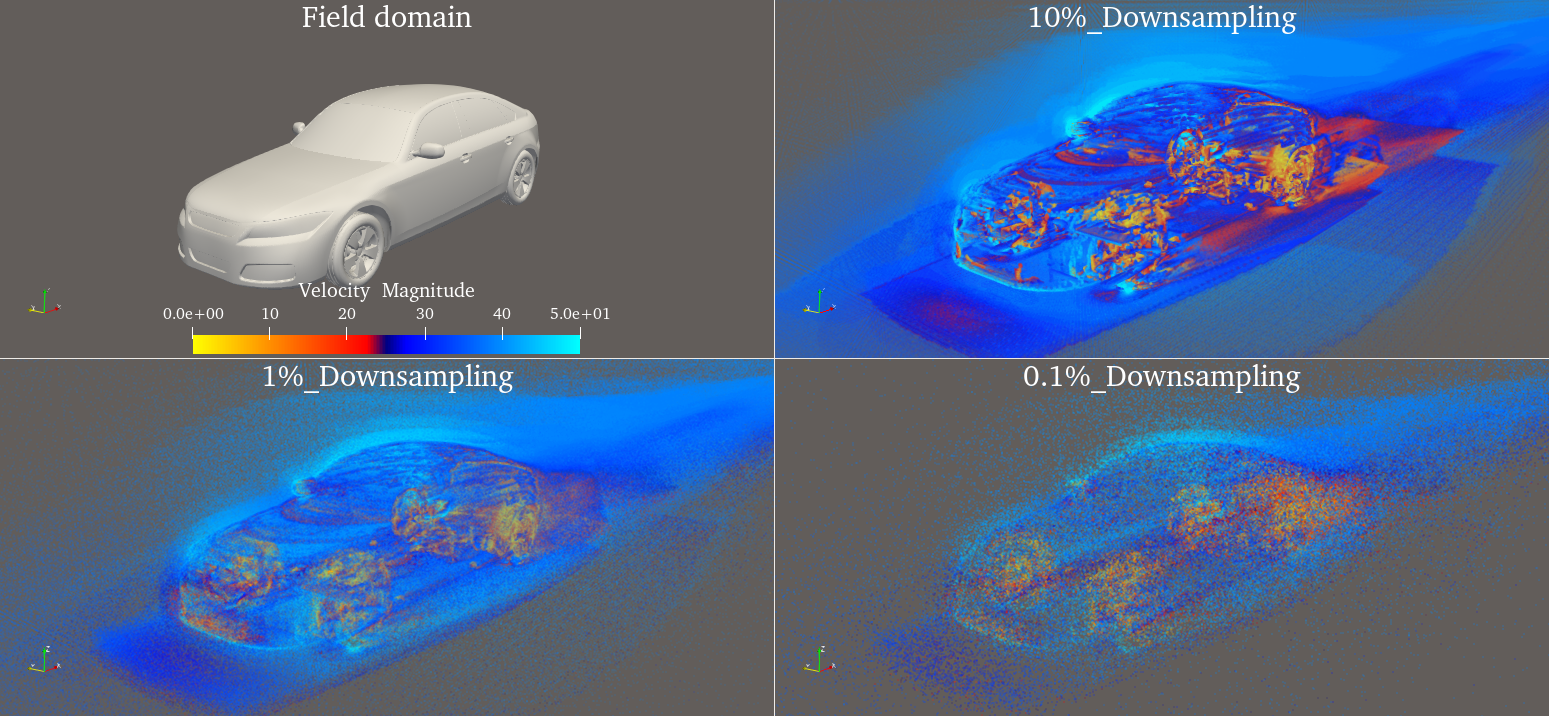}
  \caption{\textbf{M$^3$ sampling results in the volume field.} The distribution is more uniform across the domain while preserving multi-scale volumetric features, more closely matching the full-mesh ground-truth distribution.}
  \label{fig:app-vis-19}
\end{figure}

\clearpage
\subsection{Full-Mesh Inference Results}
\label{sec:app-full-resolution}

\begin{figure}[ht]
  \centering
  \includegraphics[width=\linewidth]{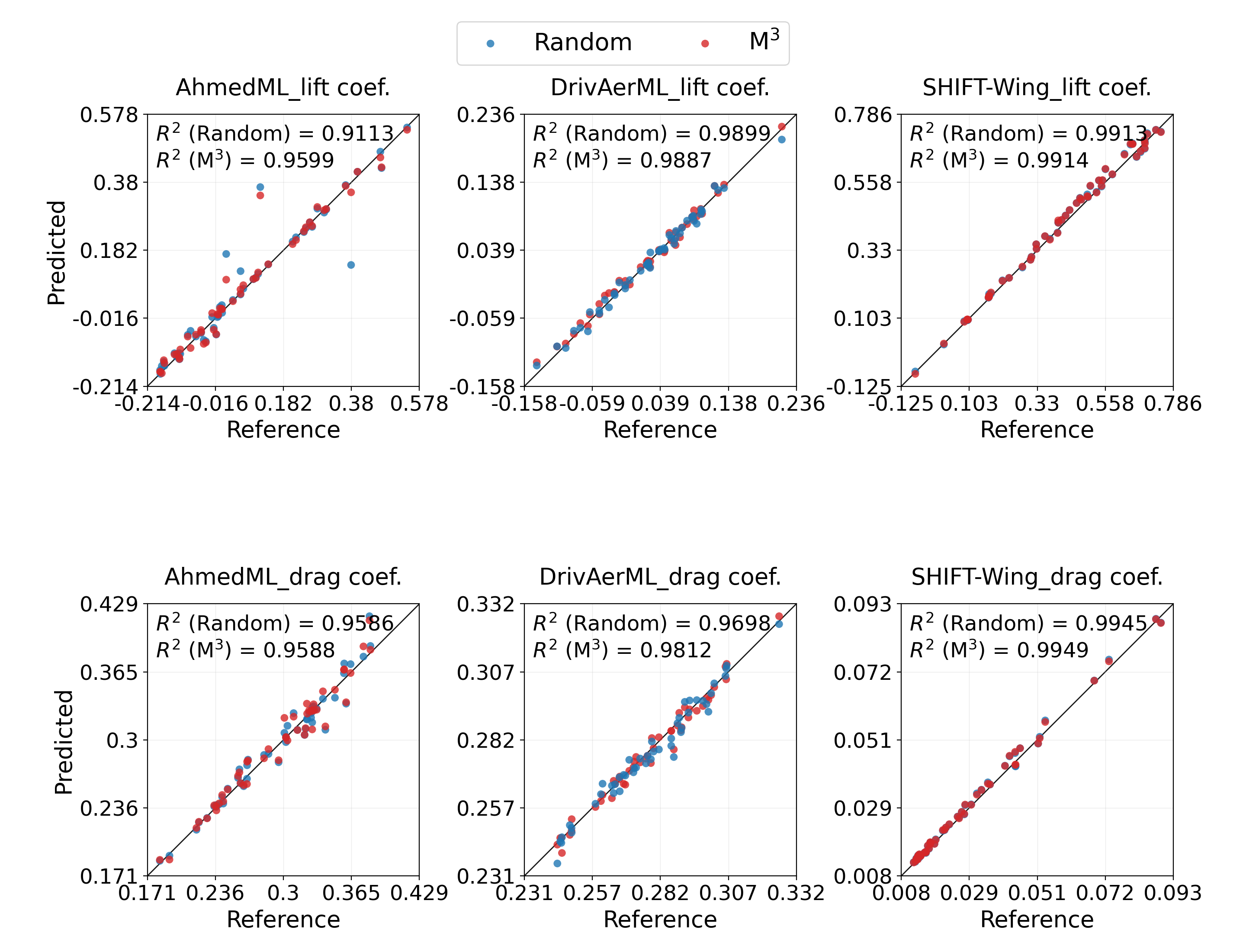}
  \caption{\textbf{Comparison of lift and drag predictions between randomly trained and M$^3$-trained models across held-out cases} (50 for SHIFT-Wing and AhmedML, 49 for DrivAerML). The M$^3$-trained models generally improve lift and drag coefficient predictions.}
  \label{fig:app-vis-20}
\end{figure}

Figure~\ref{fig:app-force-contrib} locates these integrated coefficients on the surface of DrivAerML run 278.
For each face the force is $\mathbf{F}=(p\mathbf{n}-\boldsymbol{\tau})\Delta S$, where $p$ is the
pressure, $\boldsymbol{\tau}$ the wall shear, $\mathbf{n}$ the outward normal, and $\Delta S$ the
face area.
The contributions $dF_{\mathrm{drag}}$ and $dF_{\mathrm{lift}}$ are the projections of $\mathbf{F}$
onto the streamwise and vertical directions, and summing them over the surface recovers the forces
behind the drag and lift coefficients.
The color is the absolute contribution $|dF|$, so a white face adds little to the coefficient.
The tires and the tire--ground contacts stay nearly white because their areas are small, which is
why the local field errors concentrated there barely affect the drag and lift coefficients, whereas
the hood, the roof, and the underbody carry most of the load.

\begin{figure}[ht]
  \centering
  \includegraphics[width=\linewidth]{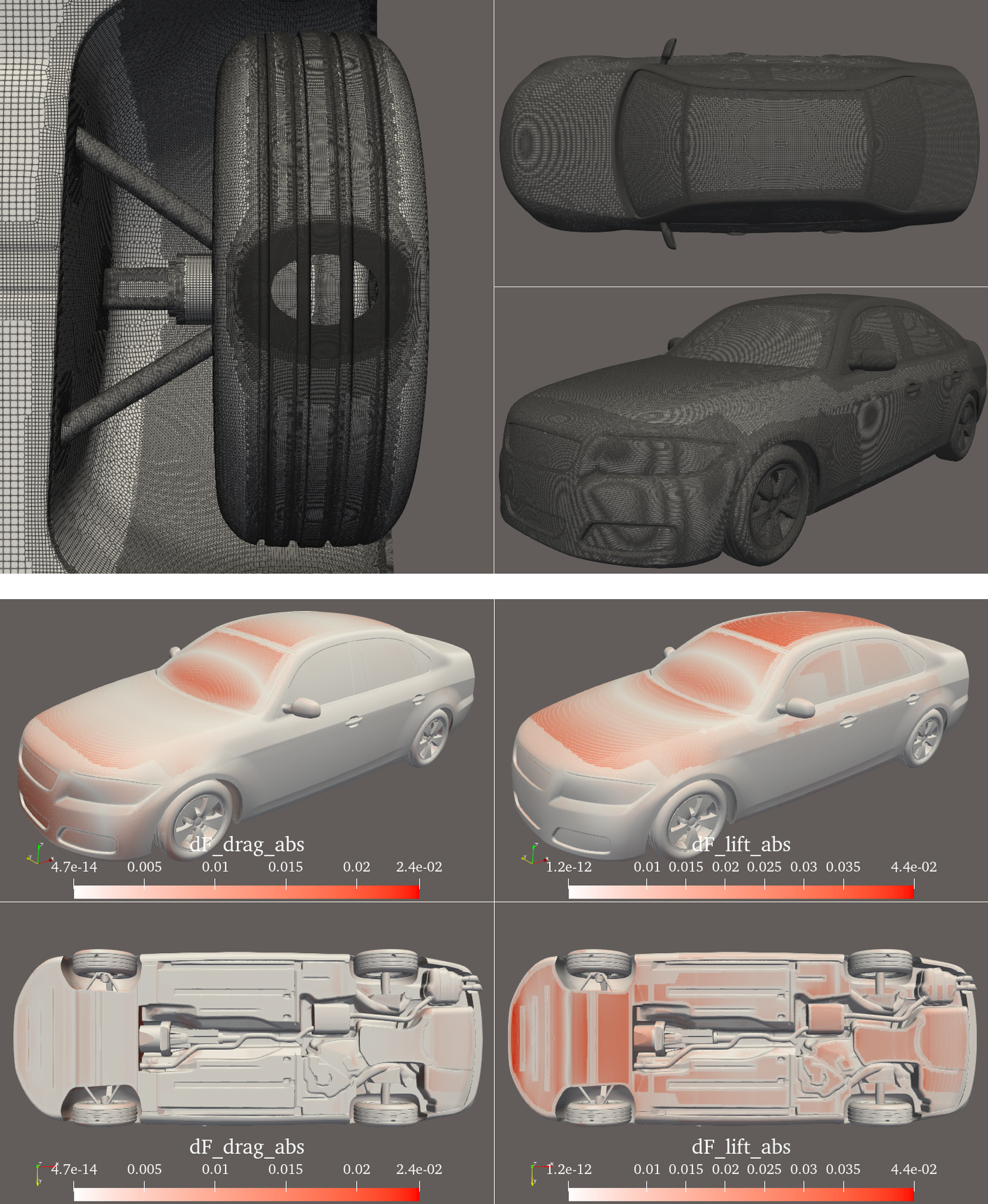}
  \caption{\textbf{Per-face drag and lift contributions on DrivAerML run 278.}
  The left column shows $|dF_{\mathrm{drag}}|$ and the right column shows $|dF_{\mathrm{lift}}|$.
  White denotes a small contribution to the integrated coefficient.}
  \label{fig:app-force-contrib}
\end{figure}

\begin{figure}[ht]
  \centering
  \includegraphics[width=\linewidth]{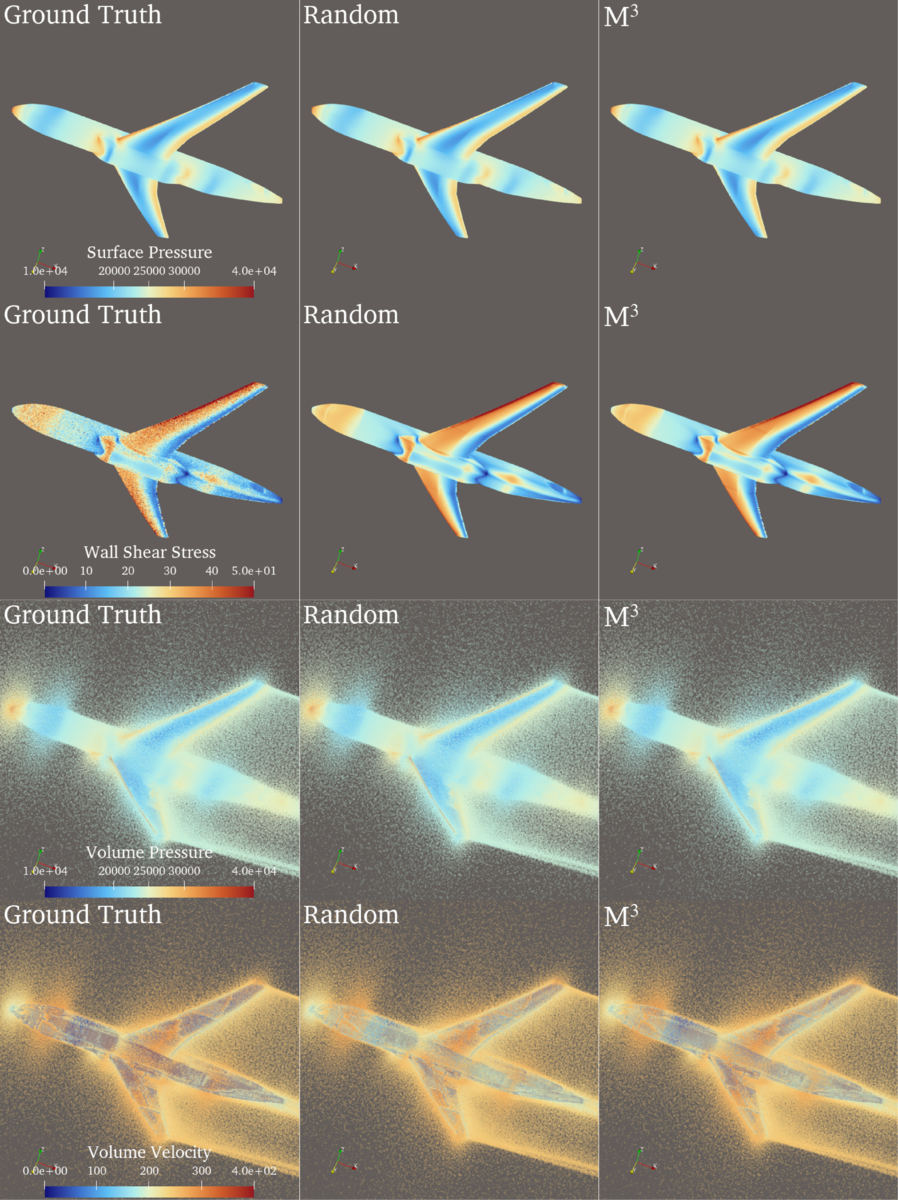}
  \caption{\textbf{Field prediction results for SHIFT-Wing.}}
  \label{fig:app-vis-21}
\end{figure}

\begin{figure}[ht]
  \centering
  \includegraphics[width=\linewidth]{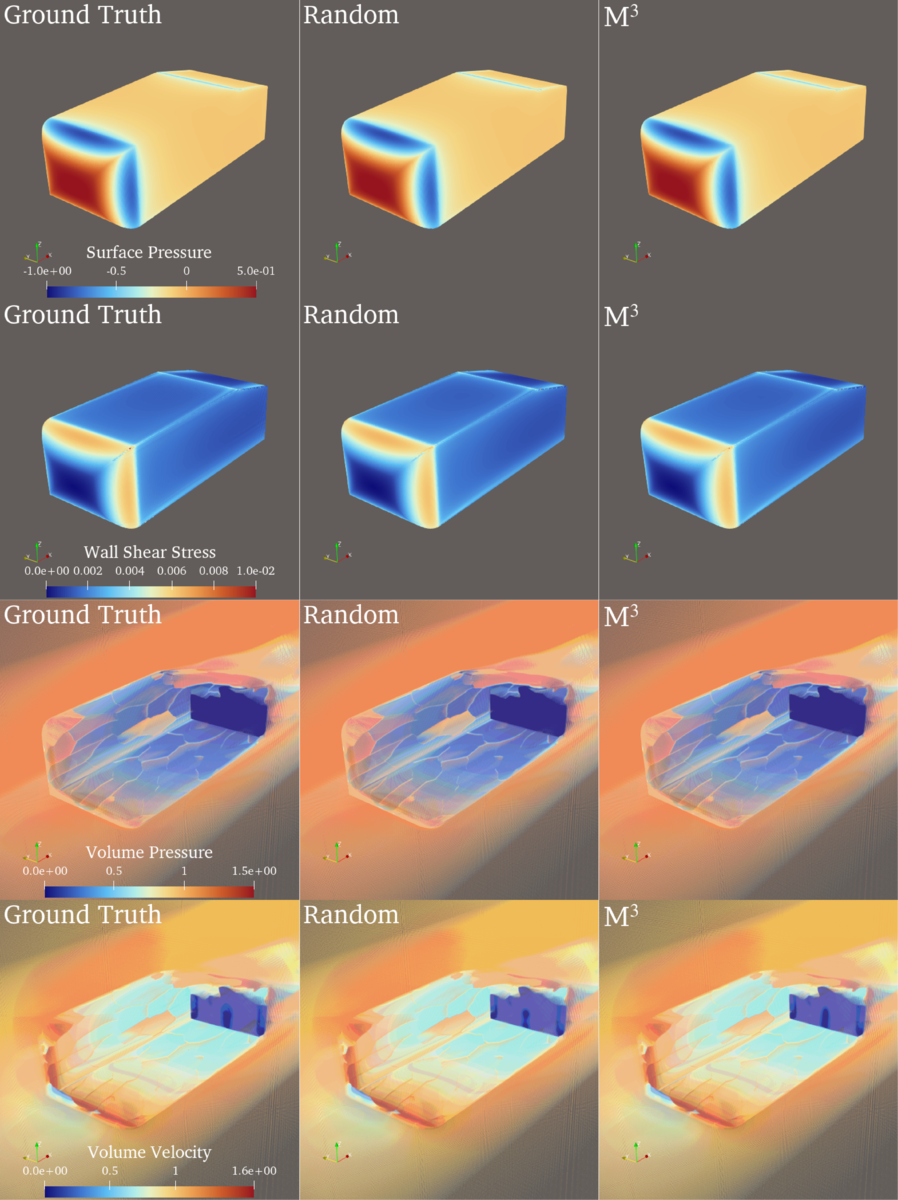}
  \caption{\textbf{Field prediction results for AhmedML.}}
  \label{fig:app-vis-22}
\end{figure}

\begin{figure}[ht]
  \centering
  \includegraphics[width=\linewidth]{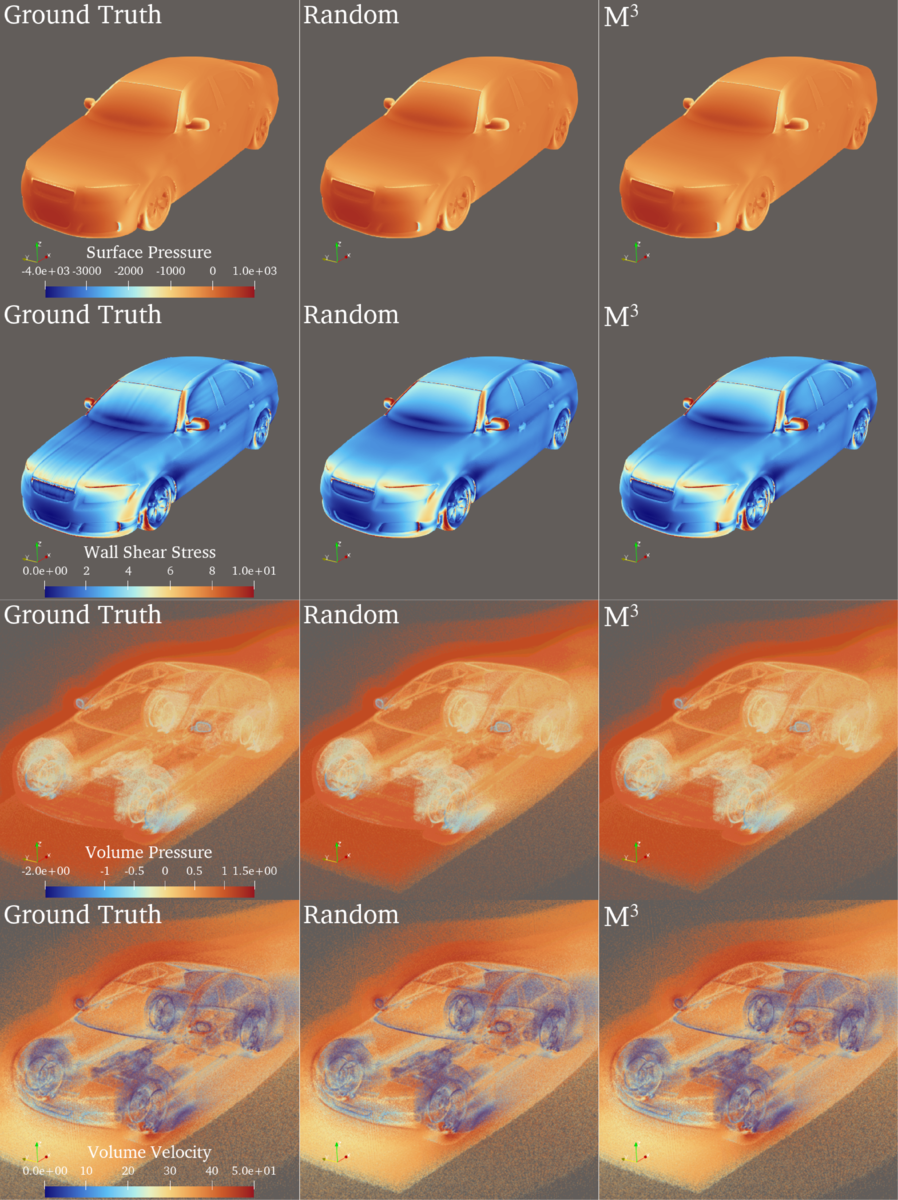}
  \caption{\textbf{Field prediction results for DrivAerML.}}
  \label{fig:app-vis-23b}
\end{figure}

\begin{figure}[ht]
  \centering
  \includegraphics[width=0.95\linewidth]{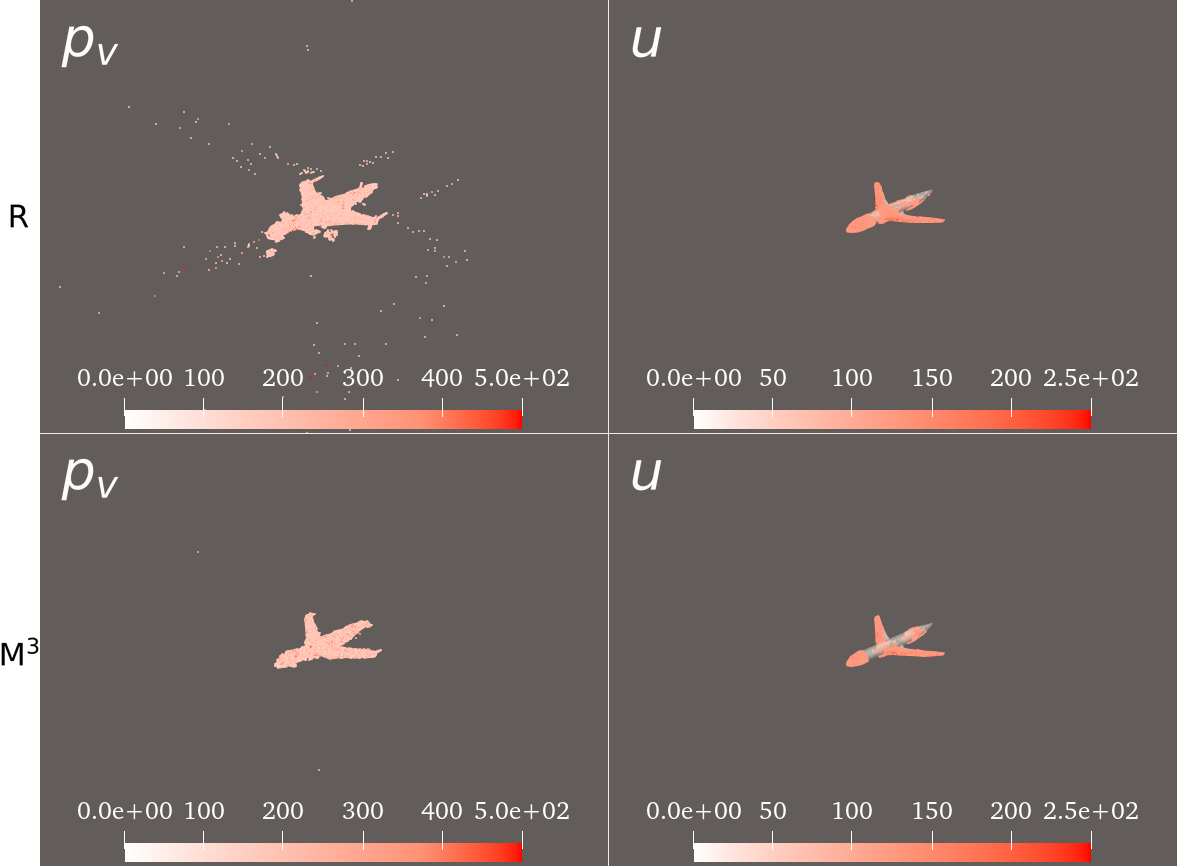}
  \caption{\textbf{Top 5\% error map of volume field for SHIFT-Wing.}}
  \label{fig:app-vis-25}
\end{figure}

\begin{figure}[ht]
  \centering
  \includegraphics[width=0.95\linewidth]{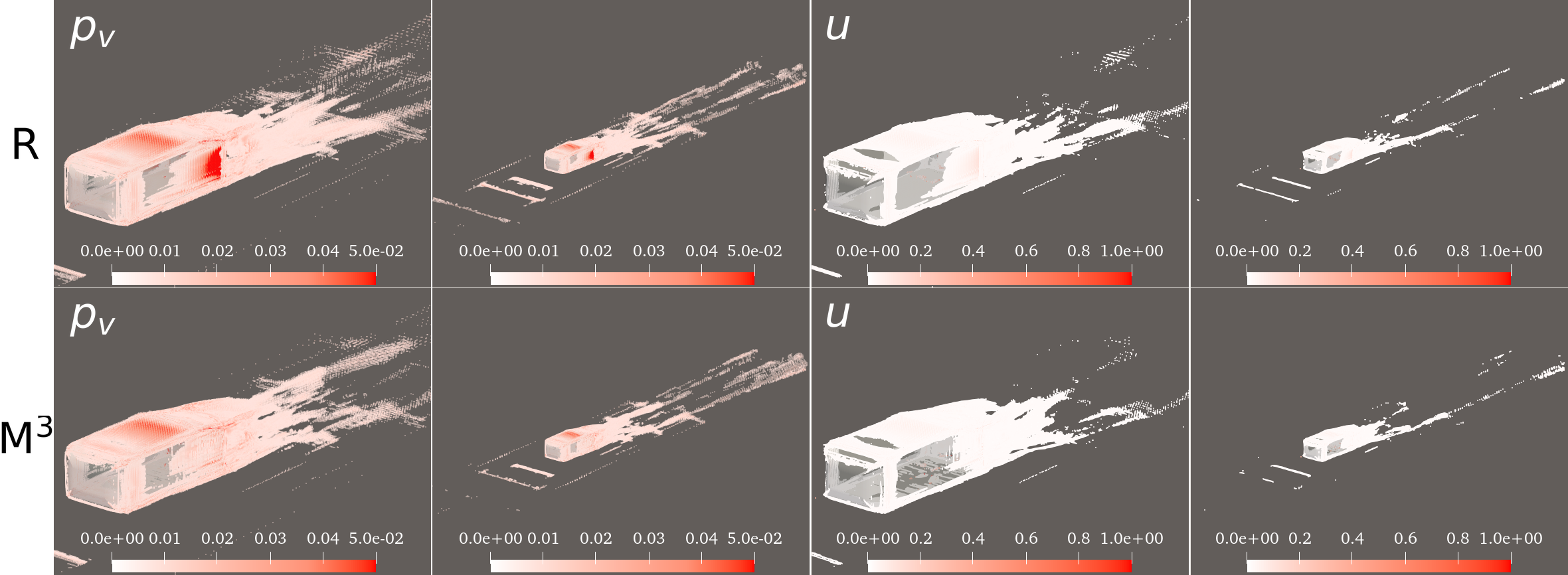}
  \caption{\textbf{Top 5\% error map of volume field for AhmedML.}}
  \label{fig:app-vis-26}
\end{figure}

\begin{figure}[ht]
  \centering
  \includegraphics[width=0.95\linewidth]{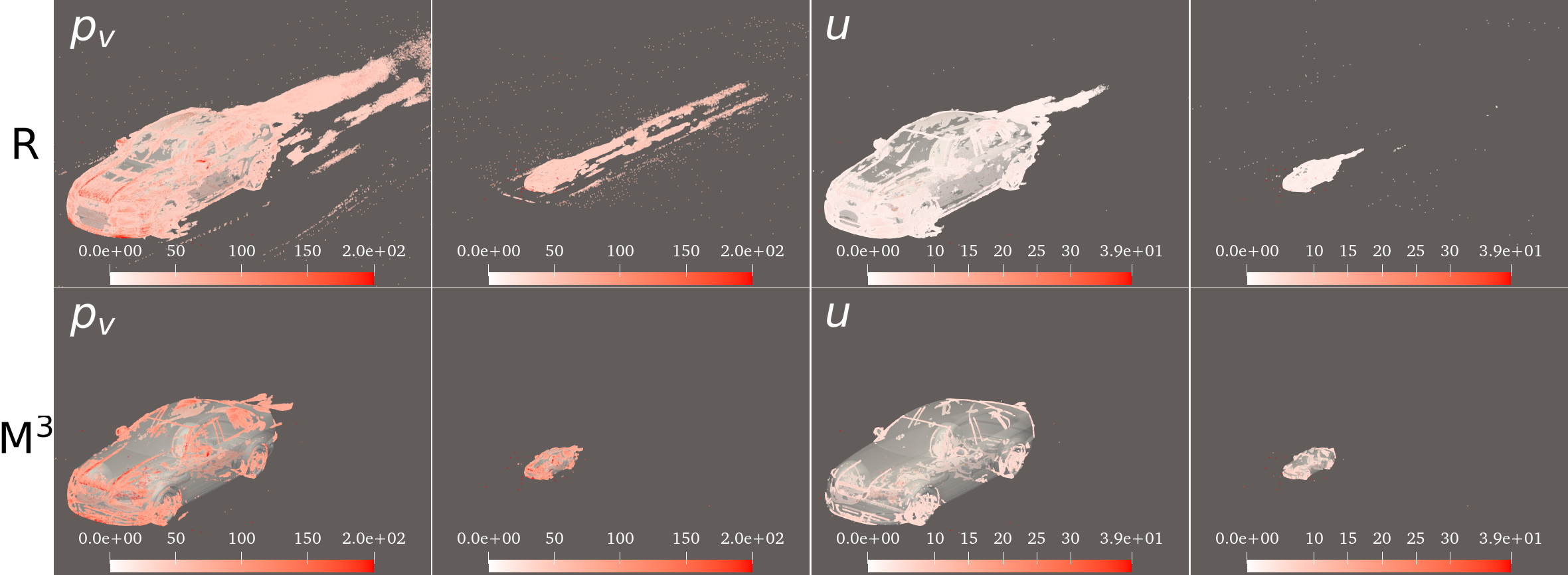}
  \caption{\textbf{Top 5\% error map of volume field for DrivAerML.}}
  \label{fig:app-vis-27b}
\end{figure}

\begin{figure}[ht]
  \centering
  \includegraphics[width=\linewidth]{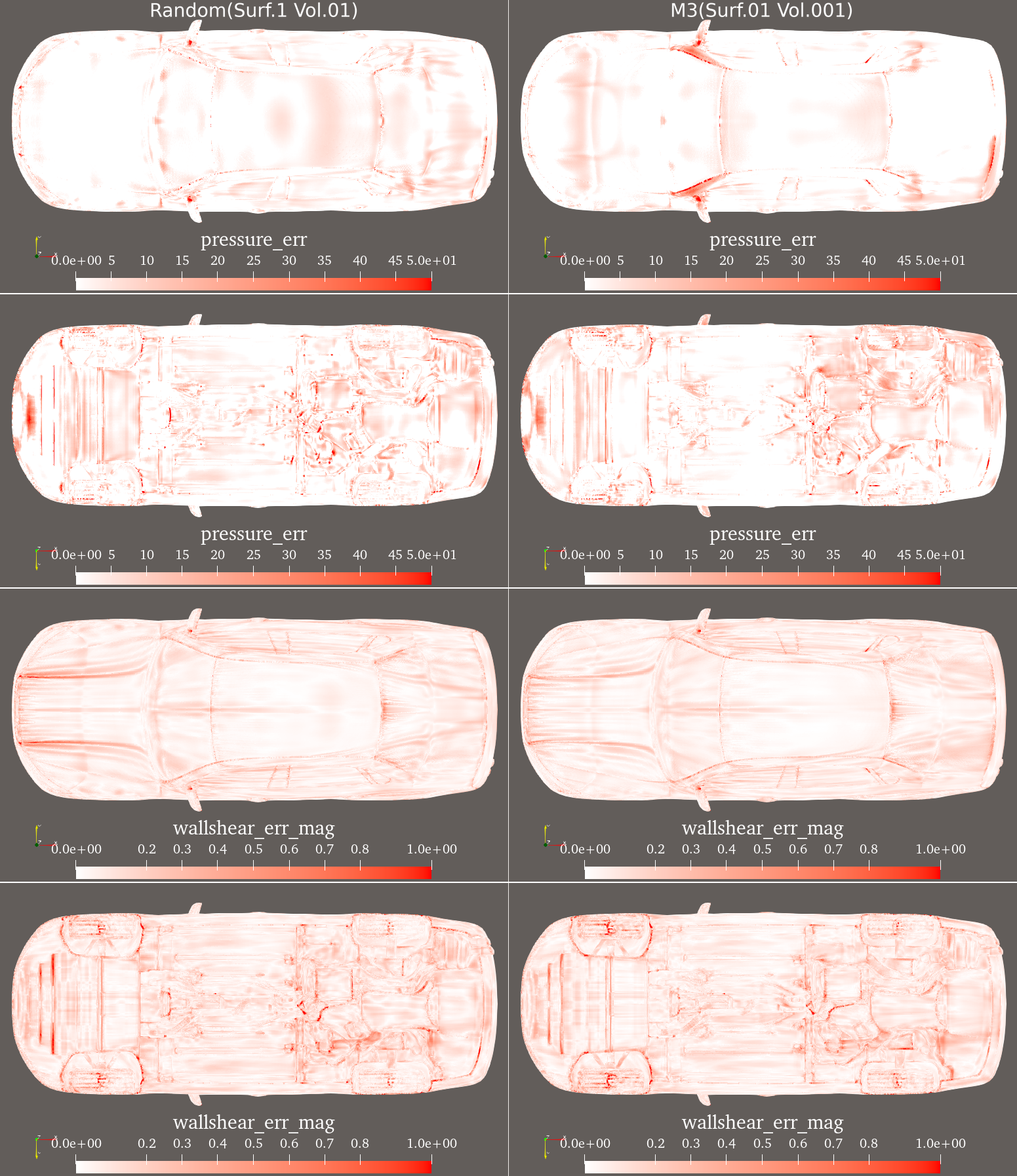}
  \caption{\textbf{Full surface error map for DrivAerML.}}
  \label{fig:app-vis-27a}
\end{figure}

\begin{figure}[ht]
  \centering
  \includegraphics[width=\linewidth]{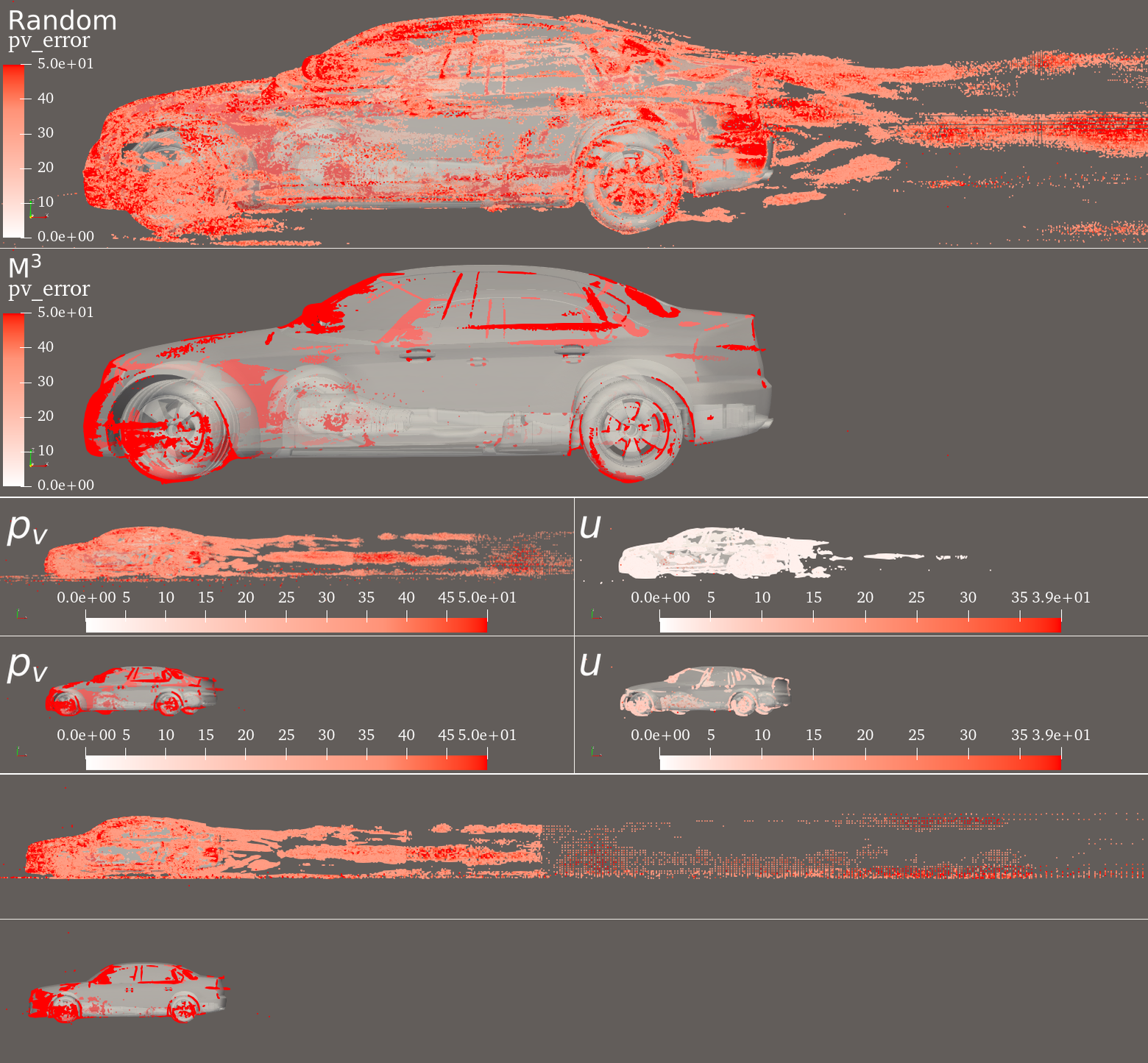}
  \caption{\textbf{Top 5\% error map of volume field for DrivAerML (multi-view 1).}}
  \label{fig:app-vis-27c}
\end{figure}

\begin{figure}[ht]
  \centering
  \includegraphics[width=\linewidth]{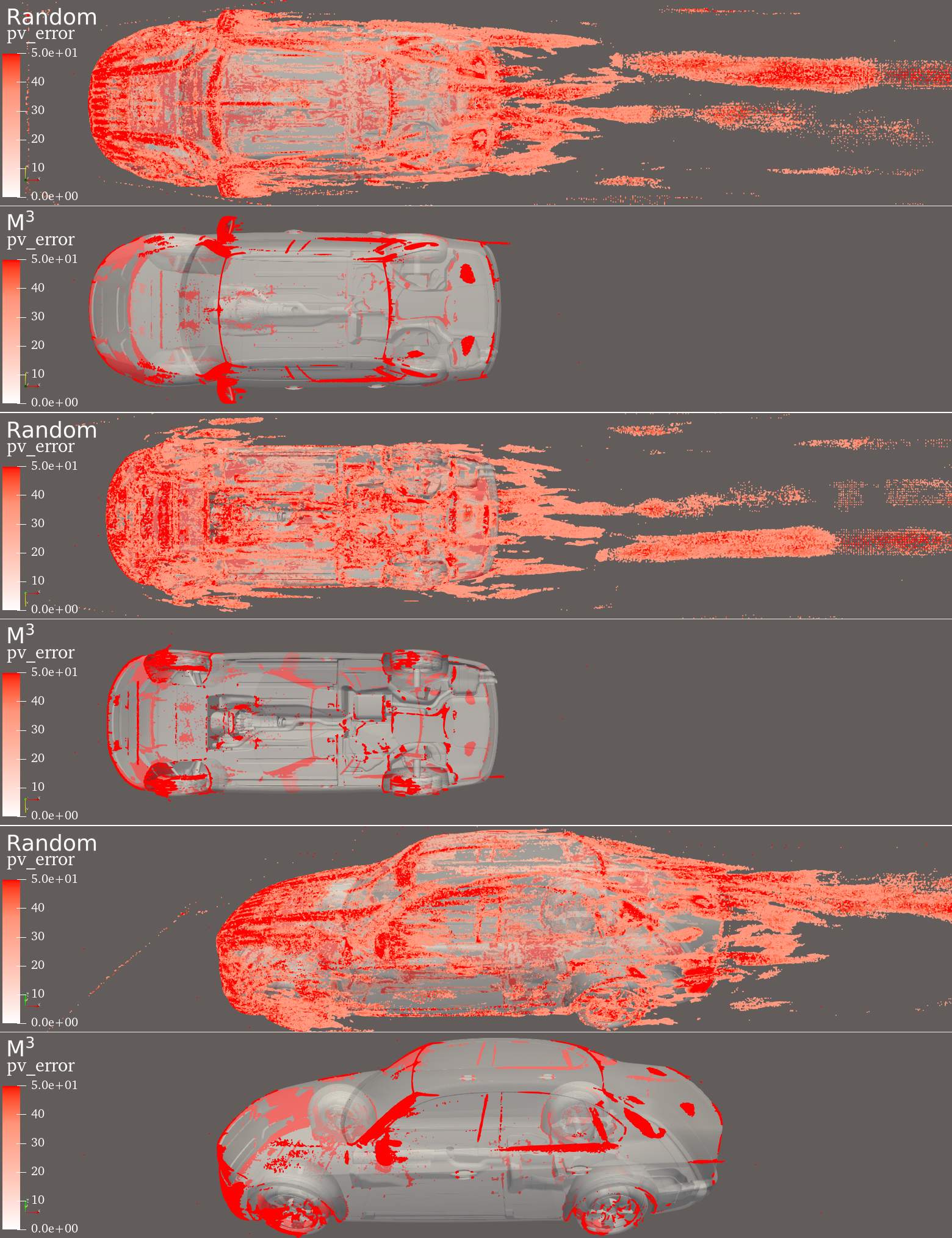}
  \caption{\textbf{Top 5\% error map of volume field for DrivAerML (multi-view 2).}}
  \label{fig:app-vis-27d}
\end{figure}

\begin{figure}[ht]
  \centering
  \includegraphics[width=\linewidth]{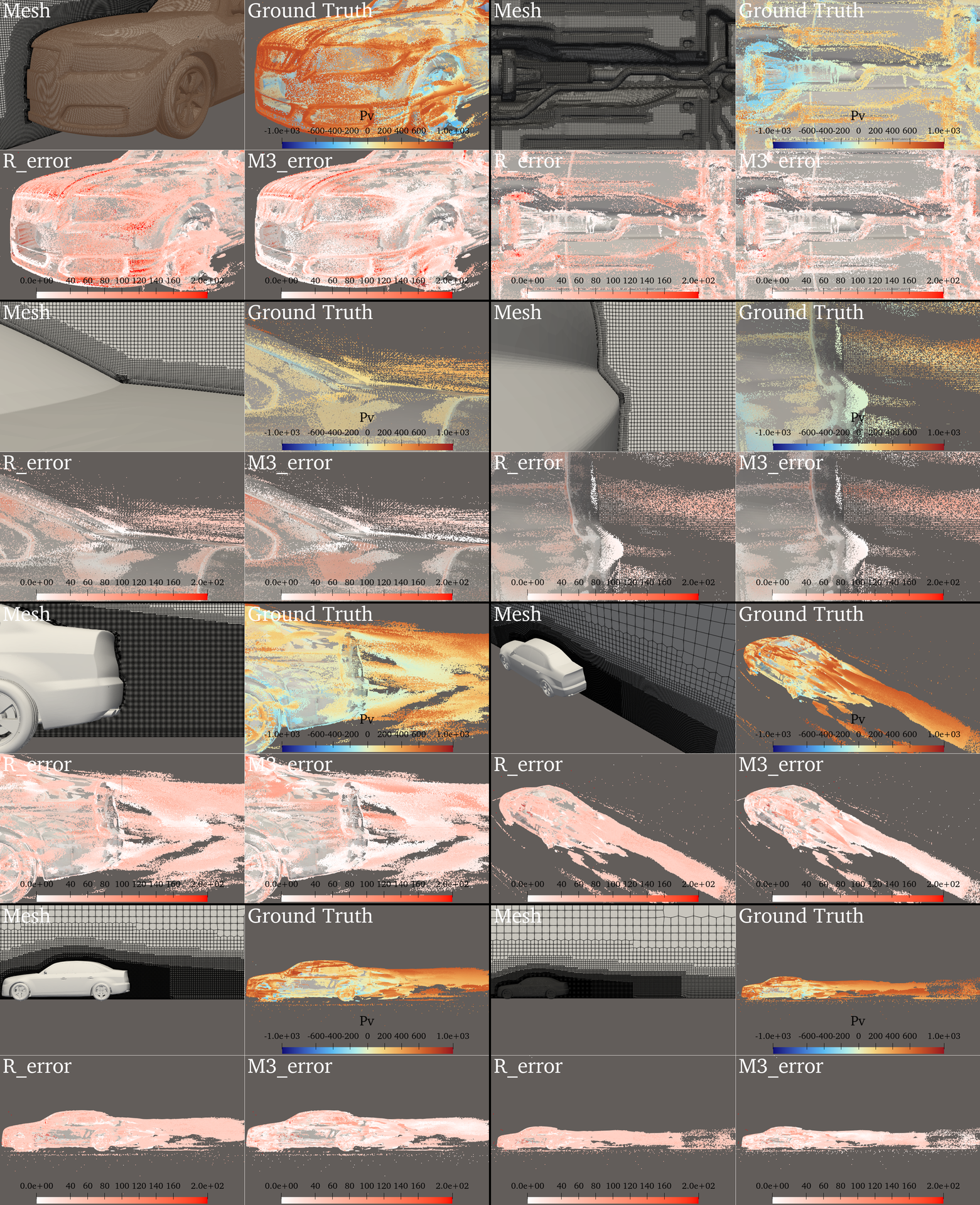}
  \caption{\textbf{Local views of high-error regions in DrivAerML.}
  The displayed query points are selected as the union of the per-model top-5\% error points.}
  \label{fig:app-local-drivaer}
\end{figure}

\begin{figure}[ht]
  \centering
  \includegraphics[width=\linewidth]{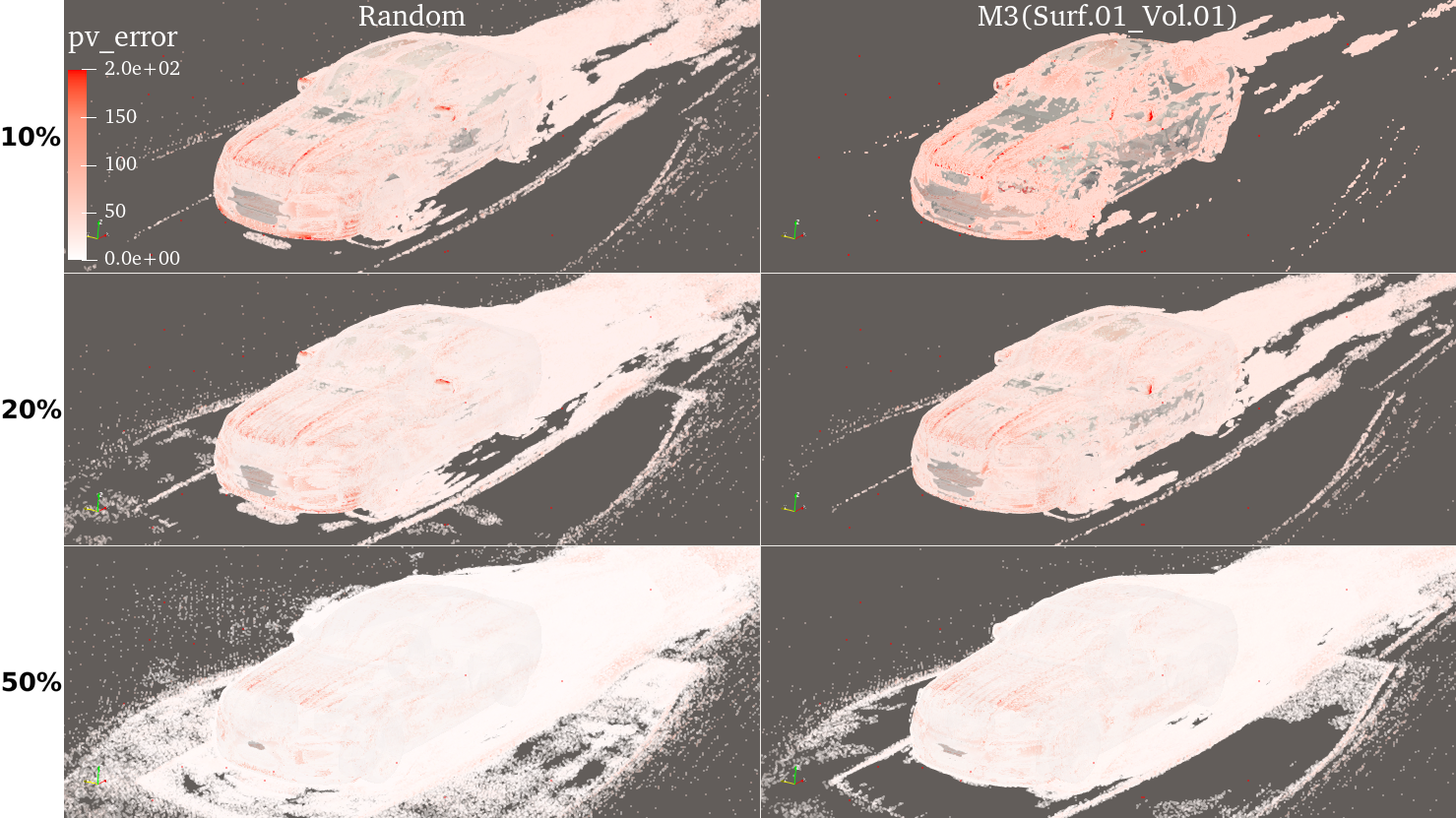}
  \caption{\textbf{Top error map for 10\%, 20\%, 50\% of full-mesh inference for DrivAerML (1).}}
  \label{fig:app-vis-28}
\end{figure}

\begin{figure}[ht]
  \centering
  \includegraphics[width=\linewidth]{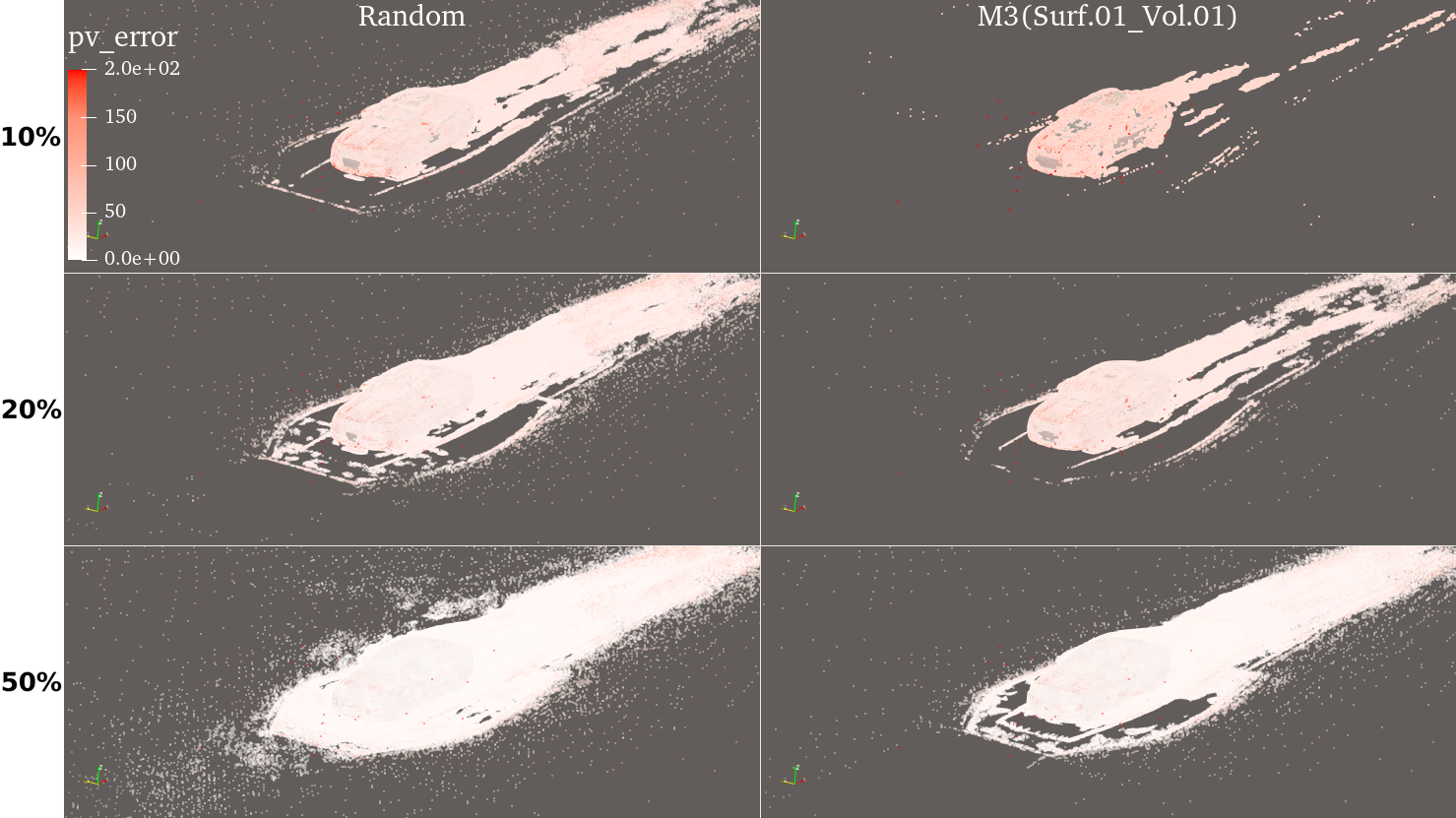}
  \caption{\textbf{Top error map for 10\%, 20\%, 50\% of full-mesh inference for DrivAerML (2).}}
  \label{fig:app-vis-29}
\end{figure}







\end{document}